\documentclass[10pt]{article}
\usepackage[a4paper,margin=1.8cm]{geometry}
\usepackage{lmodern}
\usepackage{rhineai}

\usepackage{amsmath,amsfonts,bm}
\newcommand{\E}{\mathbb E}
\newcommand{\KL}{D_{\mathrm{KL}}}
\newcommand{\Var}{\mathrm{Var}}

\usepackage[hidelinks]{hyperref}
\usepackage{url}
\usepackage{booktabs}
\usepackage{array}
\usepackage{xcolor,colortbl}
\definecolor{ourrowpink}{HTML}{FCD0E9}
\usepackage{float}
\usepackage{placeins}
\usepackage{amsthm}
\usepackage{graphicx}
\usepackage[caption=false,font=footnotesize]{subfig}
\usepackage{wrapfig}
\usepackage{needspace}
\usepackage{tcolorbox}
\tcbuselibrary{skins}
\definecolor{owtsampletext}{HTML}{205958}
\newtcolorbox{owtsamplebox}[1]{
  enhanced,
  title={#1},
  colback=black!3,
  colbacktitle=black!3,
  colframe=black!45,
  coltitle=black,
  colupper=owtsampletext,
  fonttitle=\small\bfseries,
  fontupper=\footnotesize,
  titlerule=0pt,
  boxrule=0.5pt,
  arc=5pt,
  left=7pt,
  right=7pt,
  top=3pt,
  bottom=3pt,
  toptitle=4pt,
  bottomtitle=2pt,
  boxsep=0pt,
  before skip=4pt,
  after skip=4pt,
  before upper={\setlength{\parindent}{0pt}\setlength{\parskip}{1pt}}
}
\let\appendixaddcontentsline\addcontentsline
\renewcommand{\addcontentsline}[3]{}
\definecolor{takeawayBG}{HTML}{FFEEF1}
\definecolor{takeawayFrame}{HTML}{960018}
\newtcolorbox{takeaway}{
  colback=takeawayBG,
  colframe=takeawayFrame,
  boxrule=0.7pt,
  arc=3pt,
  left=6pt,
  right=6pt,
  top=3pt,
  bottom=3pt,
  boxsep=0pt,
  before skip=4pt,
  after skip=4pt,
  fontupper=\normalfont,
  before upper={{\color{takeawayFrame}\textbf{Takeaway.}}\ }
}

\newtheorem{proposition}{Proposition}
\newtheorem{lemma}{Lemma}

\usepackage{xcolor}
\newif\ifshowtodos
\showtodostrue

\title{Spherical Brownian Diffusion for Language Modeling}

\author{Andrea Miele\rhineaiaff{1}, Ilija Bogunovic\rhineaiaff{1}}
\rhineaiaffiliations{\rhineaiaff{1}University of Basel}

\rhineaicode{\url{https://github.com/andreamiele/sbd}}
\rhineaicorrespondence{\href{mailto:andrea.miele.pro@gmail.com}{andrea.miele.pro@gmail.com}}

\begin{document}

\rhineaifrontmatter{%
Continuous diffusion language models involve two coupled design choices: how to corrupt continuous token representations and how to invert that corruption at generation time. Recent work has explored language generation on the unit sphere through geodesic and von Mises--Fisher (vMF) paths, raising the question of which forward process to use on this geometry and how much that choice matters relative to inference. We introduce {\color{rhineaiAccent}\textbf{Spherical Brownian Diffusion}}, using Brownian motion as the canonical isotropic diffusion on the sphere. Its heat semigroup provides a principled forward process, an additive dimension-free noise clock, and an exact ancestral bridge, while a clean-token posterior learned with cross-entropy supports multiple samplers. To separate corruption from inference, we derive high-dimensional approximation rates between Brownian and moment-matched vMF corruption and use this relationship for controlled substitutions. On Sudoku, Brownian and moment-matched vMF give nearly identical results across four samplers and five budgets under otherwise identical settings. In contrast, with the corruption process fixed, sampler choice produces a 25.0-point spread at a sampler budget of 180 calls. Using the ancestral bridge, we split this gain into a memory effect from re-noising the state, present on Sudoku and GSM8K, and a token-selection effect on GSM8K and largest at small budgets. Across reasoning tasks, categorical refinement improves performance, whereas unconditional text generation benefits from stochastic updates that preserve diversity.
}

\vspace{-1ex}
\section{Introduction}
\vspace{-1ex}
Most language models are autoregressive \citep{vaswani2017attention}: they
generate tokens sequentially and cannot revise earlier commitments. Diffusion models instead update many positions in parallel and can revise them throughout generation. 
In continuous diffusion language models
\citep{li2022diffusionlmimprovescontrollabletext,dieleman2022continuous,
gulrajani2023likelihood,yang2026replaid,chen2026langflow}, tokens are
represented in a continuous space. Beyond choosing how to represent these
discrete data, two coupled design choices remain: how to corrupt the
representations and how to invert that corruption at inference
\citep{dieleman2022continuous,deschenaux2026hyperspherical,
chemseddine2026spherical}.

The representation of discrete data on a continuous space creates a trade-off: one-hot states represent tokens exactly, but their dimension scales with the vocabulary and their distances carry little semantic structure \citep{deschenaux2026hyperspherical}; lower-dimensional embeddings avoid both issues but require choosing a geometry. The unit sphere is a natural bounded choice, and recent spherical models use a geodesic path (S-FLM) \citep{deschenaux2026hyperspherical} or a von Mises-Fisher (vMF) path
\citep{chemseddine2026spherical}. Because each
pairs its forward process with a different inference procedure, the effect
of the corruption kernel itself is hard to isolate. We ask how the kernel
interacts with training and inference, and how much it matters relative to
the inference rule. \looseness=-1

We introduce \textbf{Spherical Brownian Diffusion} (SBD), which corrupts token embeddings with Brownian motion on the unit sphere; its transitions are given by the spherical heat kernel, which we call \emph{Heat}.
Up to time scaling, this is the canonical isotropic diffusion on the sphere. SBD is trained with cross-entropy to predict the clean-token posterior
(Figure~\ref{fig:spherical-brownian-overview}), giving a single model that
supports both continuous reverse-time SDE/ODE sampling and categorical
refinement. This lets us vary the inference rule without changing the learned
denoiser; for refinement, q-argmax selects the posterior mode and re-noises its embedding with Heat.

\emph{Heat} also admits an exact ancestral bridge: the distribution of a less noisy state given the current state and a clean token. The bridge controls how much of the current state the sampler keeps: we add extra noise before each reverse step and bridge back, so more noise means less memory, at no extra network cost. Other spherical paths we consider do not allow this, so Heat serves both as a reference process and as a tool for separating state memory from token selection. To isolate the effect of the kernel itself, each kernel must be paired with compatible inference dynamics. We therefore introduce a conditional quantile map, built on the radial-transport framework of \citet{chemseddine2026spherical}, that handles multimodal paths without stiff velocity integration. \looseness=-1

\textbf{Results.} Using matched transport, crossed training--sampling substitutions, and the Heat bridge, we study the roles of the corruption kernel in training, continuous transport, and re-noising for categorical refinement (Sections~\ref{sec:kernel-path}, \ref{sec:kernel-mismatch} and \ref{sec:kernel-refinement} respectively). We find that when training and sampling use compatible continuous transport, the path itself matters little: Heat, moment-matched vMF, and distant mixture kernels reach similar accuracy. 
Moment-matched vMF provides a principled nearby control, as our theory shows that it approaches Heat at the critical noise scale. 
In contrast, sampling with a kernel different from the one used in training is costly: replacing Heat by a mixture kernel on a Heat-trained GSM8K model reduces accuracy to near zero, while a short matched fine-tune largely restores it. 
Kernel shape also interacts with refinement: mixture re-noising helps q-argmax but hurts confidence gating.
Inference has the larger effect overall. With Heat fixed, four inference rules span 25 points on hard Sudoku. The ancestral bridge attributes much of this gap to state memory: retaining the previous state preserves noise on Sudoku and token commitments on GSM8K. This cost grows with finer step grids on both Sudoku and GSM8K.
Across Sudoku, GSM8K, and OpenWebText, SBD is competitive with recent continuous models, with the strongest method depending on the inference regime; q-argmax also improves a released S-FLM model, showing that the refinement gain transfers beyond Heat.

\begin{figure}[!t]
  \centering

  \includegraphics[width=\linewidth]{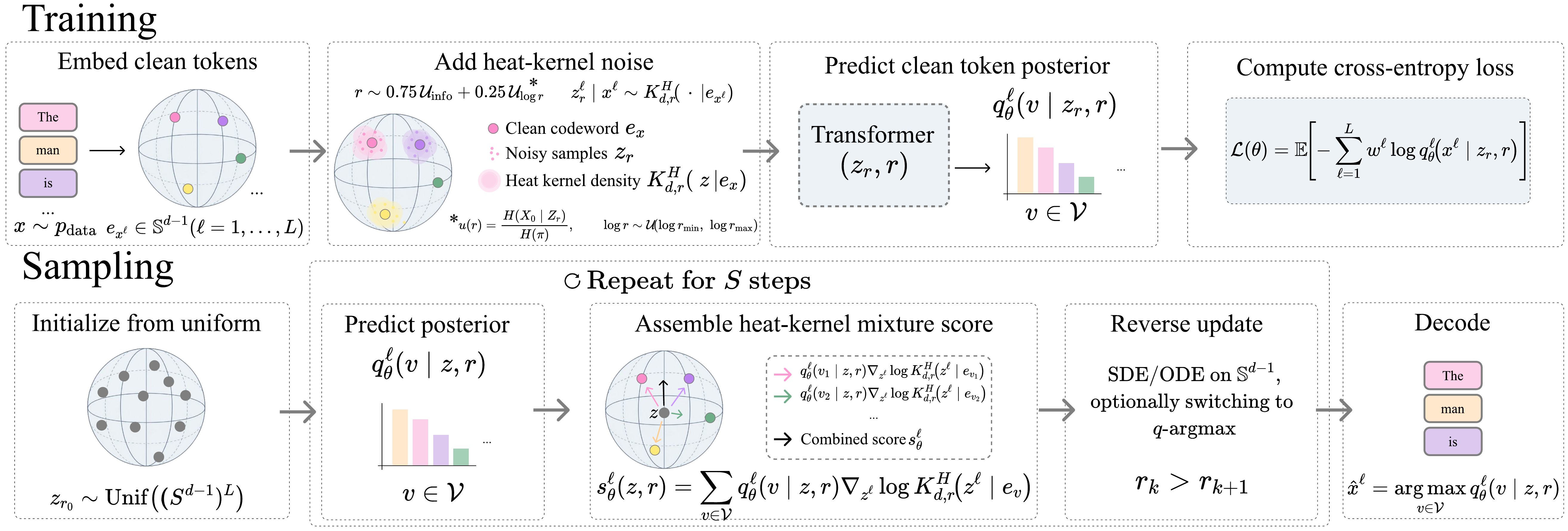}

\caption{Overview of SBD. \textbf{Training (top):} token embeddings are
  corrupted with the spherical heat kernel, and a Transformer predicts the
  clean-token posterior with cross-entropy. \textbf{Sampling (bottom):} the
  posterior either builds the score for reverse SDE/ODE updates or drives
  categorical refinement, which selects tokens and re-noises their
  codewords.}

  \label{fig:spherical-brownian-overview}
\end{figure}

\textbf{Contributions.}
\textbf{(1)} We introduce \textbf{SBD}, to our knowledge the first language
diffusion model whose forward corruption is given directly by the spherical heat semigroup on token embeddings. Its heat semigroup offers a
principled corruption process, a clean-token posterior supporting continuous
and categorical sampling, and an exact ancestral bridge controlling trajectory
memory (Section~\ref{sec:spherical-brownian-diffusion}).
\textbf{(2)} To compare spherical corruption processes without confounding
kernel shape, time parameterization, and inference, we develop a conditional
quantile-map integrator, \textbf{heat--vMF approximation rates}, and a shape
residual (Section~\ref{sec:spherical-brownian-diffusion}).
\textbf{(3)} Using these tools, we show that \textbf{inference can matter more
than the spherical kernel} in our controlled settings, and isolate state memory
as a key factor limiting refinement, with a cost that grows with budget
(Sections~\ref{sec:kernel-matters}--\ref{sec:experiments}).\looseness=-1

\section{Background}
\label{sec:background}

\textbf{Notation.}
Let $\mathbf x=(x^1,\ldots,x^L)\in\mathcal V^L$ be a token sequence, with
token embeddings $\mathbf e_v\in\mathbb S^{d-1}$. Then
$\mathbf z\in(\mathbb S^{d-1})^L$, and the tangent projection is
$P_{\mathbf z}=I-\mathbf z\mathbf z^\top$.

\subsection{Non-autoregressive generation}
Autoregressive models generate tokens from left to right. Masked discrete
diffusion models instead corrupt tokens and learn to reconstruct them in
parallel \citep{sahoo2024simple,austin2021structured}. Continuous
alternatives embed tokens in a Euclidean or manifold-valued space and learn a
reverse diffusion or flow \citep{dieleman2022continuous,lipman2023flow,
zhang2026latentgeometry}. These models can revise several positions at each
step, but their behavior depends on the latent geometry and on how continuous
states are mapped back to tokens.

\subsection{Spherical paths}
S-FLM places token embeddings on a sphere and follows deterministic
interpolating paths \citep{deschenaux2026hyperspherical}. A related approach
uses von Mises--Fisher (vMF) corruption \citep{chemseddine2026spherical}. For
a clean codeword $\mathbf e$, a vMF conditional has density proportional to
$\exp(\kappa\langle\mathbf z,\mathbf e\rangle)$ and score
\begin{equation}
  \nabla_{\mathbf z}\log p_{\mathrm{vMF}}(\mathbf z\mid\mathbf e)
  =\kappa(\mathbf e-\langle\mathbf z,\mathbf e\rangle\mathbf z).
  \label{eq:vmf-background}
\end{equation}
Its radial coefficient is constant in the cosine
\citep{mardia2000directional}. Spherical Brownian motion
has the same tangent direction but a cosine-dependent coefficient, which is
the main geometric difference studied here.

\subsection{Spherical Brownian diffusion}
\label{sec:background-brownian}
Brownian motion is the canonical isotropic diffusion on the sphere. We use the
dimension-free coordinate $r=-\log\rho$, where $\rho$ is the retained mean
cosine, and intrinsic heat time $\tau_d(r)=2r/(d-1)$. Here, $\rho$ is the
average cosine similarity between noisy and clean embeddings, and $\tau_d$ is
the elapsed time of standard Brownian motion on the sphere. The coordinate
$r$ is comparable across dimensions within one kernel family, but equal $r$
does not imply equal token information across families; analogous time-warping
ideas were used in CDCD \citep{dieleman2022continuous}. Section~
\ref{sec:kernel-matters} therefore also uses a measured information clock. In
the Brownian clock, the forward process satisfies
\begin{equation}
  \mathrm d\mathbf Z_r=-\mathbf Z_r\,\mathrm dr
  +\sqrt{\tfrac{2}{d-1}}P_{\mathbf Z_r}\,\mathrm d\mathbf W_r.
  \label{eq:spherical-brownian-background}
\end{equation}
Its transition density is the spherical heat kernel $K^H_{d,r}$, which solves
$\partial_rK^H_{d,r}=\Delta_{\mathbb S^{d-1}}K^H_{d,r}/(d-1)$ and converges to
the uniform distribution as $r\to\infty$. Score-based diffusion reverses this
process using the score of the noisy marginal
\citep{song2020score,lou2024scaling}.

For discrete data, the marginal score is a posterior average of conditional
heat-kernel scores:
\begin{equation}
  \nabla_{\mathbf z^\ell}\log p_r(\mathbf z)
  =\sum_{v\in\mathcal V}p(x^\ell=v\mid\mathbf Z_r=\mathbf z)
  \nabla_{\mathbf z^\ell}\log
  K^H_{d,r}(\mathbf z^\ell\mid\mathbf e_v).
  \label{eq:posterior-mixture-score-background}
\end{equation}
A path-level comparison with S-FLM and vMF is given in
Appendix~\ref{app:implementation-details}
(Table~\ref{tab:spherical-method-comparison}).

\section{Spherical Brownian Diffusion}
\label{sec:spherical-brownian-diffusion}

This section develops Spherical Brownian Diffusion (SBD) and the tools we use to study the role of the corruption kernel. We first define the Heat forward process and show how a clean-token posterior determines its marginal score (Section~\ref{sec:sbd-forward}). We then describe training and two ways we use this posterior at inference: standard continuous reverse dynamics and categorical refinement. Our main inference tool is an exact Heat bridge, which lets us vary state memory independently of token selection. Finally, to compare Heat with other spherical corruptions without confounding the forward kernel with the sampler, we introduce matched transport and identify moment-matched vMF as a principled nearby-kernel control. These tools support the controlled kernel study in Section~\ref{sec:kernel-matters}.

\subsection{Forward process and marginal score}
\label{sec:sbd-forward}

Our starting point is to use spherical Brownian motion as the tokenwise corruption process. Given a clean sequence $\mathbf x$, we corrupt each token embedding
independently with the spherical heat kernel introduced in
Section~\ref{sec:background-brownian}. At noise level $r$, we have:
$
  \mathbf z^\ell \mid x^\ell
  \sim K^H_{d,r}(\cdot \mid \mathbf e_{x^\ell}).$
Dependencies between positions are modeled by the Transformer rather than by
the forward process.

We sample the heat kernel directly through its cosine distribution. We first
draw the cosine $c$ with respect to the clean codeword, sample a uniform
tangent direction $\mathbf v$, and set
$\mathbf z^\ell
  =
  c\mathbf e_{x^\ell}
  + \sqrt{1-c^2}\,\mathbf v.$

Because the heat kernel is isotropic, its conditional score depends only on
$c_v^\ell=\langle \mathbf z^\ell,\mathbf e_v\rangle$:
\begin{equation}
  \nabla_{\mathbf z^\ell}\log
  K^H_{d,r}(\mathbf z^\ell\mid\mathbf e_v)
  =
  a^H_{d,r}(c_v^\ell)
  (\mathbf e_v-c_v^\ell\mathbf z^\ell),
  \qquad
  a^H_{d,r}(c)=\partial_c\log K^H_{d,r}(c).
  \label{eq:heat-radial-score}
\end{equation}

Rather than learning the continuous score directly, we train a Transformer to
predict the clean-token posterior
$q^\ell_\theta(v \mid \mathbf z,r)$ with cross-entropy. This follows the
score-interpolation view of continuous diffusion for categorical data
\citep{dieleman2022continuous}: the marginal score is the posterior average
of the known conditional scores. On the sphere, the corresponding
posterior-weighted tangent-sum representation was also used by
\citet{chemseddine2026spherical}. For Heat,
\begin{equation}
  \mathbf s^\ell_\theta(\mathbf z,r)
  =
  \sum_{v\in\mathcal V}
  q^\ell_\theta(v\mid\mathbf z,r)\,
  a^{\mathrm H}_{d,r}(c^\ell_v)
  \bigl(\mathbf e_v-c^\ell_v\mathbf z^\ell\bigr).
  \label{eq:marginal-score}
\end{equation}
Thus, a single categorical posterior provides the score required by continuous samplers and can also be used directly for categorical refinement.

\subsection{Training and codebooks}
\label{sec:sbd-training}

We train this posterior directly by denoising classification. For clean data $\mathbf x$ and a
sampled Heat-corrupted state $\mathbf z$, where
$p_r^H(\mathbf z\mid\mathbf x)=\prod_{\ell=1}^L
K^H_{d,r}(\mathbf z^\ell\mid\mathbf e_{x^\ell})$, we predict the clean
tokens with cross-entropy,
\begin{equation}
  \mathcal L_{\mathrm{CE}}(\theta)
  = \E_{\mathbf x,\,r,\,\mathbf z}
    \Big[-\textstyle\sum_{\ell=1}^{L} w_\ell
    \log q_\theta^\ell(x^\ell\mid\mathbf z,r)\Big],
  \label{eq:sbd-training-loss}
\end{equation}
where $w_\ell$ is a task-dependent loss weight. We use frozen or learned unit-norm codebooks; their construction and the
task-specific noise distributions are described in
Appendix~\ref{app:implementation-details}.

\subsection{Continuous dynamics and posterior refinement}
\label{sec:reverse-sampling}

Having learned $q_\theta$, we next describe how SBD generates. We consider two uses of the same posterior. First, it can be converted into a marginal score and used with standard reverse-SDE or probability-flow ODE dynamics.
Second, it can act directly on token predictions through categorical
refinement. We introduce q-argmax for the latter and, exploiting the Heat semigroup, an ancestral bridge that lets us interpolate between retaining and discarding the current continuous state.

\textbf{Continuous reverse dynamics.}
As a standard continuous baseline, we convert the posterior into the marginal score using Equation~\ref{eq:marginal-score}. We then integrate either the reverse
SDE or the probability-flow ODE on the sphere. The reverse SDE is stochastic,
whereas the ODE is deterministic given its initialization. Implementation
details are given in Appendix~\ref{app:implementation-details}. These are
the standard score-based reverse-time dynamics for a diffusion process
\citep{song2020score}.

\textbf{Categorical posterior refinement.}
Rather than converting the posterior into a continuous score, we can act on its token predictions directly. We introduce \emph{q-argmax}, which selects the most likely token at each position and
re-noises its codeword at the next radius:
\begin{equation}
  \widehat x_k^\ell=\arg\max\nolimits_{v\in\mathcal V}\,
  q_\theta^\ell(v\mid\mathbf z_k,r_k),
  \qquad
  \mathbf z_{k+1}^\ell\sim K^H_{d,r_{k+1}}\big(\cdot\mid\mathbf e_{\widehat x_k^\ell}\big).
  \label{eq:q-argmax-step}
\end{equation}
The token choice is deterministic given the current state, while re-noising
keeps the trajectory stochastic. Recomputing the posterior at every step
allows later iterations to revise earlier predictions. This is related to
predict-project-re-noise updates and to Diffusion-LM's repeated return to
token embeddings, but here the projection is explicitly the clean-token
posterior mode followed by a fresh Heat draw
\citep{sahoo2024simple,li2022diffusionlmimprovescontrollabletext}.

\textbf{Ancestral bridge and restart.}
q-argmax changes two aspects of ancestral sampling at once: it replaces
posterior sampling by a mode decision and discards the previous continuous state by re-noising from scratch. To separate these effects, we exploit the Heat semigroup to construct an exact ancestral bridge. For Heat, the semigroup gives the \emph{bridge} density $B^H_{s\leftarrow R}$
of the state at radius $s<R$, given the current state $y$ at radius $R$ and
the clean token $x$:
\begin{equation}
  B^H_{s\leftarrow R}(\mathbf z\mid\mathbf y,\mathbf e_x)
  = K^H_{d,s}(\mathbf z\mid\mathbf e_x)\,
    K^H_{d,R-s}(\mathbf y\mid\mathbf z)\,
    \big/\,K^H_{d,R}(\mathbf y\mid\mathbf e_x).
  \label{eq:heat-ancestral-bridge}
\end{equation}
Sampling $x$ from the exact posterior and then from this bridge gives the
exact finite-time reverse conditional; numerically, this is exact up to the
heat-table accuracy.

We use the bridge to control how much state memory is retained. After
evaluating $q_\theta(\cdot\mid\mathbf z_{r_k},r_k)$, we select
$\widehat x_k$ either by argmax or by temperature-scaled posterior sampling,
add $\Delta\geq0$ forward noise, and bridge back:
\begin{equation}
  \widetilde{\mathbf z}\sim
  K^H_{d,\Delta}(\cdot\mid\mathbf z_{r_k}),\qquad
  \mathbf z_{r_{k+1}}\sim
  B^H_{r_{k+1}\leftarrow r_k+\Delta}
  (\cdot\mid\widetilde{\mathbf z},\mathbf e_{\widehat x_k}).
  \label{eq:heat-restart-step}
\end{equation}
This requires no additional network evaluations. $\Delta=0$ retains the
ancestral bridge state, while increasing $\Delta$ removes memory and approaches
the memoryless refinement update. Token temperature independently controls
whether selection is stochastic ($T=1$) or approaches argmax ($T\to0$).

We also consider \textbf{confidence-gated refinement}, which keeps sufficiently
confident tokens fixed while continuing to refine the remaining positions.
This is related to confidence-based remasking and token retention in masked
diffusion, but our gate uses the clean-token posterior at the current noisy
state to decide which positions to hold fixed
\citep{sahoo2024simple,agarwal2026posteriorrefinement}.
Unless stated otherwise, the final sequence is decoded from one last posterior
evaluation,
$\widehat x^\ell
  =
  \arg\max_{v\in\mathcal V}
  q_\theta^\ell(v\mid\mathbf z,r_{\min}).$
Further sampler details and NFE accounting are given in
Appendix~\ref{app:implementation-details}. \looseness=-1

\subsection{Matched transport across isotropic kernels}
\label{sec:isotropic-transport}

The previous sections define SBD for Heat. To determine whether its behavior comes from the corruption kernel itself, however, we need to compare Heat with alternative spherical kernels under compatible inference dynamics. Using one kernel with another kernel's reverse dynamics would confound the forward process with sampler mismatch. We therefore construct matched transport for each isotropic kernel.

We build on Riemannian flow matching and the radial-transport framework for isotropic spherical paths \citep{chen2023riemannian,chemseddine2026spherical}. For sharply multimodal paths, for which direct velocity integration becomes
difficult, we introduce a conditional quantile-map update, $c
  \longmapsto
  F_{r'}^{-1}\!\left(F_r(c)\right),$
which is the exact one-dimensional monotone transport for a single
conditional kernel. Averaging the resulting spherical displacements under the
clean-token posterior gives a first-order update. Derivations and validation
are in Appendix~\ref{app:valid-sampler}.

\subsection[Heat--vMF comparison as a nearby-kernel control]{Heat\textendash vMF comparison as a nearby-kernel control}
\label{sec:heat-vmf-comparison}

To isolate the effect of the corruption kernel, we need a principled nearby alternative to Heat. vMF is a natural choice: it is used by recent spherical language models and shares the same isotropic geometry, while differing in the shape of its conditional distribution. We therefore ask whether moment-matched vMF approaches Heat in high dimension, providing a controlled comparison for testing whether local kernel similarity translates to similar learned generation.

A vMF conditional has score $\kappa(\mathbf e-c\mathbf z)$, with a radial coefficient that is constant in $c$, whereas the heat-kernel coefficient varies with the cosine. Intuitively, the most relevant regime for comparing their shapes is the critical noise scale, where the remaining alignment with
the original token becomes comparable to random fluctuations on the sphere, both of order $1/\sqrt d$ \citep{vershynin2018highdim}.
We compare the two kernels at the critical noise radius
$r_d=(\log d+\xi)/2$, where the retained cosine is $m/\sqrt d$ for fixed $\xi$ and $m=e^{-\xi/2}$. Let $t_d=(\log d+\xi)/(d-1)$ and choose the vMF concentration $\kappa_d^\star$ to match that retained cosine. Under Heat, write $C=\langle\mathbf Z,\mathbf e\rangle$ for the noisy-clean cosine.

\begin{proposition}[Critical-scale heat-vMF approximation]
\label{prop:heat-vmf-approximation}
For fixed $\xi$ as $d\to\infty$,
\begin{equation}
  \KL\big(H_{d,r_d}\Vert V_{d,\kappa_d^\star}\big)
  = \tfrac{m^4}{4}t_d^2(1+o(1)),
  \qquad
  \E_H\big[(a^H_{d,r_d}(C)/\kappa_d^\star-1)^2\big]^{1/2}
  = m\,t_d(1+o(1)).
  \label{eq:heat-vmf-kl-main}
\end{equation}
In particular, these rates are $O((\log d/d)^2)$ and $O(\log d/d)$.
\end{proposition}

Moment matching cancels the leading linear discrepancy, leaving a quadratic
term that controls both rates. The measured shape residual follows the
high-dimensional trend: its maximum falls from $0.025$ at $d=128$ to $0.010$
at $d=768$, though not uniformly over noise levels. Appendix~\ref{app:heat-vmf-reduction}
gives the expansion, proof, and scope of the critical-path result.
Appendix Figure~\ref{fig:heat-vmf-critical-scaling} tests the two predicted
coefficients independently across eight dimensions and three critical-noise
paths, without fitting their common target $m=e^{-\xi/2}$. Both estimates
approach that target as $d$ grows, although the finite-dimensional discrepancy
shrinks slowly, approximately as $1/\log d$.

\begin{takeaway}
At the critical noise scale, moment-matched vMF is close to Heat in both
conditional likelihood and relative score. This gives the kernel study (Section~\ref{sec:kernel-matters}) a principled nearby-kernel comparison; the experiments test how that local
agreement carries over to learned generation. The proposition concerns the
critical path $r_d=(\log d+\xi)/2$, not the full sampling trajectory.
\end{takeaway}

\section{Which parts of SBD depend on the Brownian kernel?}
\label{sec:kernel-matters}

SBD uses the corruption kernel in three roles: as the continuous probability
path (Section~\ref{sec:kernel-path}), as the training channel whose posterior
the sampler relies on (Section~\ref{sec:kernel-mismatch}), and as the
re-noising operator for categorical refinement
(Section~\ref{sec:kernel-refinement}). We vary these roles separately, then
use the Heat bridge to separate state memory from token selection effects while sampling (Section~\ref{sec:bridge}).

\subsection{Setup}

We pick kernels at increasing distance from Heat: moment-matched vMF, which
our theory predicts is close (Section~\ref{sec:heat-vmf-comparison}); SLERP, the path used by S-FLM;
and two vMF mixtures whose bimodal shape deliberately departs from Heat. At retained cosine $m=e^{-r}$, a mixture places weights $m$ and
$1-m$ on vMF components with mean cosines $m+\lambda(1-m)$ and $(1-\lambda)m$,
so its mean cosine is $m$ while $\lambda\in\{0.5,0.9\}$ changes its shape (Equation~\ref{eq:mixvmf-kernel-new}). The main SLERP comparison truncates the path at $r_{\min}=0.942$, where its information is exhausted; Appendix~\ref{app:slerp-sudoku} also reports the full path. Because equal mean cosine need not mean equal token uncertainty, we also compare kernels on a measured information clock, and we report a local shape residual that is zero when a kernel differs from Heat only by a time change (Table~\ref{tab:shape-residual-main}, definitions in Appendix~\ref{app:mixvmf-sudoku}). Unless stated otherwise, Sudoku kernel experiments use hard puzzles, $d=128$, 20k training updates, 2,000 evaluation puzzles, and three training seeds.

\subsection{The kernel as a continuous probability path}
\label{sec:kernel-path}

We first ask whether the path itself matters when everything else is
matched: each kernel is trained and sampled with its own valid transport.
On hard Sudoku at $B=180$ (181 NFEs with posterior decoding), Heat,
moment-matched vMF, truncated SLERP, and the $.5/.9$ mixtures reach
$54.5\%$, $54.7\%$, $51.4\%$, $53.3\%$, and $55.3\%$, respectively
(Tables~\ref{tab:valid-transport-main} and~\ref{tab:slerp-full}). Increasing
the sampler budget to $B=720$ (721 NFEs) does not materially change the
Heat/vMF/mixture comparison within seed variation. The full SLERP path
reaches $56.2\pm2.2\%$ at $B=180$.

The stability depends on matching the dynamics to the kernel. For the $\lambda=.9$ mixture, inserting the mixture score into Brownian ODE dynamics with nearest-codeword decoding gives
$0.0\%$ accuracy at 180 NFEs. At the same sampler budget $B=180$, posterior
decoding alone raises this to $15.3\pm3.6\%$ at 181 NFEs, and valid transport
alone to $11.2\pm0.7\%$ at 180 NFEs; using both reaches $55.3\pm2.3\%$
at 181 NFEs (Table~\ref{tab:decoding-two-by-two}). Thus, changing the kernel requires both compatible transport
and compatible decoding rather than a direct score substitution. The full
valid-transport comparison is in Table~\ref{tab:valid-transport-main}; the
decoding-by-dynamics control is in Table~\ref{tab:decoding-two-by-two}, with
transport checks in Appendix~\ref{app:valid-sampler}.

\subsection{Posterior mismatch compounds on sequence generation}
\label{sec:kernel-mismatch}

\begin{wraptable}[9]{r}{0.22\linewidth}
  \vspace{-1em}
  \caption{Shape-residual ranges.}
  \label{tab:shape-residual-main}
  \centering
  \scriptsize
  \setlength{\tabcolsep}{1.5pt}
  \renewcommand{\arraystretch}{0.9}
  \begin{tabular}{@{}lr@{}}
    \toprule
    Kernel & Residual \\
    \midrule
    \multicolumn{2}{@{}l}{\emph{$d=128$}} \\
    Heat       & $\le .001$ \\
    moment vMF & $.001$--$.025$ \\
    Mix $.5$   & $2.0$--$16$ \\
    Mix $.9$   & $9.2$--$98$ \\
    \multicolumn{2}{@{}l}{\emph{$d=768$}} \\
    info-vMF   & $.000$--$.010$ \\
    \bottomrule
  \end{tabular}
  \vspace{-0.5\baselineskip}
\end{wraptable}

Section~\ref{sec:kernel-path} matched training and sampling. In practice they
can be decoupled: a trained model may be sampled with another kernel or
schedule. We therefore ask how much a posterior trained with one kernel
loses on states from another, and whether the loss comes from kernel shape
or only its noise level. To separate the two, we optionally relabel the
network input using the information clock.

\textbf{Small shape differences are a clock effect.} On GSM8K,
moment-matched vMF changes q-argmax accuracy by only $-0.9\pm0.4$ points at
65 NFEs ($B=64$ plus posterior decoding). An information-matched vMF loses $3.8\pm0.4$ points, but relabelling
by mean cosine reduces the gap to $0.5\pm0.7$
(Table~\ref{tab:gsm8k-kernel-cross-new}). Both kernels have small shape
residuals (Table~\ref{tab:shape-residual-main}), so the difference comes from
the noise parameterization, not the shape.

\textbf{Large shape differences are not.} On hard Sudoku, mismatch cost
follows the shape residual. Relabelling reduces the $\lambda=.9$ loss from
$14.6\pm2.9$ to $9.6\pm0.8$ points but does not help $\lambda=.5$
($4.4\pm0.6\to4.8\pm0.3$; Figure~\ref{fig:mismatch-shape-residual},
Table~\ref{tab:kernel-mismatch-cost}). On GSM8K, mismatch is severe:
Heat-trained weights sampled with either mixture reach at most 0.5\%
accuracy, and relabelling lifts $\lambda=.5$ only from 0.03\% to 8.7\%
(Table~\ref{tab:gsm8k-kernel-cross-new}). The cost grows with sequence
length. The mixture-to-Heat success ratio drops from 0.47 for programs of
mean length 96 to 0.03 at length 189, with correlation $-0.96$
(Figure~\ref{fig:gsm8k-consistency}, right;
Table~\ref{tab:gsm8k-length-new}). This is consistent with early token errors
corrupting the context of later tokens. \looseness=-1

\textbf{Matched training restores accuracy.} We fine-tune the same Heat
checkpoint for 25k steps on Heat and on the $\lambda=.5$ mixture, and sample
each with its own kernel. At 257 NFEs ($B=256$ plus posterior decoding),
the mixture model reaches $26.23\pm1.30\%$, close to the Heat control
($27.32\pm0.52\%$) and far above the
mismatched mixture ($0.03\%$; Figure~\ref{fig:gsm8k-consistency}, left;
Table~\ref{tab:gsm8k-finetune-new}). Each arm uses one training seed; error
bars are over three inference seeds.

\begin{figure}[h]
  \centering

  \includegraphics[width=\linewidth]{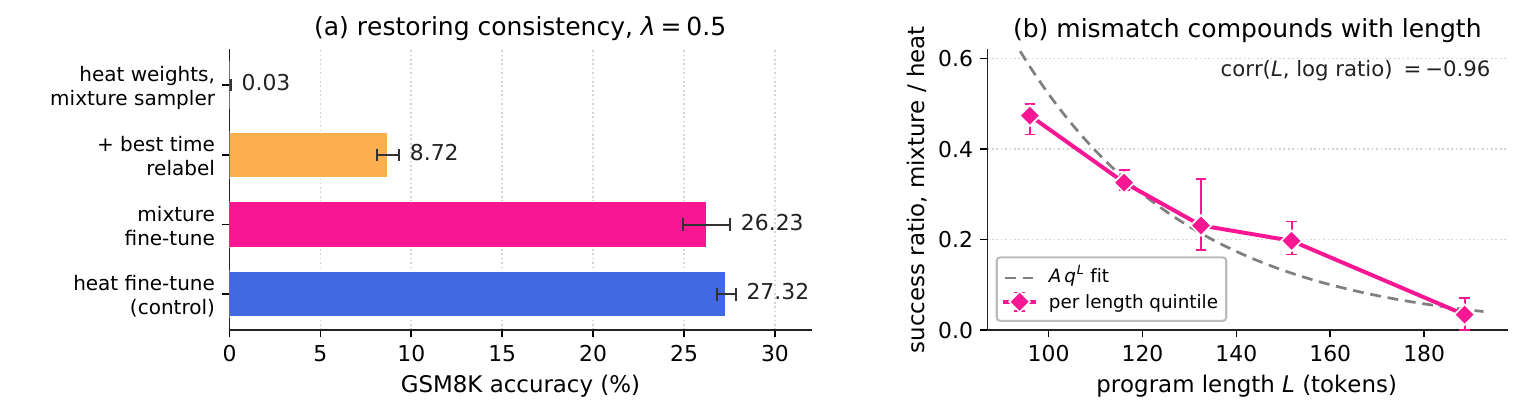}

\caption{GSM8K consistency restoration at $B=256$ (257 NFEs, left) and
mismatch by program length (right), both with posterior decoding. Error bars
show three inference seeds; fine-tunes use one training seed per arm.}
  \label{fig:gsm8k-consistency}

\end{figure}

\subsection{Kernel shape affects categorical refinement}
\label{sec:kernel-refinement}

Matched training removes the mismatch cost, but does the kernel become irrelevant? Not for categorical refinement, where it determines how selected tokens are re-noised. On hard Sudoku, matched mixture training and
sampling improve q-argmax by $6.0\pm1.7$ points for $\lambda=.5$ and
$5.2\pm2.5$ points for $\lambda=.9$. Confidence-gated refinement behaves
differently: the $\lambda=.5$ mixture reduces accuracy by $5.8\pm1.8$ points.
These effects use five independent training seeds per kernel at $B=180$
(181 NFEs with posterior decoding).
In contrast, Heat and moment-matched vMF remain within roughly one point under
tuned gating across the tested budgets. Full kernel results are given in
Appendix~\ref{app:mixvmf-sudoku}, where the q-argmax sweep and crossed
training/re-noising comparison are reported. Table~\ref{tab:five-seed-gated-kernel-effect}
gives the seed-level gating effects, and Table~\ref{tab:heat-vmf-training-seeds}
shows the matched Heat--vMF comparison across inference rules.\looseness=-1

The mixture-specific effect does not transfer directly to GSM8K. In the
single-training-seed text fine-tune, the matched $\lambda=.5$ model differs
from Heat by $-1.09\pm1.80$ points with EMA and posterior decoding at
257 NFEs (paired over three inference seeds). Separately, the S-FLM transfer result
shows that categorical refinement can carry across a different spherical
forward process. Kernel shape therefore has no task-independent ordering:
its effect depends on how the kernel is used and on the inference procedure.

\subsection{Ancestral bridges separate memory from token selection}
\label{sec:bridge}

\textbf{Why separate the two.} Changing the kernel moves accuracy by only a few points once training and sampling match. With Heat fixed, four inference rules span about 25 points on hard Sudoku. However, q-argmax changes both state memory, through fresh re-noising, and token selection, by replacing posterior sampling with argmax. Either could explain its advantage, and they call for different fixes: decorrelating successive network inputs, or lowering the decoding temperature. Because an exact posterior would make ancestral sampling exact, any gain from departing from it also measures how errors of the learned posterior compound along a
trajectory. \looseness=-1

\textbf{How we vary them.} The Heat bridge (Equation~\ref{eq:heat-ancestral-bridge}) has two parameters. The restart size $\Delta$ adds extra noise before bridging back: $\Delta=0$ is exact ancestral sampling, which keeps the full state, and $\Delta\to\infty$ forgets it, as q-argmax does. The temperature $T$ moves token choice from sampling ($T=1$) to argmax. Neither parameter adds network calls.

\textbf{Memory helps on both tasks; argmax helps on GSM8K.}
Removing memory gains $17.7\pm1.4$ points on Sudoku at 129 NFEs
($B=128$ plus posterior decoding). In contrast, memoryless $T=1$ sampling
is slightly better than argmax there, by $0.5$--$1.4$ points across the
reported budgets (Table~\ref{tab:memory-temperature-effects}). On GSM8K,
removing memory gains $8.1\pm0.6$ points at $B=64$, while argmax selection
adds $11.4\pm1.0$ points at $B=16$ but only $4.7\pm2.0$ at $B=256$.

\begin{wraptable}[9]{r}{0.35\linewidth}
  \vspace{-2em}
  \centering
  \captionsetup{font=footnotesize}
  \caption{Paired q-argmax gains from removing state memory or using argmax token selection}
  \label{tab:memory-temperature-effects}
  \scriptsize
  \setlength{\tabcolsep}{3pt}
  \begin{tabular}{@{}lrr@{}}
    \toprule
    Task ($B$) & Memory & Token \\
    \midrule
    Sudoku (16)  & $+6.4\pm1.0$  & $-1.3\pm1.0$ \\
    Sudoku (180) & $+18.4\pm1.8$ & $-1.4\pm1.0$ \\
    GSM8K (16)   & $+3.6\pm0.3$  & $+11.4\pm1.0$ \\
    GSM8K (256)  & $+9.5\pm1.4$  & $+4.7\pm2.0$ \\
    \bottomrule
  \end{tabular}
  \vspace{-0.7em}
\end{wraptable}

\textbf{The memory cost grows with budget.} As the grid is refined, the
bridge deficit grows from 6.4 to 18.4 points on Sudoku and from 3.6 to 9.5 on
GSM8K (Table~\ref{tab:bridge-budget-scaling}). The channel differs by task:
redrawing the noise alone recovers most of the Sudoku gap but worsens GSM8K,
where retaining old token commitments matters more
(Tables~\ref{tab:bridge-fresh} and~\ref{tab:gsm8k-bridge-fresh}). This
refines FMLM+'s observation that fresh re-noising avoids repeated rejected
predictions \citep{agarwal2026posteriorrefinement}.

\textbf{Why Heat.} The diagnostic needs a time-homogeneous Markov diffusion
with a tractable semigroup, so that adding $\Delta$ noise and bridging back
preserves the Heat marginals. Among the paths studied here, SLERP gives only
a deterministic bridge, and vMF marginals give no tractable stochastic one.
The spherical Poisson kernel also has a semigroup, but it is a jump process
rather than a diffusion.

\begin{takeaway}
Kernel choice has no sampler-independent ranking. With compatible continuous
transport, all five kernels perform similarly on Sudoku, but a mismatched
posterior is costly, especially on long GSM8K programs. In categorical
refinement, a mixture can help q-argmax through re-noising yet hurt
confidence gating through training. Across both tasks, retaining state
memory limits refinement, and its cost grows with budget.
\end{takeaway}

\section{Benchmarks}
\label{sec:experiments}

\begin{wraptable}[10]{r}{0.48\linewidth}
  \vspace{-2em}
  \centering
  \captionsetup{font=footnotesize}
    \caption{Sudoku whole-grid accuracy (\%) at $B=180$ sampler calls
  (gated SBD: 181 NFEs).
  FMLM+ and SBD: mean $\pm$ std.\ over five sampler seeds.}
  \label{tab:sudoku-main-180}
  \scriptsize
  \setlength{\tabcolsep}{2pt}
  \begin{tabular}{@{}lrrr@{}}
    \toprule
    Model & Easy & Medium & Hard \\
    \midrule
    AR (greedy) & 14.6 & 5.1 & 1.0 \\
    MDLM & 92.0 & 77.1 & 30.2 \\
    Duo & 96.3 & 84.7 & 58.4 \\
    CANDI & 79.3 & 45.9 & 16.7 \\
    FLM & 94.2 & 82.7 & 44.5 \\
    S-FLM & 94.8 & 85.2 & 45.0 \\
    FMLM+ PR & $97.77\pm.13$ & $92.99\pm.33$ & $\underline{81.65\pm.36}$ \\
    \rowcolor{ourrowpink}
    Ours: q-argmax & $\underline{98.69\pm.07}$ & $\underline{94.57\pm.21}$ & $79.70\pm.75$ \\
    \rowcolor{ourrowpink}
    Ours: gated & $\mathbf{99.11\pm.12}$ & $\mathbf{96.66\pm.15}$ & $\mathbf{83.36\pm.42}$ \\
    \bottomrule
  \end{tabular}
  \vspace{-0.6em}
\end{wraptable}

We benchmark SBD on Sudoku, GSM8K, and OpenWebText, spanning reasoning and
open-ended generation. SBD is competitive with recent continuous models, with
the strongest method depending on the inference regime. We use categorical
refinement for reasoning and a hybrid sampler for text (Section~\ref{sec:kernel-matters}).

\subsection{Sudoku}
\label{sec:experiments:sudoku}

\textbf{Experimental setup.}
We use the S-FLM Sudoku benchmark with 48k training and 2k evaluation puzzles
per difficulty \citep{alp2024sudoku,deschenaux2026hyperspherical}. SBD uses a
frozen max-min $d=128$ codebook, with one Heat model trained per difficulty.
We evaluate each SBD sampler with five random seeds and report whole-grid
exact-match accuracy at a sampler budget of $B=180$ calls. Full budget sweeps and baseline
details are given in Appendix~\ref{app:sudoku-protocol}. 

\textbf{Results.}
Table~\ref{tab:sudoku-main-180} reports the main comparison. On hard Sudoku,
gated SBD reaches $83.36\%$, slightly above FMLM+ ($81.65\%$), and is the
strongest reported continuous/flow model at this budget. On easy and medium puzzles,
gated SBD reaches $99.11\%$ and $96.66\%$, respectively. Full budget sweeps are reported in
Appendix~\ref{app:sudoku-protocol}.
\vspace{-1ex}
\subsection{Reasoning on GSM8K}
\label{sec:experiments:gsm8k}
\vspace{-1ex}

\begin{wrapfigure}[13]{r}{0.50\linewidth}
  \vspace{-2em}
  \centering
  \includegraphics[width=\linewidth]{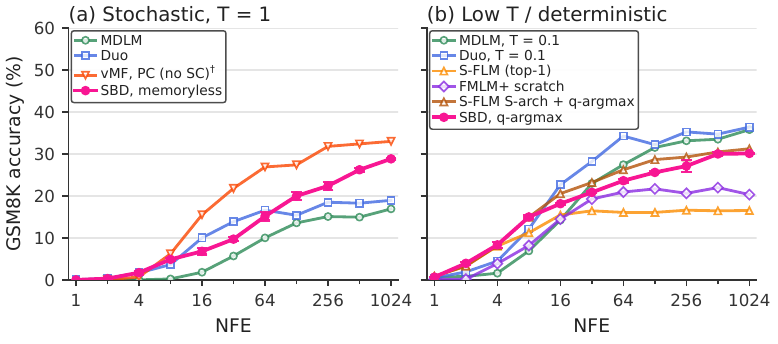}
  \vspace{-2.3em}
\caption{GSM8K accuracy vs.\ NFE. \textbf{(a)} $T=1$. \textbf{(b)}
q-argmax for SBD/S-FLM, $T=0.1$ for MDLM/Duo, endpoint argmax for FMLM+.
$\dagger$: published vMF PC. SBD shows mean $\pm$ std.\ over three seeds;
baselines are single runs.}
  \label{fig:gsm8k-comparison}
\end{wrapfigure}

\textbf{Experimental setup.}
Following prior work, we train on TinyGSM and evaluate on the 1,319 GSM8K test
problems \citep{liu2023tinygsm,cobbe2021trainingverifierssolvemath}. The model
generates one Python program per problem, which we execute to obtain the final
answer. SBD uses a learned $d=768$ codebook and is trained for 250k updates.
We compare against recent discrete and continuous non-autoregressive models
trained from scratch, together with the reported vMF predictor--corrector
baseline without self-conditioning. Teacher-assisted and self-conditioned
variants are reported separately in Appendix~\ref{app:gsm8k-protocol}.
The controlled decoding curves appear in Figure~\ref{fig:gsm8k-comparison};
the full-budget comparison and complete NFE ladder are in Appendix
Tables~\ref{tab:gsm8k-summary} and~\ref{tab:gsm8k-results}.

\textbf{Results.}
At matched $T=1$, SBD beats MDLM at each budget and Duo from 128 NFEs;
the vMF predictor--corrector is stronger, but uses a different training and
sampling recipe. Figure~\ref{fig:gsm8k-comparison} shows both decoding regimes.
Under low-temperature or deterministic decoding, SBD q-argmax leads only
through 8 NFEs; Duo overtakes by 16 and MDLM by 32. At 1,024 NFEs, SBD trails
Duo by about 6 points, narrowing the 20-point Duo--S-FLM top-1 gap. Applying
q-argmax to S-FLM S-arch reaches 31.2\% versus 30.1\% for SBD at 1,024 NFEs
(one evaluation seed); this is transfer evidence, not a controlled kernel
comparison, since schedule, truncation, embedding handling, and training differ.
Full baseline provenance is in Appendix~\ref{app:gsm8k-protocol}.

\subsection{Unconditional OpenWebText generation}
\label{sec:experiments:owt}

\begin{wrapfigure}[14]{r}{0.45\linewidth}
  \vspace{-2.2em}
  \centering
  \includegraphics[width=\linewidth]{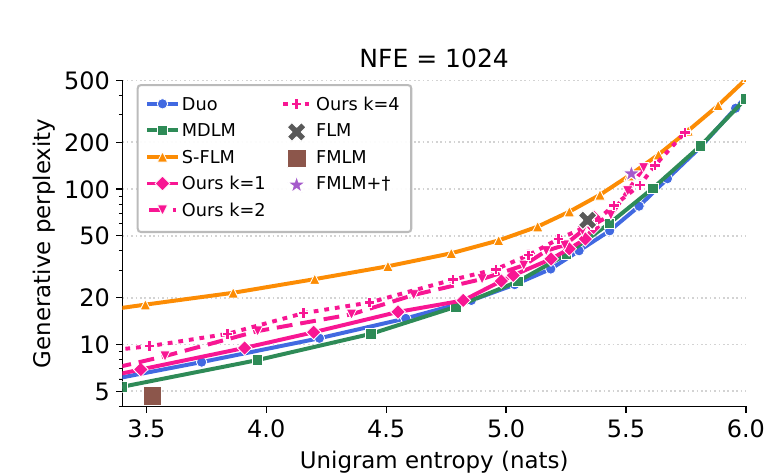}
  \vspace{-1.5em}
  \caption{OpenWebText frontier at 1,024 NFEs. Lower generative perplexity
  is better at matched unigram entropy.}
  \label{fig:owt-generative-frontier}
\end{wrapfigure}

\textbf{Experimental setup.}
We train SBD on OpenWebText with a learned $d=768$ codebook for 1M updates
and generate length-1,024 sequences
\citep{Gokaslan2019OpenWeb,radford2019language}. Following prior work, we
evaluate sample quality with GPT-2-Large generative perplexity and diversity
with unigram entropy, sweeping temperature across NFE budgets.
For open-ended generation, SBD uses a hybrid sampler combining stochastic
score updates followed by categorical refinement. The schedule and
sampler settings are detailed in Appendix~\ref{app:owt-protocol}; the
temperature-matched $k$ ablation is in Table~\ref{tab:owt-k-matched-temperature},
and all NFE frontiers in Figure~\ref{fig:owt-all-budgets}.\looseness=-1

\textbf{Results.}
At 1,024 NFEs (Figure~\ref{fig:owt-generative-frontier}), SBD improves over
S-FLM, while MDLM and Duo retain a better quality--diversity trade-off over
much of the range. Less frequent categorical refinement raises diversity but
lowers quality, so stochastic updates remain useful for open-ended generation.

\section{Related Work}
\label{sec:related-work}

Language diffusion spans discrete models such as MDLM and Duo, and continuous
embedding or simplex approaches including Diffusion-LM, CDCD, RePlaid,
Dirichlet Flow Matching, and LangFlow
\citep{sahoo2024simple,sahoo2025diffusionduality,li2022diffusionlmimprovescontrollabletext,
dieleman2022continuous,yang2026replaid,stark2024dirichlet,chen2026langflow};
CDCD's time warping motivates our information clock. Geometric approaches
include Riemannian heat-kernel and Fisher--Rao models, S-FLM, vMF spherical
flows, and RDLM \citep{bortoli2022riemannian,huang2022riemannian,davis2024fisher,
cheng2024categorical,deschenaux2026hyperspherical,
chemseddine2026spherical,jo2025continuous}. Our conditional quantile map
extends spherical flow matching to multimodal cosine laws
\citep{chemseddine2026spherical}, while q-argmax and confidence gating relate
to posterior refinement and token retention
\citep{li2022diffusionlmimprovescontrollabletext,
agarwal2026posteriorrefinement}. Restart and revision methods likewise modify
sampling dynamics or state retention
\citep{xu2023restart,yu2026revisedontfreezesamplermatched}. Our Heat bridge
gives state-memory interventions a semigroup-based construction; fuller
positioning is in Appendix~\ref{app:add-related-work}.

\section{Conclusion}

We use spherical Brownian motion both as a language model and as a tool for
studying continuous language generation. Four lessons emerge.
(i)~The isotropic kernel matters little when training and sampling are matched;
mismatch is severe but largely repaired by matched fine-tuning.
(ii)~Inference matters more: with Heat fixed, samplers span 25 points, and
posterior-based token refinement outperforms continuous integration, with gains
that transfer to S-FLM.
(iii)~Refinement benefits from fresh states: retaining trajectory memory
preserves noise and token commitments, with a growing cost at finer step
grids.
(iv)~The best sampler is task-dependent: token commitment helps reasoning,
whereas open-ended text benefits from stochastic updates that preserve
diversity. Progress may therefore depend less on new forward processes than
on inference rules that decorrelate successive states and commit to tokens at
the right time.

\subsection*{AI use statement}

In this work, we used generative AI tools to assist with manuscript editing, software code and figure preparation, and to independently check mathematical derivations and proofs.

Generative AI was not used to generate synthetic data sets, help develop theoretical models or conceptual frameworks, formulate mathematical claims, provide critical ingredients for proving mathematical claims, assist in the writing of proofs, propose or refine hypotheses, design or provide feedback on research methodology or experiments, implement methods, or interpret results. The following tasks were not applicable to this work: assist with translation; clean and reformat dataset; support qualitative and thematic data analysis.

All AI-assisted text, code, figures, and mathematical checks were reviewed and verified by the authors. In particular, the authors independently checked the final mathematical derivations and proofs and reviewed and tested the resulting code for correctness. We take responsibility for the final content of this work, including all text, claims, code, figures, and other artifacts produced with the aid of generative AI.

\subsection*{Ethics statement}

This work studies a generative modeling method using Sudoku, synthetic mathematical problems, and OpenWebText. Models trained on web text can
inherit biases and produce inaccurate or harmful content. Generated programs should be executed in a restricted environment. These experiments do not establish that the models are safe for deployment.

\subsection*{Reproducibility statement}

We report benchmark-specific model architectures, training budgets, corruption
schedules, sampling rules, NFE accounting, and evaluation protocols. The
appendices document data splits, checkpoint-selection procedures, seed counts
and baseline provenance; where we report variability, we specify whether it
is over training or sampler seeds.

\bibliographystyle{plainnat}
\bibliography{iclr2027_conference}

\clearpage
\appendix
\let\addcontentsline\appendixaddcontentsline
\setcounter{tocdepth}{2}
\renewcommand{\contentsname}{Appendix contents}
\begingroup
\hypersetup{hidelinks}
\tableofcontents
\endgroup
\clearpage

\section{Limitations}

Our experiments focus on a set of controlled settings designed to isolate the effects of corruption geometry and inference. The kernel study uses hard Sudoku with a fixed \(d=128\) codebook and 20k training updates, providing a focused setting for studying categorical shape effects. The text mixture experiment uses a single-seed 25k-step fine-tuning run rather than from-scratch training. Benchmark error bars reflect sampler variation on fixed checkpoints, while training-seed variation is reported for selected kernel experiments. The crossed evaluations further suggest that the observed gating behavior is encoded in the mixture-trained weights, although a more detailed study of its calibration remains open.

The theoretical analysis focuses on the critical-noise asymptotic regime. While the conditional cosine quantile map is exact, converting it to a spherical displacement and averaging under the posterior introduces a first-order approximation. At larger GSM8K budgets, low-temperature MDLM and Duo remain strong baselines, while SBD shows smaller gains from deterministic token selection. Our restart-size (\(\Delta\)) analyses also target representative settings: the GSM8K sweep uses a fixed NFE budget, and the Sudoku sweep is performed on the tuning split.

\section{Additional related work}
\label{app:add-related-work}

\paragraph{Geometric and continuous models.}
Several works extend diffusion and flow-based generative modeling to non-Euclidean spaces.
Riemannian score models use heat-kernel corruption on manifolds
\citep{bortoli2022riemannian,huang2022riemannian}. RDLM is a continuous
language diffusion model on a hyperspherical statistical manifold and uses
Brownian motion within logarithm-bridge processes
\citep{jo2025continuous}. In contrast, SBD directly uses the spherical heat
semigroup as the forward corruption process for token embeddings.

Spherical flows are especially close to our setting. They use vMF paths and
derive the corresponding dynamics through a radial continuity equation and a
one-dimensional transport construction \citep{chemseddine2026spherical}.
Our quantile-map construction follows the same radial view, but applies it to
more general isotropic spherical kernels, including the bimodal mixtures used
in our experiments. Related ideas also appear in work on learned
one-dimensional processes, flow-map matching, and consistency models
\citep{chemseddine2025adapting,boffi2024flowmap,song2023consistency}.

Other continuous language models study different ways of defining the noise
process or the representation space. CDCD uses time warping to distribute
training effort across noise levels, while LangFlow studies
information-uniform schedules and an ODE-based NLL bound
\citep{dieleman2022continuous,chen2026langflow}. Continuous and flow-based
language models also include TESS, TESS 2, self-conditioned embedding
diffusion, latent language diffusion, Categorical Flow Maps, Discrete Flow
Maps, and ELF
\citep{mahabadi2023tess,tae2025tess2,strudel2022selfconditioned,
lovelace2022latent,roos2026categorical,potaptchik2026discrete,hu2026elf}.
Fisher Flow Matching and Statistical Flow Matching instead use Fisher--Rao
geometry for categorical distributions, while SSD-LM and Dirichlet Flow
Matching use simplex- or Dirichlet-valued representations
\citep{davis2024fisher,cheng2024categorical,han2022ssdlm,stark2024dirichlet}.
Our backbone follows DiT and the text-diffusion architecture used by SEDD.
Analog Bits is an earlier example of self-conditioning for diffusion over
discrete data \citep{peebles2023dit,lou2024sedd,chen2022analogbits}.

\paragraph{Discrete refinement and memory.}
SBD is also related to discrete diffusion and iterative token-refinement
methods. These include D3PM, continuous-time discrete denoising, SEDD, MD4,
MDLM, and scaled masked language models
\citep{austin2021structured,campbell2022continuous,lou2024sedd,
shi2024md4,sahoo2024simple,nie2025scaling}. Partition Generative Modeling is
a related any-order approach
\citep{deschenaux2025partitiongenerativemodelingmasked}. Confidence-based
token ordering has also been explored in MaskGIT, reparameterized discrete
diffusion, and adaptive token-ordering methods
\citep{chang2022maskgit,zheng2023reparameterized,ye2024beyond,
kim2025trainworst}. For Sudoku, recent work studies self-correction and
entropy-bounded unmasking
\citep{kim2025selfcorrection,benhamu2025entropy}.

Several methods repeatedly move between continuous or noisy states and token
predictions. Diffusion-LM projects embedding trajectories back toward token
embeddings, FMLM+ accepts endpoints produced by learned flow maps, and ReMDM alternates between token predictions and remasking. Restart and revision methods similarly modify how information from
earlier sampler states is retained
\citep{li2022diffusionlmimprovescontrollabletext,
agarwal2026posteriorrefinement,sahoo2024simple,xu2023restart,
wang2025remasking,yu2026revisedontfreezesamplermatched}. Our Heat bridge is
useful for studying this effect because it lets us vary state memory while
keeping the token-selection rule fixed.

This is different from carrying denoiser features across sampling steps, as
in Loopholing \citep{jo2025loopholing}. Fresh-noise interventions are also
related to stochastic churn in EDM \citep{karras2022edm}. FMLM+ reports that
reusing the initial noise when re-noising cells lowers GSM8K accuracy from
$19.0\%$ to $14.4\%$ at 32 NFEs
\citep{agarwal2026posteriorrefinement}. Our bridge experiments separate this
noise-persistence effect from persistence of the previously selected token.

Finally, sampler comparisons depend strongly on the evaluation regime.
Generative perplexity can reward low-diversity samples: time-agnostic
analyses of masked models trace misleadingly low generative perplexity to
an effective temperature reduction caused by inaccurate categorical sampling
\citep{zheng2024timeagnostic}, and generative frontiers make the trade-off
explicit by treating generative perplexity and entropy as two components of
a divergence to a reference distribution
\citep{pynadath2026generativefrontiersevaluationmatters}. Perplexity can
also assign low values to incorrect outputs and need not select the more
accurate model \citep{velickovic2026perplexity}. Verifiable tasks such as
Sudoku and GSM8K sidestep both issues by providing direct correctness
signals, complementing entropy-aware evaluation of unconditional text. The
Diffusion Duality series similarly studies the interaction between training
objectives and inference procedures for uniform-state diffusion, including
few-step generation and predictor--corrector sampling
\citep{sahoo2025diffusionduality,deschenaux2026dualitytwo}. These works
motivate our emphasis on treating the forward process and the inference rule
as separate design choices rather than attributing generation behavior to
either component in isolation.

\section{Poisson-kernel reduction for the heat-vMF comparison}
\label{app:heat-vmf-reduction}

This appendix gives the analytic reduction behind
Proposition~\ref{prop:heat-vmf-approximation}.  Throughout, we fix
$\xi\in\mathbb R$ and write
\begin{equation}
  t_d=\frac{\log d+\xi}{d-1},
  \qquad
  r_d=\frac{d-1}{2}t_d,
  \qquad
  m=e^{-\xi/2},
  \qquad
  U=\sqrt d\,C,
  \qquad
  \bar t_d=t_d-\frac{2}{d}.
  \label{eq:appendix-critical-notation}
\end{equation}
Thus $q_d=e^{-r_d}=m/\sqrt d$ and $\operatorname{He}_2(x)=x^2-1$.
All constants below may depend on the
fixed operating point $\xi$, but not on $d$.  Densities are taken relative
to the uniform measure on $\mathbb S^{d-1}$, or equivalently to the uniform
coordinate law of $U$.

The proof has two comparison steps:
$\text{heat}\to\text{Poisson}\to\text{moment-matched vMF}$.
Both leave the same quadratic Hermite mode, so their coefficients can be
added directly. We then obtain KL from the variance of the centered
log-likelihood ratio. Appendix~\ref{app:heat-vmf-technical} gives the
differentiated expansions and tail estimates used in these steps.

\subsection{Heat to Poisson}

Let $P_{n,d}$ be the normalized zonal Gegenbauer polynomial of degree $n$ and
$h_{n,d}$ its multiplicity.  For $D=q\partial_q$, the spherical Poisson
kernel is
\begin{equation}
  P_{q,d}(u)
  =\frac{1-q^2}
  {(1-2qu/\sqrt d+q^2)^{d/2}}
  =\sum_{n\geq0}h_{n,d}q^nP_{n,d}(u/\sqrt d).
  \label{eq:appendix-poisson-kernel}
\end{equation}
The heat coefficient at intrinsic time $t_d$ factors exactly as
\begin{equation}
  \exp\!\left[-\frac12n(n+d-2)t_d\right]
  =q_d^n\exp\!\left[-\frac12n(n-1)t_d\right].
  \label{eq:appendix-heat-coefficient-factorization}
\end{equation}
Consequently,
\begin{equation}
  K^H_{d,r_d}(u/\sqrt d)
  =\left.
   \exp\!\left[-\frac{t_d}{2}D(D-1)\right]P_{q,d}(u)
   \right|_{q=q_d}.
  \label{eq:appendix-exact-poisson-bridge}
\end{equation}
Equation~\ref{eq:appendix-exact-poisson-bridge} is an identity of the
convergent zonal expansions; no asymptotic approximation has yet been made.
\begin{lemma}[Central $C^1$ heat-Poisson estimate]
\label{lem:heat-poisson-central}
Fix $\xi\in\mathbb R$.  There are remainders $R_d^H$ and $S_d^H$ such that,
under the heat law of $U$,
\begin{align}
  \log\frac{K^H_{d,r_d}(U/\sqrt d)}{P_{q_d,d}(U)}
  &=-\frac{m^2t_d}{2}\operatorname{He}_2(U-m)+R_d^H(U),
  \label{eq:appendix-heat-poisson-density}\\
  \partial_u\log\frac{K^H_{d,r_d}(u/\sqrt d)}{P_{q_d,d}(u)}
  &=-m^2t_d(u-m)+S_d^H(u),
  \label{eq:appendix-heat-poisson-score}
\end{align}
with
\begin{align}
  \|R_d^H-\E_HR_d^H\|_{L^2(H)}
  +\|R_d^H-\E_HR_d^H\|_{L^3(H)}
  &=O(t_d^2+t_d/d),
  \label{eq:appendix-heat-poisson-density-bound}\\
  \|S_d^H\|_{L^2(H)}
  &=O(t_d^2+t_d/d).
  \label{eq:appendix-heat-poisson-score-bound}
\end{align}
The corresponding uniform $C^1$ estimate on a growing central window is
given in Equation~\ref{eq:appendix-heat-poisson-central-bound}.
\end{lemma}

\begin{proof}[Proof outline]
Equation~\ref{eq:appendix-exact-poisson-bridge} isolates the small heat
perturbation of an explicit Poisson kernel. Its first-order term is
$-t_d D(D-1)P/2$, and
$D(D-1)P/P=m^2\operatorname{He}_2(u-m)+O((1+|u|^4)/d)$.
An exact complex-Gaussian representation controls the Taylor remainder
before and after one $u$ derivative. Sub-Gaussian coordinate tails and
global heat-kernel bounds then give the displayed heat-weighted norms.
The full proof is in Appendix~\ref{app:heat-poisson-proof}; this argument
does not differentiate an uncontrolled density approximation.
\end{proof}

\subsection{Poisson to moment-matched vMF}

\begin{lemma}[Moment-matched Poisson-vMF expansion]
\label{lem:poisson-vmf}
At $q_d=m/\sqrt d$, on each growing central window
$|u-m|\leq A\sqrt{\log d}$ with fixed $A$, the log-density ratio and its
derivative satisfy
\begin{align}
  \log\frac{P_{q_d,d}(u)}
  {K^V_{d,\kappa_d^\star}(u/\sqrt d)}
  &=\frac{m^2}{d}\operatorname{He}_2(u-m)
    +O\!\left(\frac{1+|u|^6}{d^2}\right),
  \label{eq:appendix-poisson-vmf-density}\\
  \partial_u\log\frac{P_{q_d,d}(u)}
  {K^V_{d,\kappa_d^\star}(u/\sqrt d)}
  &=\frac{2m^2}{d}(u-m)
    +O\!\left(\frac{1+|u|^5}{d^2}\right).
  \label{eq:appendix-poisson-vmf-score}
\end{align}
The density remainder is $O(d^{-2})$ in $L^2(H)$ and $L^3(H)$,
and the derivative remainder is $O(d^{-2})$ in $L^2(H)$.
\end{lemma}

\begin{proof}
The vMF coordinate law is an exponential tilt of the uniform spherical
coordinate law. Its log-partition function satisfies
$\psi_d(\lambda)=\lambda^2/2-\lambda^4/(4d)+O(d^{-2})$.
Matching the heat mean, $\psi_d'(\lambda_d^\star)=m$, gives
$\lambda_d^\star=m+m^3/d+O(d^{-2})$.
Expanding the explicit Poisson density and this matched vMF density, then
subtracting, leaves $m^2\operatorname{He}_2(u-m)/d$.
Differentiating the same explicit formulas gives the score term.
Appendix~\ref{app:poisson-vmf-calibration} records the expansions and
the tail argument that makes their remainders valid in the stated norms.
\end{proof}

\subsection{From the two comparisons to score and KL}

\begin{proof}[Proof of Proposition~\ref{prop:heat-vmf-approximation}]

Combining Lemma~\ref{lem:heat-poisson-central} with
Lemma~\ref{lem:poisson-vmf} gives
\begin{equation}
  L_d(U)
  :=\log\frac{K^H_{d,r_d}(U/\sqrt d)}
  {K^V_{d,\kappa_d^\star}(U/\sqrt d)}
  =-\frac{m^2\bar t_d}{2}\operatorname{He}_2(U-m)+\rho_d(U),
  \label{eq:appendix-density-residual}
\end{equation}
where
\begin{equation}
  \|\rho_d-\E_H\rho_d\|_{L^2(H)}
  +\|\rho_d-\E_H\rho_d\|_{L^3(H)}
  =O(t_d^2+t_d/d+d^{-2}).
  \label{eq:appendix-density-remainder}
\end{equation}
The score expansion is obtained independently from the differentiated
estimate:
\begin{equation}
  \frac{a^H_{d,r_d}(U/\sqrt d)-\kappa_d^\star}{\sqrt d}
  =-m^2\bar t_d(U-m)+\sigma_d(U),
  \qquad
  \|\sigma_d\|_{L^2(H)}
  =O(t_d^2+t_d/d+d^{-2}).
  \label{eq:appendix-score-residual}
\end{equation}
Here we used
$\partial_u\log K(u/\sqrt d)=a(u/\sqrt d)/\sqrt d$.
Equation~\ref{eq:appendix-vmf-concentration-expansion} then gives
\begin{equation}
  \frac{a^H_{d,r_d}(U/\sqrt d)-\kappa_d^\star}
  {\kappa_d^\star}
  =-m\bar t_d(U-m)
   +O_{L^2(H)}(t_d^2+t_d/d+d^{-2}).
  \label{eq:appendix-relative-score-refined}
\end{equation}
The explicit $2/d$ term is lower order than $t_d$, but it determines the
leading finite-dimensional correction along the critical path. Absorbing
this term into the remainder gives the simpler likelihood and score
expansions in Proposition~\ref{prop:heat-vmf-approximation}.

Let $X_d=L_d-\E_HL_d$. Define the centered quadratic mode and residual by
\begin{equation}
  H_{2,d}=\operatorname{He}_2(U-m)
  -\E_H\operatorname{He}_2(U-m),
  \qquad
  \eta_d=\rho_d-\E_H\rho_d.
\end{equation}
Then
\begin{equation}
  X_d=-\frac{m^2\bar t_d}{2}H_{2,d}+\eta_d.
\end{equation}
Equations~\ref{eq:appendix-density-remainder} and
\ref{eq:appendix-hermite-variance} give
\begin{equation}
  \Var_H(X_d)
  =\frac{m^4}{2}\bar t_d^2
   +O(t_d^3+t_d^2/d).
  \label{eq:appendix-centered-variance}
\end{equation}
Because $L_d=\log(p_H/p_V)$,
\begin{equation}
  \E_H[e^{-L_d}]=1,
  \qquad
  \KL(H\Vert V)=\log\E_H[e^{-X_d}].
  \label{eq:appendix-kl-centered}
\end{equation}
The uniform Taylor expansion on
$G_d=\{|U-m|\leq A\sqrt{\log d}\}$, together with the exponentially weighted
tail control in Appendix~\ref{app:likelihood-tail-control}, therefore yields
\begin{equation}
  \KL(H_{d,r_d}\Vert V_{d,\kappa_d^\star})
  =\frac{m^4}{4}\left(t_d-\frac{2}{d}\right)^2
   +O(t_d^3+t_d^2/d).
  \label{eq:appendix-kl-error}
\end{equation}
In particular,
\begin{equation}
  \KL(H_{d,r_d}\Vert V_{d,\kappa_d^\star})
  =\frac{m^4}{4}t_d^2
   \left[1+O((\log d)^{-1})\right],
\end{equation}
which proves Equation~\ref{eq:heat-vmf-kl-main}. The moment recursion in
Appendix~\ref{app:coordinate-moments} also gives
$\E_H[(U-m)^2]=1+O(t_d+d^{-1})$. Taking the $L^2(H)$ norm in
Equation~\ref{eq:appendix-relative-score-refined} therefore gives the
RMS relative-score error $m t_d(1+o(1))$.
\end{proof}

\begin{takeaway}
Moment matching removes the leading linear mismatch. The remaining
quadratic mode gives an $O(t_d)$ relative score error and an $O(t_d^2)$
conditional KL. The score claim follows from a differentiated estimate, while the KL claim follows from a likelihood-weighted tail bound beyond the Gaussian limit.
\end{takeaway}

\subsection{Finite-dimensional refinement}

Combining Equation~\ref{eq:appendix-relative-score-refined} with the
second-moment estimate used above, and taking the fourth root in
Equation~\ref{eq:appendix-kl-error}, gives the finite-dimensional correction
\begin{align}
  \frac{\widehat m_{\rm score}}{m}
  &=1-\frac{2}{dt_d}+o((\log d)^{-1}),
  \label{eq:appendix-score-finite-d}\\
  \frac{\widehat m_{\rm KL}}{m}
  &=\sqrt{1-\frac{2}{dt_d}}+o((\log d)^{-1}),
  \label{eq:appendix-kl-finite-d}
\end{align}
where $\widehat m_{\rm score}$ is recovered from the RMS relative-score error
divided by $t_d$ (equivalently, from the Gaussian-calibrated
$\operatorname{MAD}/[t_d\sqrt{2/\pi}]$), and
$\widehat m_{\rm KL}=(4\KL/t_d^2)^{1/4}$.  Since
$dt_d=\log d+\xi+o(1)$, the predicted leading corrections are respectively
$-2/(\log d+\xi)$ and $-1/(\log d+\xi)$.  These constants follow from the
explicit $2/d$ term; the numerical sweep supports them but is not used in the
proof.

\subsection{Scope and marginal score fields}

For a codebook $\{\mathbf{e}_x\}$, the heat and vMF marginal scores can be
written as
\begin{align}
  s_H(\mathbf{z})
  &=\sum_x P_H(x\mid\mathbf{z})a_H(c_x)
    (\mathbf{e}_x-c_x\mathbf{z}), \\
  s_V(\mathbf{z})
  &=\kappa\sum_x P_V(x\mid\mathbf{z})
    (\mathbf{e}_x-c_x\mathbf{z}),
\end{align}
where $c_x=\langle\mathbf{z},\mathbf{e}_x\rangle$.  Since
$\|\mathbf{e}_x-c_x\mathbf{z}\|\leq1$,
\begin{equation}
  \|s_H(\mathbf{z})-s_V(\mathbf{z})\|
  \leq
  \E_{P_H(x\mid\mathbf{z})}|a_H(c_x)-\kappa|
  +\kappa\|P_H(\cdot\mid\mathbf{z})-P_V(\cdot\mid\mathbf{z})\|_1.
\end{equation}
Rotational symmetry reduces the expectation in the first term to the
single-channel heat law.  After averaging over the heat marginal, the KL
chain rule bounds the posterior KL by the single-channel KL.  Pinsker's
inequality and Jensen's inequality therefore give
\begin{equation}
  \frac1\kappa\E\|s_H(\mathbf{Z})-s_V(\mathbf{Z})\|
  \leq
  \E_H\left|\frac{a_H(C)}{\kappa}-1\right|
  +\sqrt{2\KL(H\Vert V_\kappa)}.
  \label{eq:appendix-marginal-score-bound}
\end{equation}
At $\kappa=\kappa_d^\star$, both terms are $O(t_d)$ under
Lemma~\ref{lem:heat-poisson-central}.

For context, Figure~\ref{fig:appendix-fixed-r-heat-vmf} shows that the
fixed-radius regime behaves differently from the critical path analyzed
above.

\begin{figure}[t]
  \centering
  \includegraphics[width=\linewidth]{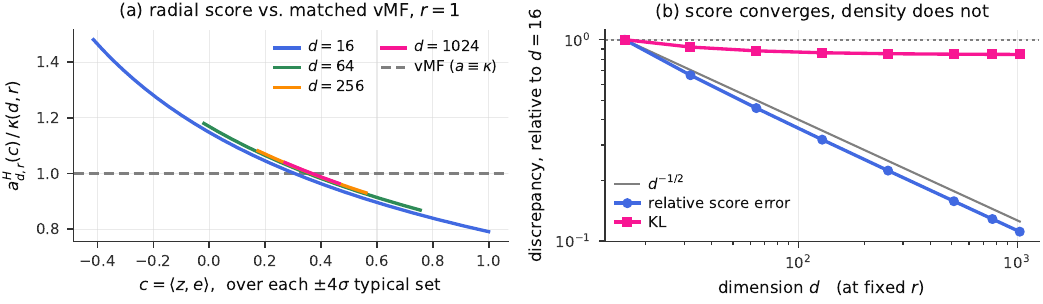}
  \caption{Fixed-noise heat-vMF diagnostic at $r=1$, which is distinct from the
  critical scaling in Proposition~\ref{prop:heat-vmf-approximation}.
  \textbf{Left:} at fixed $r$, the moment-matched vMF radial score becomes a
  better local approximation to the heat score as dimension increases.
  \textbf{Right:} both discrepancies are normalized by their values at
  $d=16$. The RMS relative radial-score error decreases approximately as
  $d^{-1/2}$, while the conditional channel KL remains nearly constant.
  This does not
  contradict Equation~\ref{eq:appendix-kl-error}: at fixed $r$ the retained
  cosine is constant, whereas the proposition follows
  $r_d=(\log d+\xi)/2$, for which the retained cosine is
  $m/\sqrt d$.}
  \label{fig:appendix-fixed-r-heat-vmf}
\end{figure}

\begin{takeaway}
The critical-scale comparison also controls the exact marginal score field
for a fixed codebook. It does not guarantee a learned posterior or better
generation quality. The fixed-$r$ diagnostic shows improving score
agreement without a similar decrease in KL; the theorem applies only to
the critical noise path.
\end{takeaway}
\paragraph{Scope of the critical-path result.}
Proposition~\ref{prop:heat-vmf-approximation} establishes the approximation at
$r_d=(\log d+\xi)/2$, corresponding to the critical regime of the sampling trajectory. At $d=128$
and $\xi=0$, $r_d\approx2.43$, near the high-noise edge of the informative
Heat/vMF band in Table~\ref{tab:kernel-information-clock}; the result is
therefore local. Up to a time change, Heat is the isotropic diffusion for
which the standard Brownian score-reversal dynamics are exact; other kernel
shapes require compatible transport.
\FloatBarrier

\section{Implementation details}
\label{app:implementation-details}

\paragraph{Why Brownian motion is the canonical isotropic diffusion.}
For completeness, consider a nondegenerate, time-homogeneous continuous
diffusion on $\mathbb S^{d-1}$, with $d\geq3$, whose law is invariant under
the action of $\mathrm{SO}(d)$. Its generator has the local form
\begin{equation}
  \mathcal L f
  = b\mathbin{\cdot}\nabla f
  +\frac12\operatorname{tr}\!\left(A\nabla^2 f\right),
  \label{eq:canonical-isotropic-generator}
\end{equation}
where $b$ is a tangent vector field and $A$ is a positive-definite symmetric
tensor on the tangent bundle. Rotation invariance forces $b$ to be an
$\mathrm{SO}(d)$-invariant vector field. At any point, invariance under the
isotropy subgroup $\mathrm{SO}(d-1)$ therefore gives $b=0$. The same isotropy
group acts irreducibly on the tangent space, so $A$ must be a scalar multiple
of the Riemannian metric. Transitivity of $\mathrm{SO}(d)$ makes this scalar
constant over the sphere. Hence
\begin{equation}
  \mathcal L=c\,\Delta_{\mathbb S^{d-1}},\qquad c>0,
  \label{eq:canonical-isotropic-laplacian}
\end{equation}
which is spherical Brownian motion up to a rescaling of time. Its transition
kernels consequently form the heat semigroup.

\begin{table}[!htbp]
  \caption{Comparison with S-FLM
  \citep{deschenaux2026hyperspherical} and vMF spherical flows
  \citep{chemseddine2026spherical}. Brownian motion defines a Markov heat
  semigroup, while the other methods specify a path or concentration schedule.}
  \label{tab:spherical-method-comparison}
  \centering
  \small
  \setlength{\tabcolsep}{3.5pt}
  \begin{tabular}{@{}
    >{\raggedright\arraybackslash}p{0.12\linewidth}
    >{\raggedright\arraybackslash}p{0.23\linewidth}
    >{\raggedright\arraybackslash}p{0.27\linewidth}
    >{\raggedright\arraybackslash}p{0.28\linewidth}@{}}
    \toprule
    Method & Forward family & Noise parameterization & Semigroup / path \\
    \midrule
    S-FLM
      & Geodesic interpolation
      & Interpolation schedule $\alpha_t$
      & Designed path \\
    vMF
      & vMF family
      & Concentration $\kappa_t$
      & Chosen concentration schedule \\
    \rowcolor{ourrowpink}
    \textbf{Ours}
      & Spherical heat kernel
      & Retained-cosine clock $r=-\log\rho$
      & Heat semigroup \\
    \bottomrule
  \end{tabular}
\end{table}

\subsection{Heat-kernel tables}

The score in Equation~\ref{eq:marginal-score} requires
$a^H_{d,r}(c)=\partial_c\log K^H_{d,r}(c)$.  At intrinsic heat time $\tau$,
the kernel has the Gegenbauer expansion
\begin{equation}
  \overline K^H_{d,\tau}(c)
  \propto\sum_{j\geq0}m_{j,d}
  e^{-\lambda_{j,d}\tau/2}P_j^{(d)}(c),
  \qquad \lambda_{j,d}=j(j+d-2),
  \label{eq:heat-gegenbauer-series}
\end{equation}
where $m_{j,d}$ is the dimension of the degree-$j$ spherical-harmonic space.
We scale terms by the largest term while accumulating the series; the relative
size of the final sum then measures numerical cancellation.

The series is not reliable over the full domain.  At $d=128$, it suffers
severe cancellation around $r\in[0.4,1.8]$ in the region containing most
probability mass, while the small-time approximation is already inaccurate.
Splicing the two approximations produces errors of up to 8\% under the
identity
\begin{equation}
  \E[a^H_{d,r}(C)(1-C^2)]=(d-1)\E[C].
  \label{eq:heat-stein-identity}
\end{equation}

We therefore solve the heat equation.  For
$f(c,\tau)=\log\overline K^H_{d,\tau}(c)$,
\begin{equation}
  \partial_\tau f
  =\frac12\left[(1-c^2)(f_{cc}+f_c^2)-(d-1)cf_c\right].
  \label{eq:log-heat-pde}
\end{equation}
The diffusion coefficient vanishes at $c=\pm1$, so both endpoints are outflow
boundaries.  Working with the log density avoids underflow.  The solver starts at $r=0.05$, where the small-time approximation has a
Stein residual of $8\times10^{-5}$, i.e., a relative violation of
$8\times10^{-5}$ in Equation~\ref{eq:heat-stein-identity}, which the exact
heat kernel satisfies. It stores a snapshot at each required radius.

At runtime, $\log a^H_{d,r}(c)$ is bilinearly interpolated from a float32 grid
uniform in $(\log r,c)$.  The same table supplies the radial marginal for
direct forward sampling.  We check agreement with the spectral series where
it is reliable, Equation~\ref{eq:heat-stein-identity}, unit mass, and
insensitivity to the solver grid.

\begin{figure}[!htbp]
  \centering
  \includegraphics[width=\linewidth]{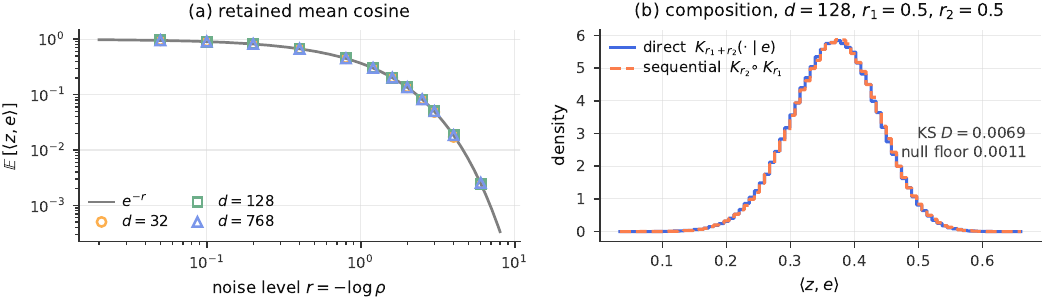}
  \vspace{-2em}
  \caption{Numerical validation of the Brownian clock and forward sampler.
  \textbf{(a)} For $d\in\{32,128,768\}$, forward samples recover
  $\E\langle\mathbf Z_r,\mathbf e\rangle=e^{-r}$ with relative error from $0.002\%$ to $0.1\%$
  wherever the target is above the Monte Carlo resolution. Because
  $r=-\log\rho$, this checks the heat tables and sampler.
  \textbf{(b)} At $d=128$, direct corruption at $r_1+r_2$ and sequential
  corruption at $(r_1,r_2)=(0.5,0.5)$ differ by KS distance $D=0.0069$,
  compared with a null floor of $0.0011$ from two independent direct draws.
  Across all tested dimensions and seven radius pairs, $D\leq0.023$.
  The exact Brownian semigroup identity is therefore only approximate in the
  numerical sampler.}
  \label{fig:semigroup-clock}
\end{figure}

\subsection{Full-vocabulary mixture}

For $N$ positions, Equation~\ref{eq:marginal-score} uses one
$[N\times d][d\times|\mathcal V|]$ product for all cosines, an elementwise
lookup for the radial factors, and one
$[N\times|\mathcal V|][|\mathcal V|\times d]$ product for the weighted
codewords.  We evaluate the exact vocabulary sum without top-$k$ truncation.
The contraction is split into chunks of 2,048 codewords and performed in two
passes, first for the log normalizer and then for the weighted sum, so no
additional dense vocabulary-sized tensor is retained beyond the logits.

\subsection{Codebook construction}
\label{app:codebook-construction}
Frozen and learned codebooks share a max-min initialization.  Starting from
Gaussian vectors projected to the sphere, we minimize
\begin{equation}
  \mathcal L_{\mathrm{code}}
  =\E_v\left[\beta\operatorname{logsumexp}_{w\neq v}
  \left(\frac{\langle\mathbf e_v,\mathbf e_w\rangle}{\beta}\right)\right],
  \label{eq:max-min-codebook}
\end{equation}
annealing $\beta$ from $0.08$ to $0.02$ and renormalizing after every update.

For learned codebooks, we optionally use
\begin{align}
  \mathcal R={}&\lambda_{\mathrm{anchor}}
  \sum_v\|\mathbf e_v-\mathbf a_v\|^2 \\
  &+\lambda_{\mathrm{sep}}\sum_v
  \frac{\operatorname{softplus}\!\left(
  \gamma[\cos(\mathbf e_v,\operatorname{NN}(v))-\mu]
  \right)}{\gamma}.
  \label{eq:codebook-regularizer}
\end{align}
The free configuration uses $\lambda_{\mathrm{anchor}}=0$,
$\lambda_{\mathrm{sep}}=1$, $\mu=0.7$, and $\gamma=20$. Here $\mu$ is a
cosine threshold, equivalent to a nearest-neighbor angle of
$\arccos(0.7)\simeq0.7954$ radians. The separation term
compares a random subset of rows with the complete codebook, and the codebook
learning rate is one tenth of the backbone rate. The separate $d=128$
OWT geometry diagnostic at 150k updates has median displacement
63.7 degrees and median nearest-neighbor angle 0.970 radians
(Appendix~\ref{app:codebook-geometry}). Setting
$\lambda_{\mathrm{anchor}}=3\times10^{-3}$ keeps them within 8 degrees of the
initialization.

\subsection{Adaptive schedules and solver grids}
\label{sec:sampling-schedules}

One training-noise ablation mixes the information clock
$u(r)=H(X_0\mid Z_r)/H(\pi)$ with a log-uniform radius draw:
\begin{equation}
  r\sim0.75\,\mathcal U_{\mathrm{info}}+0.25\,\mathcal U_{\log r}.
  \label{eq:default-noise-draw}
\end{equation}
All positions contribute to the cross-entropy objective; content and
padding losses are normalized separately. These variants focus training
examples or solver points on noise levels
where the posterior changes, without changing the forward process or
training objective.

For fixed-entropy training, we measure
$H_\theta(r)=\E[\mathbb H(q_\theta(\cdot\mid\mathbf z,r))]$ once and draw
\begin{equation}
  q(r)\propto(1-\eta)
  \left|\frac{\mathrm dH_\theta}{\mathrm d\log r}\right|
  +\eta\mathcal U_{\log r}.
  \label{eq:fixed-entropy-draw}
\end{equation}
For entropy-based solver grids, we invert the CDF of the same density at
$S+1$ equally spaced quantiles and use $\eta=0.15$.  The log-uniform component
is needed because the entropy curve is nearly flat at low noise.

For online adaptation, we keep a rolling buffer of $(\log r,\ell)$ pairs:
8,192 for Sudoku, 6,400 for TinyGSM, and 3,200 for OWT.
This loss-slope allocation follows S-FLM's adaptive schedule, itself inspired
by InfoNoise \citep{deschenaux2026hyperspherical,raya2026infonoise}.
After a 1,000-step warm-up, every 500 steps we fit a regularized cubic
smoothing spline with 20 knots and update
\begin{equation}
  q(\log r)\propto
  \max\!\left(\frac{\mathrm d\ell}{\mathrm d\log r},0\right)+\varepsilon.
  \label{eq:online-adaptive-draw}
\end{equation}
A uniform floor preserves coverage (mixing weight $10^{-3}$ for language),
and both the buffer and discretized CDF are checkpointed.

For the same vocabulary, the informative interval $u\in[0.1,0.9]$ occupies
13.9\% of the log-radius range at $d=128$ and 9.3\% at $d=768$.  A log-uniform
draw therefore allocates relatively fewer examples to informative levels as
the dimension grows.

\subsection{Sampling details}

For a decreasing grid $r_0>\cdots>r_S$, let $h=r_k-r_{k+1}$ and
$g^2=2/(d-1)$. The spherical reverse-SDE and probability-flow ODE
updates are
\begin{align}
  \mathbf z&\leftarrow\operatorname{Exp}_{\mathbf z}
  \!\left(g^2s_\theta(\mathbf z,r_k)h+g\sqrt h\,\boldsymbol\xi\right),
  &&\text{(SDE)},\label{eq:reverse-sde-step}\\
  \mathbf z&\leftarrow\operatorname{Exp}_{\mathbf z}
  \!\left(\tfrac12g^2s_\theta(\mathbf z,r_k)h\right),
  &&\text{(ODE)},\label{eq:probability-flow-step}
\end{align}
where $\boldsymbol\xi$ is isotropic tangent noise and the exponential map
keeps every iterate on the sphere.

Prompt positions are fixed at initialization and reset to their clean
codewords after every update. Plain q-argmax independently draws fresh
heat-kernel noise around each selected codeword at every step; it does not
reuse the OWT Gumbel tensor or evaluate the vocabulary score mixture.
A first-order predictor with $c$ corrector steps costs
$1+c$ network evaluations per step, Heun's method costs two, and q-argmax
costs one. We distinguish the sampler budget $B$ from the total NFE:
nearest-codeword or direct token readout uses $B$ calls, while a separate
posterior decode uses $B+1$. A Brownian Heun run takes $B/2$ solver steps;
q-argmax, the bridge, and quantile transport take $B$ steps. Diagnostic
comparisons that hold the grid rather than the call count fixed are identified
separately.

\subsection{Exact ancestral sampling for Heat}
\label{app:heat-ancestral}

The bridge in Equation~\ref{eq:heat-ancestral-bridge} follows from Bayes'
rule and the Chapman--Kolmogorov identity. For a clean token $x$, the joint
density of a bridge state $\mathbf z_s$ and an observed state
$\mathbf y=\mathbf z_R$ is
$K^H_{d,s}(\mathbf z_s\mid\mathbf e_x)
K^H_{d,R-s}(\mathbf y\mid\mathbf z_s)$; division by the marginal
$K^H_{d,R}(\mathbf y\mid\mathbf e_x)$ normalizes it. Drawing
$x\sim p(x\mid\mathbf y,R)$ and then drawing from the bridge is an exact
reverse step when the posterior is exact. In the numerical implementation,
posterior approximation and heat-table interpolation are the sources of
error.

\paragraph{Two-coordinate bridge sampler.}
Let $\eta=\langle\mathbf e_x,\mathbf y\rangle$,
$c_1=\langle\mathbf z_s,\mathbf e_x\rangle$, and
$c_2=\langle\mathbf z_s,\mathbf y\rangle$. The projection of a uniform point
on $\mathbb S^{d-1}$ onto the span of $\mathbf e_x$ and $\mathbf y$ has
density, up to normalization,
\begin{equation}
  J_d(c_1,c_2\mid\eta)
  \propto
  \frac{1}{\sqrt{1-\eta^2}}
  \left[
    1-\frac{c_1^2-2\eta c_1c_2+c_2^2}{1-\eta^2}
  \right]_+^{(d-4)/2}.
  \label{eq:heat-bridge-volume-factor}
\end{equation}
The bridge-coordinate density is proportional to
$K^H_{d,s}(c_1)K^H_{d,R-s}(c_2)J_d(c_1,c_2\mid\eta)$. After drawing
$(c_1,c_2)$, we choose a uniform direction in the orthogonal complement and
reconstruct $\mathbf z_s$. This includes the sphere's volume factor and
reduces the bridge draw to a two-dimensional numerical sampling problem.

\paragraph{Restart and comparison kernels.}
In Equation~\ref{eq:heat-restart-step}, the posterior is evaluated once at the
current state; the selected token is held fixed while the state is re-noised
by $\Delta$ and bridged back. Heat's semigroup means the restarted state has
the correct conditional marginal at radius $r_k+\Delta$. At $\Delta=0$ this
is the ancestral bridge. As $\Delta$ grows, its retained cosine $e^{-\Delta}$
decays; when the restart state is nearly uniform, bridging back approaches
fresh heat re-noising of the selected codeword, the memoryless q-argmax update.
The separate token rule is argmax or sampling from $q_\theta$ at $T=1$. Both
controls add no network calls.

This construction depends on a tractable Markov transition family, not on
isotropy alone. Spherical Brownian motion is the canonical isotropic
continuous diffusion up to time scaling, as shown above. Among the spherical
language-model paths compared here, SLERP is deterministic conditional on
the clean token, so its bridge is a point mass. vMF distributions are not
closed under composition; although vMF marginals can be generated by
transport ODEs or stochastic dynamics, the former is deterministic and the
latter does not provide a tractable transition density for an exact bridge.
Other semigroup families or couplings may permit related bridge constructions. \footnote{The spherical
Poisson kernel also forms a semigroup, but it corresponds to subordinated
Brownian motion, a jump process rather than a continuous diffusion.}

The bridge uses the heat tables described above. For radii below $0.05$, we
use the small-time approximation with Stein residual $8\times10^{-5}$; for
larger radii, the tables are obtained by solving the log heat equation. The
tabulated Chapman--Kolmogorov identity is satisfied to about $10^{-4}$ in
log-kernel values, and the conditional mean cosine follows $e^{-h}$ for an
increment $h$. These density-level checks are distinct from Monte Carlo
composition of the numerical forward sampler, whose KS errors are reported
in Figure~\ref{fig:semigroup-clock}.

\paragraph{Heat configurations for the restart experiments.}
Table~\ref{tab:heat-restart-config} records the checkpoints and evaluation
protocols used below.
\begin{table}[!htbp]
  \caption{Model and evaluation configuration for the Heat restart
  experiments. Sudoku uses $B=128$ sampler calls; posterior decoding in
  the bridge comparison gives 129 total NFEs.}
  \label{tab:heat-restart-config}
  \centering
  \scriptsize
  \setlength{\tabcolsep}{3pt}
  \begin{tabular}{p{0.12\linewidth}p{0.33\linewidth}p{0.22\linewidth}p{0.23\linewidth}}
    \toprule
    Task & Heat model & Evaluation set & Restart protocol \\
    \midrule
    Hard Sudoku & 8-layer DiT, width 512; frozen $d=128$ codebook; 20k updates, batch 256 & 1,000 tuning and 2,000 disjoint evaluation puzzles & $B=128$; seed regimes specified per table \\
    GSM8K & Heat-trained online 250k checkpoint; context length 512 & All 1,319 test problems & $B=64$; three paired inference seeds \\
    \bottomrule
  \end{tabular}
\end{table}

\paragraph{Sudoku bridge versus q-argmax.}
Table~\ref{tab:sudoku-bridge-comparison} compares the ancestral bridge with
q-argmax on the 2,000-puzzle hard-Sudoku evaluation split at $B=128$.
These are the four-training-seed subset used to compare both gated and
ungated variants; the completed five-training-seed memory comparison is in
Tables~\ref{tab:memory-temperature-effects} and~\ref{tab:bridge-budget-scaling}.
Posterior decoding adds
one final network call, for 129 calls total. In the ungated case,
the bridge loses mainly through illegal grids, not by returning legal
solutions that violate puzzle clues.

\begin{table}[!htbp]
  \caption{Hard-Sudoku exact-match accuracy (\%) on the 2,000-puzzle
  evaluation split. Each sampler uses $B=128$ sampling calls plus one final
  posterior-decoding call (129 network calls total). Four paired training
  seeds, with sampler outputs paired within seed; $\Delta$ columns are paired
  differences from q-argmax.
  The main-text memory gain uses the completed five-seed comparison instead.}
  \label{tab:sudoku-bridge-comparison}
  \centering
  \small
  \setlength{\tabcolsep}{3pt}
  \begin{tabular}{lrrrr}
    \toprule
    Sampler & Ungated & $\Delta$ & Gated (tuned) & $\Delta$ \\
    \midrule
    q-argmax & $77.1\pm2.1$ & -- & $82.1\pm1.3$ & -- \\
    Bridge, argmax token & $59.4\pm3.1$ & $-17.8\pm1.6$ & $79.1\pm1.9$ & $-3.0\pm1.1$ \\
    Bridge, sampled token & $60.4\pm2.2$ & $-16.8\pm1.3$ & $79.1\pm1.3$ & $-3.0\pm0.6$ \\
    Memoryless, $T=1$ & $77.5\pm2.5$ & $+0.4\pm0.8$ (n.s.) & -- & -- \\
    \bottomrule
  \end{tabular}
\end{table}

The ungated failure breakdown (valid / legal but wrong puzzle / illegal) is
77.1 / 2.6 / 20.2 for q-argmax and 59.4 / 2.6 / 38.1 for the bridge. The
same clue-violation rate shows that the bridge's lower exact match comes
from illegal grids rather than added clue drift; Table~\ref{tab:sudoku-bridge-comparison}
reports the corresponding exact-match rates.

\paragraph{Restart-size sweep on Sudoku.}
This family was tuned on a separate 1,000-puzzle split with four sampler
seeds, no gate, and $B=128$. Table~\ref{tab:sudoku-restart} shows recovery
as added noise erases more state memory; the tuning split is distinct from
the evaluation split above.

\begin{table}[!htbp]
  \caption{Hard-Sudoku restart sweep on the 1,000-puzzle tuning split at
  $B=128$ sampler calls, four sampler seeds, without gating.}
  \label{tab:sudoku-restart}
  \centering
  \small
  \setlength{\tabcolsep}{6pt}
  \begin{tabular}{rrrr}
    \toprule
    $\Delta$ & Retained cosine $e^{-\Delta}$ & Argmax token & Sampled token ($T=1$) \\
    \midrule
    $0$ (bridge) & $1.00$ & $59.0$ & $60.0$ \\
    $0.5$ & $0.61$ & $75.9$ & $76.7$ \\
    $1$ & $0.37$ & $76.5$ & $77.8$ \\
    $2.43\;(\frac12\log d)$ & $0.09$ & $76.9$ & $78.0$ \\
    $8$ (near-memoryless) & $\approx0$ & $76.6$ & $78.8$ \\
    \bottomrule
  \end{tabular}
\end{table}

\paragraph{Why the ancestral bridge plateaus.}
The following diagnostics use a separate 1,000-puzzle hard-Sudoku tuning
split, three independently trained Heat models (seeds 0--2), and $B=128$.
The bridge diagnostics in Tables~\ref{tab:bridge-exposure} and
\ref{tab:bridge-fresh} use these three training seeds; the budget sweep in
Table~\ref{tab:bridge-budget-scaling} instead uses five sampler seeds on
1,000 hard-Sudoku tuning puzzles, and the
Sudoku restart-size sweep in Table~\ref{tab:sudoku-restart} uses four sampler
seeds on the tuning split. These seed regimes and splits are distinct.
Table~\ref{tab:bridge-exposure} compares each trajectory with a fresh
forward sample of the true solution at the same noise level. The measured
cosine marginal and codeword margin (i.e. cosine to the nearest codeword minus cosine to the second-nearest codeword, averaged over cells) show no \emph{bridge-specific} shift. Yet the bridge makes fewer wrong-to-right and
right-to-wrong changes, halving net revision
(Table~\ref{tab:bridge-fresh}).

\begin{table}[!htbp]
  \caption{Hard-Sudoku trajectory diagnostics on 1,000 tuning puzzles at
  $B=128$, averaged over the informative noise band and three Heat training
  seeds. KS uses cosine to the nearest codeword against forward samples of
  the true solution; the reference codeword margin is $0.115$.}
  \label{tab:bridge-exposure}
  \centering\small
  \setlength{\tabcolsep}{4pt}
  \begin{tabular}{lrrr}
    \toprule
    Diagnostic & Memoryless & Bridge & Restart $\Delta=0.5$ \\
    \midrule
    KS versus forward cosine marginal & 0.004 & 0.006 & 0.005 \\
    Codeword margin & 0.115 & 0.115 & 0.116 \\
    Calibration error on trajectory states & 0.050 & 0.070 & 0.053 \\
    Cell accuracy after $r\approx1.9$ (\%) & 95.0 & 92.7 & 94.6 \\
    \bottomrule
  \end{tabular}
\end{table}

\begin{table}[!htbp]
  \caption{Hard-Sudoku memory intervention on the same 1,000 tuning puzzles,
  three Heat training seeds and $B=128$. Bridge-fresh retains the previous
  token choice but redraws the noise. Flip counts are per cell over the
  informative band; net revision is wrong-to-right minus right-to-wrong.}
  \label{tab:bridge-fresh}
  \centering\small
  \setlength{\tabcolsep}{4pt}
  \begin{tabular}{lrrrr}
    \toprule
    Diagnostic & Memoryless & Bridge & Bridge-fresh & Restart $\Delta=0.5$ \\
    \midrule
    Valid grids (\%) & 74.8 & 58.6 & 73.1 & 74.8 \\
    Wrong$\to$right & 0.291 & 0.168 & 1.390 & 0.273 \\
    Right$\to$wrong & 0.252 & 0.148 & 1.351 & 0.233 \\
    Net revision & 0.040 & 0.020 & 0.039 & 0.040 \\
    Wrong-cell confidence, trajectory / redrawn & .976/.976 & .968/.969 & .975/.975 & .976/.976 \\
    Calibration error, trajectory / redrawn & .052/.052 & .068/.067 & .055/.055 & .052/.052 \\
    \bottomrule
  \end{tabular}
\end{table}

Bridge-fresh recovers $14.5$ of the $16.2$ lost points, about $90\%$, while
retaining commitment memory. Its paired accuracy difference from memoryless
is $-1.70\pm0.44$ points across the three training seeds. Its gross flip
count is about five times the
memoryless count, but its \emph{net} revision and accuracy are close to
memoryless. Re-noising the same committed grid gives almost unchanged
confidence and calibration within each sampler. Thus the calibration gap
mainly reflects which grids were committed, instead of a different reading
of a fixed grid. Confidence is nearly unchanged when evaluated on the same states, suggesting that the calibration gap mainly reflects which grids become committed. The remaining gap measures the commitment-memory cost in this intervention.

\paragraph{Budget scaling on Sudoku and GSM8K.}
Table~\ref{tab:bridge-budget-scaling} uses the 1,000-puzzle Sudoku tuning
split and five sampler seeds, unlike the 2,000-puzzle benchmark in
Table~\ref{tab:sudoku-full-nfe} or the five-training-seed kernel diagonal in
Table~\ref{tab:heat-vmf-training-seeds}. Its $78.0\%$ q-argmax result at
$B=180$ therefore need not equal the $77.9\%$ estimate from the latter
study; the two use different splits and seed regimes.
On Sudoku the bridge is nearly flat
while q-argmax gains $12.9$ points. GSM8K also shows a growing deficit, but
its bridge improves by $3.0$ points. For a Heat step of radius $h$, the old
state retains mean cosine $e^{-h}$; as the grid is refined, successive
readings become more correlated. The heuristic sum
$\sum_k(1-e^{-h_k})\leq\sum_k h_k$ stays bounded by the width of the
informative band, rather than growing with NFE. This scaling argument captures the effective number of posterior readings up to approximation. The slight Sudoku bridge increase, $58.8\to59.5\%$, is compatible with saturation but does not establish the mechanism on its own.

\begin{table}[!htbp]
  \caption{Bridge versus q-argmax across sampler budgets $B$. Sudoku uses
  1,000 hard tuning puzzles and five sampler seeds; GSM8K uses all 1,319 test questions,
  the Heat 250k checkpoint, three paired inference seeds and common prior
  draws. Values are mean accuracy (\%); differences are bridge minus
  q-argmax. The last row uses $B=180$ for Sudoku and $B=256$ for GSM8K.}
  \label{tab:bridge-budget-scaling}
  \centering\scriptsize
  \setlength{\tabcolsep}{4pt}
  \begin{tabular}{r rrr rrr}
    \toprule
    $B$ (Sudoku / GSM8K) & \multicolumn{3}{c}{Sudoku} & \multicolumn{3}{c}{GSM8K} \\
    \cmidrule(lr){2-4}\cmidrule(lr){5-7}
    & Bridge & q-argmax & Difference & Bridge & q-argmax & Difference \\
    \midrule
    16 & 58.8 & 65.1 & $-6.4\pm1.0$ & 14.6 & 18.2 & $-3.6\pm0.3$ \\
    32 & 59.1 & 69.2 & $-10.1$ & 15.3 & 20.9 & $-5.5$ \\
    64 & 58.9 & 73.7 & $-14.8\pm0.8$ & 15.5 & 23.6 & $-8.1\pm0.6$ \\
    128 & 59.3 & 77.0 & $-17.7\pm1.4$ & 16.6 & 25.6 & $-9.0$ \\
    180 / 256 & 59.5 & 78.0 & $-18.4\pm1.8$ & 17.6 & 27.1 & $-9.5\pm1.4$ \\
    \bottomrule
  \end{tabular}
\end{table}

The token-selection differences summarized in the main-text bridge table are
listed at every evaluated budget in Table~\ref{tab:token-selection-effects}.
Negative values mean that sampling tokens at $T=1$ is better than argmax.
The Sudoku rows use the 2,000-puzzle evaluation split and five training
seeds; the GSM8K rows use three paired inference seeds and common prior draws (i.e. the same initial uniform-sphere samples at $r_\text{max}$ are reused across compared samplers for each seed (common random numbers), so differences are paired).

\begin{table}[!htbp]
  \caption{Paired token-selection effect, measured as q-argmax minus
  memoryless sampling at $T=1$ (percentage points; mean $\pm$ sd). Sudoku
  uses 2,000 evaluation puzzles, five training seeds, and posterior decoding
  (NFE $=B+1$). GSM8K uses all 1,319 test questions, three paired inference
  seeds, and common prior draws. Negative values favor $T=1$ sampling.}
  \label{tab:token-selection-effects}
  \centering\small
  \setlength{\tabcolsep}{5pt}
  \begin{tabular}{lrr}
    \toprule
    Task ($B$) & q-argmax $-$ bridge & q-argmax $-$ $T=1$ \\
    \midrule
    Sudoku (16) & $+6.4\pm1.0$ & $-1.3\pm1.0$ \\
    Sudoku (64) & $+14.8\pm0.8$ & $-0.6\pm0.8$ \\
    Sudoku (128) & $+17.7\pm1.4$ & $-0.5\pm0.7$ \\
    Sudoku (180) & $+18.4\pm1.8$ & $-1.4\pm1.0$ \\
    GSM8K (16) & $+3.6\pm0.3$ & $+11.4\pm1.0$ \\
    GSM8K (64) & $+8.1\pm0.6$ & $+8.5\pm1.3$ \\
    GSM8K (256) & $+9.5\pm1.4$ & $+4.7\pm2.0$ \\
    \bottomrule
  \end{tabular}
\end{table}

The memory channels separate differently on the two tasks. At GSM8K
sampler budget $B=64$, bridge-fresh is worse than the exact bridge
(Table~\ref{tab:gsm8k-bridge-fresh}). A single-seed control that forces the
previous token to equal the current choice scores $22.67\%$ against
$23.12\%$ for q-argmax on the same seed, and the bridge-fresh run reproduces $14.10\%$. The argmax changes on about $9.2\%$ of GSM8K
positions per step versus about $0.5\%$ of Sudoku cells, with $5.9$ versus
$0.65$ changes per position over a run. Retaining an old token therefore
matters more for high-churn GSM8K; the Sudoku intervention isolates the contribution of noise persistence to the observed performance gap of the bridge sampler.

\begin{table}[!htbp]
  \caption{GSM8K memory intervention at sampler budget $B=64$ on all 1,319 test
  questions: one Heat 250k checkpoint, three paired inference seeds and
  common prior draws. Differences are paired against q-argmax.}
  \label{tab:gsm8k-bridge-fresh}
  \centering\small
  \begin{tabular}{lrr}
    \toprule
    Sampler & Accuracy (\%) & Difference (points) \\
    \midrule
    q-argmax & 23.63 & -- \\
    Exact bridge & 15.49 & $-8.1\pm0.6$ \\
    Bridge-fresh & 13.34 & $-10.3\pm1.3$ \\
    \bottomrule
  \end{tabular}
\end{table}

\paragraph{Restart-size sweep on GSM8K.}
We use the Heat-trained 250k checkpoint and all 1,319 GSM8K test questions,
with $B=64$ and three paired inference seeds. Table~\ref{tab:gsm8k-restart}
shows that memory loss also helps program generation, with slower recovery
from partial memory loss than on Sudoku.

\begin{table}[!htbp]
  \caption{GSM8K accuracy for the Heat restart sweep at $B=64$. Results use
  the 250k Heat checkpoint and three paired inference seeds.}
  \label{tab:gsm8k-restart}
  \centering
  \small
  \begin{tabular}{rrr}
    \toprule
    $\Delta$ & GSM8K accuracy (\%) & Difference from q-argmax (points) \\
    \midrule
    $0$ (exact ancestral) & $15.49$ & $-8.14\pm0.61$ \\
    $0.5$ & $21.89$ & $-1.74\pm0.39$ \\
    $8$ & $23.60$ & $-0.03\pm0.95$ \\
    $\infty$ (q-argmax) & $23.63$ & -- \\
    \bottomrule
  \end{tabular}
\end{table}

\paragraph{Temperature-controlled GSM8K comparisons.}
Table~\ref{tab:gsm8k-matched-temperature} compares released MDLM and Duo
checkpoints with SBD under matched $T=1$ and low-temperature/deterministic
decoding. Baseline values are single runs; SBD means are over three inference
seeds. Table~\ref{tab:gsm8k-temperature-ladder} gives the longer memoryless
$T=1$ and q-argmax curves used in Figure~\ref{fig:gsm8k-comparison}.

\begin{table}[!htbp]
  \caption{GSM8K accuracy (\%) by NFE and decoding temperature. MDLM and Duo
  values are single runs; SBD values are means over three inference seeds.}
  \label{tab:gsm8k-matched-temperature}
  \centering
  \scriptsize
  \setlength{\tabcolsep}{3pt}
  \begin{tabular}{r|rrr|rrr}
    \toprule
    NFE & Duo $T=1$ & MDLM $T=1$ & SBD memoryless $T=1$
        & Duo $T=.1$ & MDLM $T=.1$ & SBD q-argmax \\
    \midrule
    16  & $10.0$ & $1.8$ & $6.80$ & $22.7$ & $14.3$ & $18.15$ \\
    32  & $13.9$ & $5.7$ & $9.70$ & $28.2$ & $23.0$ & $20.85$ \\
    64  & $16.6$ & $10.0$ & $15.14$ & $34.3$ & $27.5$ & $23.63$ \\
    128 & $15.4$ & $13.5$ & $19.99$ & $32.2$ & $31.5$ & $25.60$ \\
    256 & $18.5$ & $15.1$ & $22.39$ & $35.3$ & $33.1$ & $27.12$ \\
    \bottomrule
  \end{tabular}
\end{table}

\begin{table}[!htbp]
  \caption{GSM8K accuracy (\%) for memoryless SBD at $T=1$ and SBD q-argmax
  across the full evaluated ladder. Means are over three inference seeds.}
  \label{tab:gsm8k-temperature-ladder}
  \centering
  \small
  \begin{tabular}{r|rr}
    \toprule
    NFE & Memoryless, $T=1$ & q-argmax \\
    \midrule
    1 & $0.0$ & $0.7$ \\
    4 & $1.7$ & $8.3$ \\
    16 & $6.8$ & $18.2$ \\
    64 & $15.1$ & $23.6$ \\
    256 & $22.4$ & $27.1$ \\
    512 & $26.2$ & $30.0$ \\
    1,024 & $28.8$ & $30.1$ \\
    \bottomrule
  \end{tabular}
\end{table}

\subsection{Training cost}
\label{app:training-cost}

Table~\ref{tab:training-cost} compares training throughput and total
cost for 250k updates on TinyGSM and 1M updates on OpenWebText.
The AR, MDLM, Duo, CANDI, FLM, and S-FLM values, including the S-FLM
variants, are taken directly from Table~3 of
\citet{deschenaux2026hyperspherical}; we do not remeasure these baselines.
Their measurements use 8 NVIDIA H100 GPUs for TinyGSM and 16 for OWT.
Our measurements use NVIDIA GH200 GPUs, so this is not a
hardware-matched comparison of method efficiency. Baseline totals are
copied as reported, instead of being recomputed from their rounded throughput.

\begin{table}[!htbp]
  \caption{Training cost on TinyGSM (250k updates) and OpenWebText
  (1M updates). Baseline values are from Table~3 of
  \citet{deschenaux2026hyperspherical}. Our duration and GPU-hours
  for TinyGSM refer to the complete run. OWT GPU-hours are implied by the
  reported 4.219 updates/s over 1M updates on 16 GPUs; the full-run wall time
  was not recorded. Throughput is measured over training segments, not
  averaged over the full run. Baselines use H100s and our runs use GH200s.
  NR indicates an unreported full-run duration.}
  \label{tab:training-cost}
  \centering
  \small
  \setlength{\tabcolsep}{4pt}
  \begin{tabular}{lrrrrrr}
    \toprule
    & \multicolumn{2}{c}{Updates/s}
    & \multicolumn{2}{c}{Duration (h)}
    & \multicolumn{2}{c}{GPU-hours} \\
    \cmidrule(lr){2-3}\cmidrule(lr){4-5}\cmidrule(lr){6-7}
    Model & TinyGSM & OWT & TinyGSM & OWT & TinyGSM & OWT \\
    \midrule
    AR & 5.4 & 6.0 & 12.8 & 46.6 & 102.1 & 745.7 \\
    MDLM & 3.0 & 3.5 & 22.8 & 80.3 & 182.7 & 1,284.5 \\
    Duo & 3.0 & 3.5 & 22.8 & 80.3 & 182.7 & 1,284.5 \\
    CANDI & 3.0 & 2.5 & 23.5 & 108.9 & 188.3 & 1,742.9 \\
    FLM & 3.2 & 3.3 & 21.7 & 84.2 & 173.6 & 1,346.8 \\
    S-FLM & 4.2 & 4.5 & 16.5 & 61.7 & 132.3 & 987.7 \\
    S-FLM + adapt. & 4.2 & 4.5 & 16.5 & 61.7 & 132.3 & 987.7 \\
    S-FLM + adapt. + $\mathbb S$-arch
      & 2.6 & 2.8 & 26.7 & 99.2 & 213.7 & 1,587.3 \\
    \addlinespace
    \rowcolor{ourrowpink}
    Ours & 3.465 & 4.219 & 20.0 & NR & 160.3 & 1,053 \\
    \bottomrule
  \end{tabular}
\end{table}
\FloatBarrier

\subsection{Sampling latency and throughput}
\label{app:sampling-cost}

Table~\ref{tab:sampling-cost} reports measurements on the same single
NVIDIA GH200 GPU, using released baseline checkpoints and our 1M-update
checkpoint. All methods generate 1,024-token sequences at $T=1$ where applicable and
256 actual network evaluations, with the call counts independently
verified. Each configuration uses two warmup runs followed by seven
timed repeats. Latency is measured at batch size 1 in seconds per
sample; throughput is measured at batch size 4 in generated tokens per second.
Peak memory is also reported at batch size 4. These values report newly measured inference costs, complementing the training costs reported above.

\begin{table}[!htbp]
  \caption{Sampling cost on one NVIDIA GH200 GPU at 256 NFEs,
  sequence length 1,024, and $T=1$ where applicable. All configurations use two warmups
  and seven timed repeats. Latency uses batch size 1; throughput and peak
  memory use batch size 4.}
  \label{tab:sampling-cost}
  \centering
  \small
  \setlength{\tabcolsep}{5pt}
  \begin{tabular}{@{}lrrr@{}}
    \toprule
    Method & \shortstack{Latency\\(s/sample)}
    & \shortstack{Throughput\\(tokens/s)}
    & \shortstack{Peak memory\\(GiB)} \\
    \midrule
    MDLM & 3.40 & 979 & 3.73 \\
    Duo & 3.21 & 1,259 & 3.73 \\
    S-FLM & 3.62 & 620 & 5.59 \\
    \rowcolor{ourrowpink}
    Ours: q-argmax & 3.37 & 1,221 & 2.31 \\
    \rowcolor{ourrowpink}
    Ours: Heun SDE & 7.54 & 257 & 2.73 \\
    \rowcolor{ourrowpink}
    Ours: OWT hybrid, $k=1$ & 7.03 & 277 & 5.80 \\
    \bottomrule
  \end{tabular}
\end{table}
\FloatBarrier

At matched NFE, q-argmax has latency comparable to MDLM and Duo, with
throughput close to Duo and the lowest measured peak memory. The Heun
SDE and OWT hybrid are slower despite the same number of network calls,
showing that matched NFE need not imply matched wall-clock cost.
These timings characterize inference efficiency under the evaluated settings, with generation quality assessed separately.

\section{Additional experimental details}
\label{app:experiment-details}

\paragraph{Models, comparisons, and validation checks.}
We refer to trained models by their configuration.
Unless stated otherwise, a \emph{learned codebook} is initialized at the
max-min packing (Equation~\ref{eq:max-min-codebook}) and trained jointly with
the backbone, and the \emph{adaptive noise schedule} is the online
loss-adaptive $r$-schedule of Appendix~\ref{sec:sampling-schedules}
(Equation~\ref{eq:online-adaptive-draw}). A learned codebook is
\emph{unanchored} when $\lambda_{\mathrm{anchor}}=0$ in
Equation~\ref{eq:codebook-regularizer}; the anchored variant uses
$\lambda_{\mathrm{anchor}}=3\times10^{-3}$ and is otherwise identical
(Appendix~\ref{app:codebook-construction}). The OpenWebText model
(Appendix~\ref{app:owt-protocol}) is trained for $10^{6}$ updates with an
unanchored learned codebook ($d=768$) and the adaptive noise schedule
(run \texttt{owt\_d768\_online\_free\_lr3e4}; Table~\ref{tab:our-experiment-setup}).
We call a model \emph{matched} when it is sampled with the kernel it was
trained on, and \emph{crossed} when it is sampled with a different kernel
(Section~\ref{sec:kernel-mismatch}; Tables~\ref{tab:kernel-mismatch-cost}
and~\ref{tab:gsm8k-kernel-cross-new}). Paired Heat--vMF comparisons
(Tables~\ref{tab:heat-vmf-training-seeds} and~\ref{tab:heat-vmf-resolution})
use the same seeds, puzzles, and sampler for both training kernels. For the
error-compounding analysis on GSM8K (Appendix~\ref{app:gsm8k-heat-vmf}), we
use two denoising probes that involve no sampling. The \emph{code-token probe}
measures top-1 accuracy of the Heat posterior on program tokens only
(Table~\ref{tab:gsm8k-denoising-probe}). The \emph{full-answer probe}
measures per-token accuracy over the whole answer region (program and
padding) and supplies the per-token rate $p$ for the $p^{L}$ check
(Table~\ref{tab:gsm8k-length-new} and Figure~\ref{fig:gsm8k-consistency}).
We validate every transport with three checks
(Appendix~\ref{app:valid-sampler}). The \emph{Brownian-compatibility index}
(Equation~\ref{eq:raw-t0-l1}, with the $L^2$ variant in
Equation~\ref{eq:raw-t0-l2}) measures how far a kernel family is from being
generated by Brownian motion; it is zero for Heat. The
\emph{conditional-law check} starts the sampler from pure noise with a single
codeword and tests, by a Kolmogorov--Smirnov (KS) test, whether it
reproduces the kernel's conditional law $K_r$. The
\emph{intermediate-marginal check} uses an exact posterior on a 10-codeword
toy problem and tests the marginals at several radii (KS) and the final
token frequencies (total variation).

\subsection{Model and training setup}
\label{app:training-setup}

Our Sudoku experiments use an 8-layer Diffusion Transformer (DiT) with hidden
width 512 and eight attention heads. GSM8K and OWT use a 12-block spectral
DiT with hidden width 768, 12 attention heads, and conditioning dimension
128 \citep{peebles2023dit,lou2024sedd}. The codebook dimension $d$ is separate from
the backbone embedding: we compare $d=128$ and $d=768$ as specified per task.
The posterior network receives the normalized log-radius
$\widetilde r=(\log r-\log r_{\min})/
(\log r_{\max}-\log r_{\min})$.
In every training and sampling run, $r_{\min}=0.02$ and $r_{\max}=8$,
corresponding to retained mean cosines $e^{-0.02}\simeq0.98$ and
$e^{-8}\simeq3.35\times10^{-4}$. The sole exception is the truncated
SLERP ablation, which uses $r_{\min}=0.942$.
Unconditioned positions are initialized uniformly on the sphere. The exact
Heat marginal at $r=8$ is not uniform, but its retained mean cosine is only
$e^{-8}\simeq3.35\times10^{-4}$ and it is numerically close to stationarity.
Uniform initialization therefore introduces a small finite-horizon
approximation instead of changing the limiting process.
We use AdamW with $\beta=(0.9,0.999)$, $\epsilon=10^{-8}$, gradient clipping
at 1.0, and a backbone learning rate of $3\times10^{-4}$. Sudoku uses
weight decay 0.01, a one-cycle schedule with 5\% warm-up and cosine decay,
fp32 training, and no EMA. Language models use zero weight decay, a constant
learning rate after 2,500 warm-up updates, bf16 training, and EMA decay
0.9999; language evaluation uses EMA weights. Dropout is 0.1 during training
and disabled during evaluation. Table~\ref{tab:our-experiment-setup} summarizes
the task-specific settings. These are our training budgets, not the budgets
used to train the released baseline checkpoints.

\begin{table}[!htbp]
  \caption{Main-model training configuration. Sequence length includes the
  prompt in conditional tasks. Language parameter counts separate the backbone
  from the trainable codebook; all models omit self-conditioning.}
  \label{tab:our-experiment-setup}
  \centering
  \small
  \setlength{\tabcolsep}{4pt}
  \begin{tabular}{@{}p{.20\linewidth}p{.22\linewidth}p{.25\linewidth}p{.25\linewidth}@{}}
    \toprule
    Setting & Sudoku & TinyGSM $\to$ GSM8K & OpenWebText \\
    \midrule
    Backbone & 8 layers, width 512, 8 heads & 12 blocks, width 768, 12 heads; conditioning width 128 & Same as GSM8K \\
    Parameters & 26.2M & 130.8M + 37.7M codebook & 131.6M + 38.6M codebook \\
    Codebook $d$ & 128 & 768 & 768 \\
    Tokenizer & Sudoku digits & SmolLM-135M; 49,152 + PAD & GPT-2; 50,257 + 1 \\
    Sequence length & 180 & 512 & 1,024, wrapped \\
    Training updates & 20k & 250k & 1M \\
    Global batch & 256 ($4\times64$) & 512 ($8\times64$) & 512 ($16\times32$) \\
    Weight decay & 0.01 & 0 & 0 \\
    Learning rate & $3\times10^{-4}$; one-cycle & $3\times10^{-4}$; constant after warm-up & Same as GSM8K \\
    LR warm-up & 5\% of updates & 2,500 updates & 2,500 updates \\
    Precision / EMA & fp32 / none & bf16 / 0.9999 & bf16 / 0.9999 \\
    Adaptive buffer & 8,192 & 6,400 & 3,200 \\
    Codebook & Frozen max-min & Learned from max-min anchor; LR $3\times10^{-5}$ & Same as GSM8K \\
    \bottomrule
  \end{tabular}
\end{table}

\paragraph{Codebooks and noise schedules.}
The online loss-adaptive schedule starts after 1,000 updates and is refit
every 500 updates; language runs mix in a uniform component of weight
$10^{-3}$. Learned language codebooks are parameterized around a fixed
max-min anchor but have no anchor penalty. They use a separation penalty
of weight 1.0 and cosine margin 0.7, with a codebook learning rate equal
to one tenth of the backbone rate (Appendix~\ref{app:implementation-details}).
Frozen codebooks retain their max-min initialization. Learned codebooks
are updated jointly with the backbone; their initialization and
regularization are described in Appendix~\ref{app:implementation-details}.
The main models use the online loss-adaptive schedule above. A separate
schedule ablation mixes a uniform information-clock draw with a
log-uniform draw, with weights $0.75$ and $0.25$
(Equation~\ref{eq:default-noise-draw}).
Fixed-entropy and online adaptive schedules instead allocate more training
examples to informative noise levels. The fixed-entropy schedule is
calibrated once, whereas the online schedule is updated from the recent
cross-entropy profile. We distinguish the training noise distribution from
the sampling grid: the \emph{entropy + log sampling} ablation uses an entropy
training schedule and a log-radius solver grid, while \emph{entropy + entropy
sampling} also uses an entropy-based solver grid.

\paragraph{Conditional generation and decoding.}
The five GSM8K checkpoint ablations and Sudoku q-argmax runs use
Equation~\ref{eq:q-argmax-step} at every noise level. These runs have no
score-update phase, sampler switch, or Gumbel perturbation. Each step
selects the most likely clean token and independently re-noises its
embedding at the next radius. Prompt positions remain fixed. The final
token at each position is the most likely token under the predicted
clean-token posterior. Only OWT uses the score/q-argmax hybrid. Argmax
makes each token choice deterministic given the current state, but the
trajectory remains stochastic because initialization and re-noising are
random. Confidence-gated Sudoku sampling also retains accepted
tokens and stops re-noising their positions; unfrozen positions continue
with the same q-argmax rule. We do not use self-conditioning in our models.

\paragraph{Sampling budget and uncertainty.}
We count network evaluations (NFEs) as Transformer forward passes during
generation. A q-argmax step and a first-order ODE or SDE predictor each use
one evaluation; a Heun step uses two. A predictor with $c$ corrector
updates uses $1+c$ evaluations. Kernel-table lookup and heat-kernel
re-noising do not call the Transformer and do not add NFEs. NFE therefore
measures model calls, not total wall-clock cost. Sudoku benchmark error bars
vary five sampler seeds on one trained checkpoint per difficulty, except
the low-budget SDE/ODE entries, which use three. Kernel-study
error bars vary independent training seeds where stated. GSM8K fine-tune
error bars vary three inference seeds on one training seed per arm.

\subsection{Sudoku benchmark protocol}
\label{app:sudoku-protocol}
\label{app:sudoku-commitment}

We use 48k training and 2k evaluation puzzles per difficulty from the S-FLM
generator \citep{alp2024sudoku,deschenaux2026hyperspherical}. Easy, medium,
and hard puzzles reveal 40, 35, and 30 cells. One frozen max-min $d=128$
Heat model is trained for 20k updates per difficulty with the online schedule,
cross-entropy, global batch 256, and a fixed puzzle prefix. All benchmark
entries use dropout-free evaluation and explicit network-call accounting.
The information-grid entries use five sampler seeds on the same weights;
low-budget SDE/ODE instead use the log-radius grid and three sampler seeds,
as indicated in Table~\ref{tab:sudoku-full-nfe}.

Confidence thresholds are selected separately for each difficulty and seed on
1,000 held-out puzzles from a generator seed disjoint from training and
evaluation. Seven candidates span $0.99$ to $0.9999999$. Posterior decoding
uses one last evaluation of $q_\theta(v\mid\mathbf z,r_{\min})$. It costs one
extra call, so gated rows at $B=180$ are charged 181 NFEs. It helps
gating by $0.08$, $0.69$, and $0.77$ points on easy, medium, and hard puzzles
and is neutral for the transport samplers: the largest SDE/ODE gain is
$0.27$ points and the typical gain is $0.03$ points.

Table~\ref{tab:sudoku-full-nfe} gives the full sweep. The solver grid is
chosen per sampler on held-out puzzles; SDE/ODE use the log-radius grid below
$B=64$. True NFE counts network calls, so $B=16$ and $B=32$ correspond to 8 and
16 two-call Heun steps, respectively, while posterior decoding is charged one
additional call.

\begin{table}[!htbp]
\caption{Sudoku exact-match accuracy (\%). The solver grid is chosen per
sampler on held-out puzzles; SDE/ODE use the log-radius grid below $B=64$.
Low-budget SDE/ODE entries average three sampler seeds; all other entries
average five sampler seeds on one checkpoint per difficulty. Error bars do
not include training variance. Column headings are sampler calls $B$;
gated posterior decoding uses $B+1$ total NFEs, and the other rows use $B$.}
\label{tab:sudoku-full-nfe}
\centering\scriptsize
\setlength{\tabcolsep}{3pt}
\begin{tabular}{llrrrrr}
\toprule
Difficulty & Sampler / $B$ & 16 & 32 & 64 & 128 & 180 \\
\midrule
Easy & SDE & $95.32\pm.26$ & $96.93\pm.28$ & $97.20\pm.18$ & $97.32\pm.33$ & $97.10\pm.15$ \\
 & ODE & $96.33\pm.63$ & $96.38\pm.20$ & $96.24\pm.21$ & $96.54\pm.15$ & $96.42\pm.21$ \\
 & q-argmax & $97.87\pm.18$ & $98.24\pm.10$ & $98.50\pm.08$ & $98.68\pm.11$ & $98.69\pm.07$ \\
 & gated q-argmax & $98.85\pm.11$ & $98.94\pm.09$ & $98.96\pm.13$ & $99.14\pm.11$ & $99.11\pm.12$ \\
\addlinespace
Medium & SDE & $84.83\pm.13$ & $88.70\pm.48$ & $89.15\pm.57$ & $89.35\pm.53$ & $89.75\pm.40$ \\
 & ODE & $86.82\pm.38$ & $87.32\pm.95$ & $87.29\pm.25$ & $87.92\pm.22$ & $88.15\pm.57$ \\
 & q-argmax & $91.17\pm.42$ & $92.56\pm.23$ & $93.33\pm.33$ & $94.11\pm.11$ & $94.57\pm.21$ \\
 & gated q-argmax & $94.62\pm.40$ & $95.74\pm.33$ & $96.21\pm.33$ & $96.47\pm.22$ & $96.66\pm.15$ \\
\addlinespace
Hard & SDE & $50.45\pm.41$ & $60.63\pm.92$ & $62.27\pm.73$ & $63.28\pm.93$ & $63.25\pm.36$ \\
 & ODE & $56.02\pm.75$ & $57.95\pm.74$ & $58.50\pm.88$ & $58.72\pm.51$ & $58.24\pm.82$ \\
 & q-argmax & $68.96\pm.76$ & $72.23\pm.81$ & $75.58\pm.17$ & $78.51\pm.36$ & $79.70\pm.75$ \\
 & gated q-argmax & $74.80\pm.95$ & $78.55\pm.53$ & $81.08\pm.43$ & $82.76\pm.84$ & $83.36\pm.42$ \\
\bottomrule
\end{tabular}
\end{table}
\FloatBarrier

\paragraph{Provenance of hard-Sudoku q-argmax results.}
Table~\ref{tab:sudoku-qargmax-provenance} separates four Heat q-argmax
results that should not be treated as repeated measurements of one setting.
Only the benchmark row is used in the main Sudoku comparison. The three-seed
Heat--vMF comparison in Table~\ref{tab:heat-vmf-training-seeds} uses an
identical paired protocol for both training arms, with inference settings
matched within each sampler. The Heat arm used 20k updates and the log-radius grid.
Its $73.88\%$ estimate is $5.82$ points below the benchmark's $79.70\%$,
which used the information grid. This is consistent with the later observed $5.6$-point gain from information spacing for hard-Sudoku q-argmax
at $B=180$. The seed regimes also differ: three training seeds here versus
five sampler seeds on one checkpoint in the benchmark.

\begin{table}[!htbp]
  \caption{Protocol provenance for four Heat q-argmax results (\%) on 2,000
  hard puzzles at a sampler budget of $B=180$ calls. Standard deviations reflect the
  seed type in the last column. Gated runs include a final posterior decode
  (+1 call).}
  \label{tab:sudoku-qargmax-provenance}
  \centering\scriptsize
  \setlength{\tabcolsep}{3pt}
  \begin{tabular}{@{}p{.19\linewidth}r p{.11\linewidth}p{.15\linewidth}p{.14\linewidth}p{.17\linewidth}@{}}
    \toprule
    Study & Accuracy & Training & Grid & Decode & Seeds \\
    \midrule
    Benchmark & $79.70\pm.75$ & online, 20k & information & none, 180 calls & 5 sampler; 1 checkpoint \\
    Heat--vMF contrast & $73.88\pm.41$ & 20k updates & log-radius & q-argmax & 3 paired training seeds \\
    Kernel Heat diagonal & $77.9\pm1.8$ & online, 20k & information-matched & posterior, 181 calls & 5 training seeds \\
    Frozen-codebook ablation & $79.68\pm.42$ & online, 20k & information & posterior, 181 calls & 3 sampler; 1 checkpoint \\
    \bottomrule
  \end{tabular}
\end{table}

At $B=64$, both grids are available for continuous transport.
Table~\ref{tab:sudoku-grid-control} reports five sampler seeds per grid on
the same 2,000 puzzles and one fixed Heat checkpoint per difficulty, with
nearest-codeword decoding (64 total NFEs) and dropout disabled. Differences
are small and have mixed signs: two favor the information grid and four
favor the log grid. The largest is $-0.92$ points for medium/ODE, compared
with a quadrature-combined standard deviation of about $0.54$; without
paired differences we do not assign a paired significance level. In contrast,
the information grid adds $5.6$ points to q-argmax on hard Sudoku at $B=180$.
This supports a larger grid effect for categorical refinement than for the
transport samplers in this control.

\begin{table}[!htbp]
\caption{NFE-64 transport control with nearest-codeword decoding:
exact-match accuracy (\%), mean $\pm$ sd over five sampler seeds per grid
on one fixed Heat checkpoint per difficulty and 2,000 puzzles.}
\label{tab:sudoku-grid-control}
\centering\small
\setlength{\tabcolsep}{5pt}
\begin{tabular}{lrrrr}
\toprule
Difficulty & SDE log & SDE information & ODE log & ODE information \\
\midrule
Easy & $97.10\pm0.41$ & $97.20\pm0.18$ & $96.47\pm0.23$ & $96.24\pm0.21$ \\
Medium & $89.44\pm0.31$ & $89.15\pm0.57$ & $88.21\pm0.48$ & $87.29\pm0.25$ \\
Hard & $62.97\pm0.76$ & $62.27\pm0.73$ & $58.34\pm1.02$ & $58.50\pm0.88$ \\
\bottomrule
\end{tabular}
\end{table}

\paragraph{FMLM+ comparison.}
We evaluate released checkpoints on the same evaluation puzzles with
$p=0.999$, uniform budget, fresh noise, and five sampler seeds.

\begin{table}[!htbp]
\caption{FMLM+ Sudoku exact-match accuracy (\%), mean $\pm$ standard
deviation over five sampler seeds on fixed released checkpoints.}
\label{tab:fmlm-full-nfe}
\centering\small
\begin{tabular}{lrrr}
\toprule NFE & Easy & Medium & Hard \\
\midrule
4 & $97.60\pm.09$ & $91.53\pm.39$ & $70.36\pm.78$ \\
16 & $97.87\pm.27$ & $92.97\pm.54$ & $81.35\pm.55$ \\
64 & $97.87\pm.29$ & $92.99\pm.42$ & $81.44\pm.71$ \\
180 & $97.77\pm.13$ & $92.99\pm.33$ & $\underline{81.65\pm.36}$ \\
\bottomrule
\end{tabular}
\end{table}

\begin{figure}[!htbp]
  \centering
  \includegraphics[width=\linewidth]{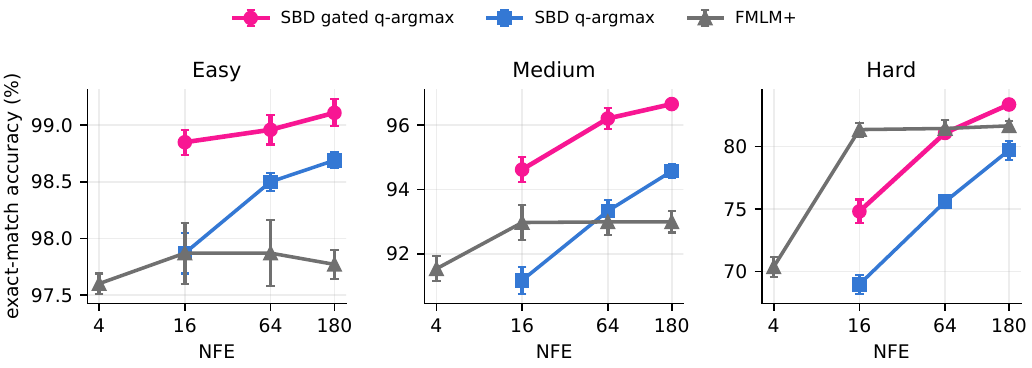}
  \caption{Complete Sudoku accuracy curves for easy, medium, and hard
  puzzles. Error bars are standard deviations over five sampler seeds on
  fixed checkpoints, not training seeds. Comparing gated SBD with FMLM+,
  FMLM+ leads on hard Sudoku at $B=16$, they are tied within variation at
  $B=64$, and SBD leads at $B=180$. The plotted budget labels are sampler
  calls $B$; gated SBD's posterior decode adds one call.}
  \label{fig:sudoku-nfe-all-appendix}
\end{figure}
Figure~\ref{fig:sudoku-nfe-all-appendix} shows the full budget curves behind
the fixed-budget comparison.

\paragraph{Failure decomposition.}
q-argmax produces the most legal grids but can drift from a clue because every
position is re-noised. Gating nearly removes this failure but prevents some
repairs, while posterior decoding recovers part of the lost legality; the
measured breakdown is in Table~\ref{tab:sudoku-failure-new}.

\begin{table}[!htbp]
\caption{Hard-Sudoku failure decomposition (\%) at $B=180$ sampler-call
budget, averaged over five sampler seeds. The posterior-decoded gated row uses
181 actual calls. ``Wrong puzzle'' is legal but inconsistent with a clue.}
\label{tab:sudoku-failure-new}
\centering\scriptsize
\begin{tabular}{lrrrrrr}
\toprule Sampler & Valid & Legal & Wrong puzzle & Rows & Columns & Boxes \\
\midrule
SDE & 63.25 & 68.96 & 5.71 & 75.01 & 75.77 & 81.00 \\
ODE & 58.24 & 62.80 & 4.56 & 70.32 & 71.53 & 77.75 \\
q-argmax & 79.70 & 86.94 & 7.24 & 89.39 & 89.70 & 91.92 \\
gated, nearest & 82.59 & 82.66 & 0.07 & 85.87 & 86.27 & 89.01 \\
gated, posterior & 83.36 & 85.00 & 1.64 & 87.98 & 88.07 & 90.56 \\
\bottomrule
\end{tabular}
\end{table}

\paragraph{Codebook ablation.}
We only consider hard Sudoku here. At $d=128$, frozen
codewords reach $79.68\pm0.42$ with q-argmax and $83.70\pm0.66$ with gating;
learned codewords reach $76.48\pm0.16$ and $81.03\pm0.36$. We omit $d=768$
because every seed selects the tightest tested threshold and commits almost
every cell early; those rows measure posterior saturation rather than
codebook quality. Frozen codebooks suit Sudoku's nine-token vocabulary,
whereas language uses learned codebooks for a much larger vocabulary; this
ablation does not test the language setting.

\subsection{Kernel-consistency study}
\label{app:mixvmf-sudoku}
\label{app:valid-sampler}

\paragraph{Radial transport and the quantile map.}
For an isotropic conditional kernel, let $f_r(c)$ and $F_r(c)$ be the density
and CDF of its cosine coordinate. The conditional radial velocity and its
spherical lift are
\begin{align}
  v_r(c)&=-\frac{\partial_rF_r(c)}{f_r(c)},
  \label{eq:conditional-cdf-velocity}\\
  \mathbf u_r(\mathbf z\mid\mathbf e)
  &=\frac{v_r(c)}{1-c^2}(\mathbf e-c\mathbf z).
  \label{eq:isotropic-spherical-velocity}
\end{align}
The marginal field is the posterior average of these conditional fields. Given
the marginal velocity $\mathbf u_\theta$ and score $\mathbf s_\theta$, the
velocity--score family is
\begin{equation}
  \mathbf z\leftarrow\operatorname{Exp}_{\mathbf z}\!\left(
  h[-\mathbf u_\theta+\varepsilon\mathbf s_\theta]
  +\sqrt{2\varepsilon h}\,\boldsymbol\xi\right),
  \label{eq:isotropic-reverse-family}
\end{equation}
with tangent noise $\boldsymbol\xi$; $\varepsilon=0$ gives the ODE, and
$\varepsilon=1/(d-1)$ recovers the Brownian reverse SDE for Heat. The exact
conditional quantile map in Section~\ref{sec:isotropic-transport} avoids
integrating through low-density gaps between mixture modes. The subsequent
posterior-averaged spherical displacement is first-order.

\begin{table}[!htbp]
  \caption{Hard-Sudoku accuracy (\%) at $B=180$ sampler calls across
  three independent training seeds. The Heat and moment-matched vMF arms use
  an identical paired protocol, with inference settings matched within each
  procedure. Entries are mean $\pm$ standard deviation. Posterior-decoded
  gated rows use 181 actual calls.}
  \label{tab:heat-vmf-training-seeds}
  \centering
  \small
  \setlength{\tabcolsep}{6pt}
  \begin{tabular}{lrrrr}
    \toprule
    Training kernel & ODE & SDE & q-argmax & Gated q-argmax \\
    \midrule
    Heat & $58.42\pm0.44$ & $62.83\pm1.16$ & $73.88\pm0.41$ & $83.42\pm0.32$ \\
    Moment-matched vMF & $58.81\pm0.47$ & $63.11\pm1.14$ & $73.92\pm0.50$ & $82.61\pm0.31$ \\
    $\Delta$ (vMF $-$ heat) & $+0.39$ & $+0.28$ & $+0.04$ & $-0.81$ \\
    \bottomrule
  \end{tabular}
\end{table}

\paragraph{Kernel family.}
For $a\in[0,1)$, let $\operatorname{vMF}_d(a\mid\mathbf e)$ denote the
von Mises--Fisher distribution on $\mathbb S^{d-1}$ centered at $\mathbf e$
with concentration $\kappa_d(a)$ chosen to give mean cosine $a$; equivalently,
$I_{d/2}(\kappa_d(a))/I_{d/2-1}(\kappa_d(a))=a$. For $m=e^{-r}$, define
\begin{equation}
K^{\mathrm{mix}}_{r,\lambda}(\cdot\mid\mathbf e)
=m\,\operatorname{vMF}_d(m+\lambda(1-m)\mid\mathbf e)
+(1-m)\,\operatorname{vMF}_d((1-\lambda)m\mid\mathbf e).
\label{eq:mixvmf-kernel-new}
\end{equation}
The component weights are $m$ and $1-m$, and the corresponding component
mean cosines are $m+\lambda(1-m)$ and $(1-\lambda)m$. Their weighted mean is
$m$, so $\lambda$ changes the conditional shape without changing the retained
mean cosine. At $r=1.6$, the cosine standard deviation is
$0.08$ for Heat, $0.22$ for $\lambda=.5$, and $0.37$ for $\lambda=.9$.

\begin{table}[!htbp]
\caption{Radius at remaining token-entropy fraction $u$. A smaller radius
reveals the token later during generation. The full SLERP path matches Heat
within $|\Delta u|\leq0.016$. This comparison uses the full-path clock; the truncated variant is evaluated separately. Section~\ref{sec:kernel-path} reports SLERP truncated at
$r_{\min}=0.942$; Table~\ref{tab:slerp-full} compares both paths.}
\label{tab:kernel-information-clock}
\centering\small
\begin{tabular}{lrrrr}
\toprule $u$ & Heat & vMF & Mix $.5$ & Mix $.9$ \\
\midrule
.05 & 1.05 & 1.03 & .47 & .056 \\
.25 & 1.31 & 1.31 & .81 & .31 \\
.50 & 1.56 & 1.55 & 1.18 & .73 \\
.90 & 2.39 & 2.38 & 2.51 & 2.30 \\
\multicolumn{5}{l}{SLERP (full): $|\Delta u|\leq0.016$ relative to Heat;}\\
\multicolumn{5}{l}{\quad separate truncation ablation: $r_{\min}=0.942$.}\\
\bottomrule
\end{tabular}
\end{table}

\paragraph{Valid transport checks.}
Write $V=v_r(c)/(1-c^2)$, $S=a_r(c)/(d-1)$, and $w=1-c^2$.
Our primary raw compatibility diagnostic is the tangent-norm $L^1$ ratio
\begin{equation}
 T_0^{(1)}=\frac{\E[|V+S|\sqrt w]}{\E[|V|\sqrt w]}.
 \label{eq:raw-t0-l1}
\end{equation}
Table~\ref{tab:t0-raw-new} estimates it with 20,000 draws per radius.
The $L^2$ diagnostic in Table~\ref{tab:shape-residual-full} instead
uses 50,000 draws per radius and estimates
\begin{equation}
 T_0^{(2)}=\frac{\sqrt{\E[(V+S)^2w]}}{\sqrt{\E[V^2w]}}.
 \label{eq:raw-t0-l2}
\end{equation}
The two estimators differ even on the same samples. We use $T_0^{(1)}$
for raw-compatibility comparisons: the order of magnitude is about $10^{-4}$
for Heat, $10^{-2}$ for moment-matched vMF, and order one for the mixtures.
Raw $T_0$ also counts a pure change of speed, so the fitted residual below
separates time change from shape. Valid-ODE accuracy is compared in
Table~\ref{tab:valid-transport-main}.

\begin{table}[!htbp]
\caption{Raw $T_0^{(1)}$: the $L^1$ estimator in
Equation~\ref{eq:raw-t0-l1}, using 20,000 draws per radius at $d=128$.
The mixture estimates are most informative at the level of order of magnitude because of estimator variability.}
\label{tab:t0-raw-new}
\centering\small
\begin{tabular}{rrrrr}
\toprule $r$ & Heat & moment vMF & mix $.5$ & mix $.9$ \\
\midrule
.05 & $7.2\times10^{-4}$ & $3.6\times10^{-3}$ & 1.00 & 6.49 \\
1.0 & $7.9\times10^{-5}$ & $1.9\times10^{-2}$ & .87 & 3.69 \\
1.6 & $4.6\times10^{-5}$ & $1.3\times10^{-2}$ & .84 & 4.33 \\
3.0 & $3.2\times10^{-5}$ & $9.8\times10^{-3}$ & .80 & 1.41 \\
\bottomrule
\end{tabular}
\end{table}

\paragraph{Tail sensitivity of compatibility estimates.}
The mixture velocity $v_r=-\partial_rF_r/f_r$ can become very large in the
low-density valley between its two modes. On identical samples, repeated
estimates show that the largest 0.1\% of draws account for 84--99\% of
$\E[V^2w]$ for the mixtures (Table~\ref{tab:t0-estimator-reconciliation}).
When these large velocities dominate both numerator and denominator,
$T_0^{(2)}$ is driven toward one and can vary substantially between repeats.
The $L^1$ estimator is more stable, but is not immune to rare draws. The
light-tailed vMF control is stable under both estimators. The shape residual
uses the same second moments and inherits this sensitivity for mixtures:
we interpret its mixture values as orders of magnitude, while the
two-to-four-order gap from vMF supports the qualitative ordering.

\begin{table}[!htbp]
\caption{Compatibility estimators recomputed on identical samples, with
two repeats per setting at $d=128$. The last column is the fraction of
$\E[V^2w]$ contributed by the largest 0.1\% of samples.}
\label{tab:t0-estimator-reconciliation}
\centering\small
\setlength{\tabcolsep}{5pt}
\begin{tabular}{lrrrr}
\toprule Kernel & $r$ & $T_0^{(1)}$, repeats & $T_0^{(2)}$, repeats & Tail share \\
\midrule
vMF & 1.0 & .019 / .019 & .024 / .024 & .001 \\
Mix $.5$ & 1.0 & .87 / .87 & .99 / .99 & .90--.94 \\
Mix $.9$ & .05 & 7.40 / 7.34 & 2.46 / 1.51 & .84--.96 \\
Mix $.9$ & 1.0 & 4.14 / 4.32 & 1.05 / 1.12 & .97--.99 \\
\bottomrule
\end{tabular}
\end{table}

The conditional-law and intermediate-marginal checks intermediate marginals and final token
frequencies. At 360 steps, the quantile ODE reaches KS at most $0.026$ for
Heat, vMF, and $\lambda=.5$, and $0.023$--$0.044$ for $\lambda=.9$, against
null floors of $0.004$--$0.014$. The velocity integrator fails for the
mixtures. The Heat score-ODE arm in the kernel study reaches
$58.42\pm0.44\%$ (Table~\ref{tab:heat-vmf-training-seeds}); the separate
valid-transport study reports $54.5\pm3.2\%$ for Heat's quantile ODE
(Table~\ref{tab:valid-transport-main}; the same group appears in
Table~\ref{tab:slerp-full}). These estimates come from distinct studies and are therefore interpreted as separate measurements rather than a paired accuracy comparison. We also computed a $99.60$--$99.95\%$ puzzle-level agreement between Heat paths.

\begin{table}[!htbp]
\caption{Hard-Sudoku accuracy (\%) using each kernel's valid quantile ODE,
mean $\pm$ sd over three training seeds for every kernel. SLERP uses its full
path here; Table~\ref{tab:slerp-full} also reports the truncated path used in
Section~\ref{sec:kernel-path}. Posterior decoding gives NFE $=B+1$.
We did not run SLERP at $B=720$.}
\label{tab:valid-transport-main}
\centering\small
\begin{tabular}{lrrrrr}
\toprule $B$ & Heat & vMF & SLERP (full) & Mix $.5$ & Mix $.9$ \\
\midrule
180 & $54.5\pm3.2$ & $54.7\pm2.3$ & $56.2\pm2.2$ & $53.3\pm2.4$ & $55.3\pm2.3$ \\
720 & $54.4\pm1.6$ & $55.2\pm2.1$ & -- & $53.2\pm2.6$ & $55.6\pm2.2$ \\
\bottomrule
\end{tabular}
\end{table}

\paragraph{Shape residual.}
At each radius, let $V=v_r(c)/(1-c^2)$, $S=a_r(c)/(d-1)$, and
$w=1-c^2$. Weighted least squares gives
$\alpha^*=-\E[VSw]/\E[V^2w]$ and residual
$\|\alpha^*V+S\|/\|\alpha^*V\|$. The residual is zero exactly when the
local path is Heat after a time change. For SLERP, only the band-averaged
residual is available; its radius-wise range and closest-Heat KS statistic
are not reported in Table~\ref{tab:shape-residual-full}.

\begin{table}[!htbp]
\caption{Compatibility estimates across 13 radii from $.05$ to $5$.
Raw $T_0^{(2)}$ is the $L^2$ estimator with 50,000 draws per radius; Table~\ref{tab:t0-raw-new} reports the corresponding $L^1$-based estimate. Mixture $L^2$ ratios
and residuals are tail-sensitive order-of-magnitude estimates. SLERP's
residual is a band average, not a range.}
\label{tab:shape-residual-full}
\centering\scriptsize
\setlength{\tabcolsep}{3pt}
\begin{tabular}{lrrrrr}
\toprule Kernel & $d$ & Raw $T_0^{(2)}$ ($L^2$) & $\alpha^*$ & Residual & KS \\
\midrule
Heat & 128 & $\le.001$ & 1.000 & $\le.001$ & $\le.012$ \\
Heat & 768 & $\le.007$ & 1.000 & $\le.007$ & $\le.030$ \\
moment vMF & 128 & $.004$--$.025$ & $.993$--$1.012$ & $.001$--$.025$ & $.0001$--$.017$ \\
moment vMF & 768 & $.004$--$.012$ & $.993$--$1.012$ & $.000$--$.010$ & $.0001$--$.022$ \\
info-vMF & 768 & $.004$--$.45$ & $.993$--$1.45$ & $.000$--$.010$ & $.0001$--$.022$ \\
SLERP (S-FLM path) & 128 & -- & $1.003$ & $.047$ & -- \\
mix $.5$ & 128 & $.99$--$1.22$ & $.007$--$.57$ & $2.0$--$16$ & $.007$--$.45$ \\
mix $.5$ & 768 & $1.00$--$1.62$ & $.000$--$.59$ & $1.6$--$77$ & $.007$--$.45$ \\
mix $.9$ & 128 & $1.00$--$2.48$ & $.000$--$.26$ & $9.2$--$98$ & $.007$--$.95$ \\
mix $.9$ & 768 & $1.02$--$10.4$ & $.005$--$1.78$ & $2.4$--$44$ & $.007$--$.95$ \\
\bottomrule
\end{tabular}
\end{table}

\begin{figure}[!htbp]
  \centering
  \includegraphics[width=\linewidth]{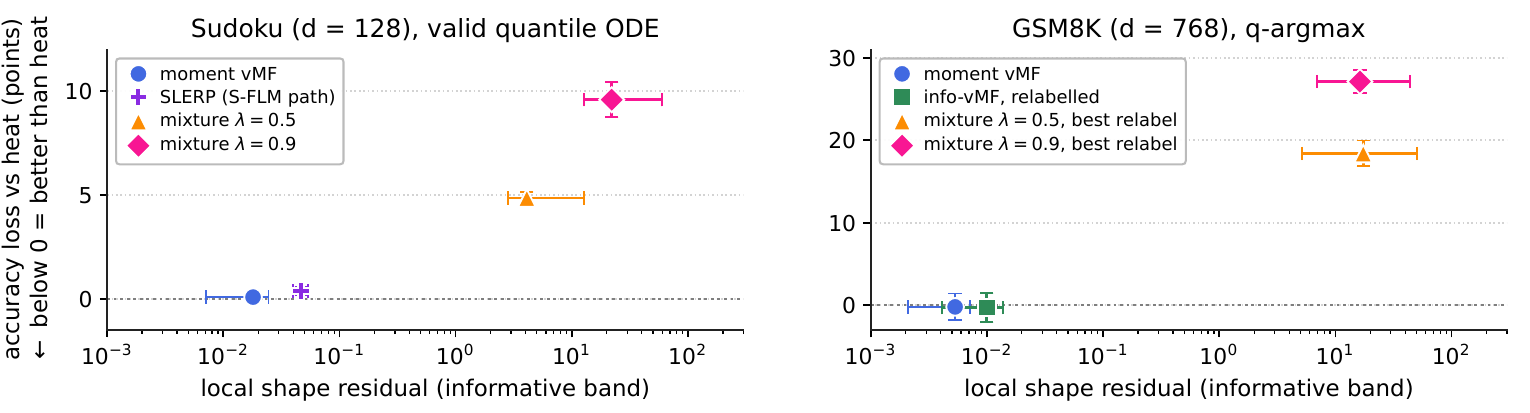}
  \caption{Mismatch cost and the local shape residual. Horizontal bars span
  the residual range over tested noise levels. Sudoku uses the valid quantile
  ODE at $B=180$ with posterior decoding (181 NFEs); only the mixture
  points use information relabelling. vMF and SLERP use their native clocks.
  GSM8K uses q-argmax. The
  residual is local in $r$, so the figure shows ordering instead of a fitted
  relationship. Negative values mean better accuracy than Heat. Differences
  as small as the vMF/SLERP costs are not resolved with three training seeds;
  this does not apply to the larger mixture losses. Mixture residuals are
  tail-sensitive and indicate orders of magnitude. SLERP uses Heat-trained
  weights with the matched geodesic quantile ODE; its native exact velocity
  and radial transport agree numerically (Appendix~\ref{app:slerp-sudoku}).}
  \label{fig:mismatch-shape-residual}
\end{figure}

Figure~\ref{fig:mismatch-shape-residual} orders the measured mismatch costs
by shape residual; Table~\ref{tab:kernel-mismatch-cost} lists the plotted
values.

\begin{table}[!htbp]
\caption{Kernel mismatch costs shown in Figure~\ref{fig:mismatch-shape-residual}
on hard Sudoku. Positive values mean worse accuracy than Heat. Costs are
paired mean $\pm$ sd over three Heat training seeds at $B=180$ with posterior
decoding (181 NFEs). Only mixtures use time relabelling; their residuals
are order-of-magnitude estimates.}
\label{tab:kernel-mismatch-cost}
\centering\small
\begin{tabular}{lrr}
\toprule Kernel & Figure residual & Cost vs. Heat (points) \\
\midrule
Moment-matched vMF & $.017$ & $+.08\pm.08$ \\
SLERP (S-FLM path) & $.047$ & $+.38\pm.20$ \\
Mixture $\lambda=.5$ & $5.5$ & $+4.8\pm.3$ \\
Mixture $\lambda=.9$ & $26.9$ & $+9.6\pm.8$ \\
\bottomrule
\end{tabular}
\end{table}

\paragraph{Dynamics and decoding.}
The misspecified $\lambda=.9$ baseline fails for two reasons. Posterior
decoding fixes the diffuse final component, while quantile transport fixes the
path; Table~\ref{tab:decoding-two-by-two} separates these effects.

\begin{table}[!htbp]
\caption{Hard-Sudoku accuracy (\%) at a sampler budget of $B=180$ network
calls (Brownian ODE: Heun, 90 steps; quantile ODE: 180 steps), mean $\pm$ sd
over three training seeds. Nearest decoding spends 180 NFEs; posterior
decoding adds one call (181 NFEs).}
\label{tab:decoding-two-by-two}
\centering\small
\begin{tabular}{llrr}
\toprule Kernel & Dynamics & Nearest & Posterior \\
\midrule
Mix $.5$ & Brownian ODE & $44.8\pm1.4$ & $55.0\pm2.4$ \\
 & quantile ODE & $52.1\pm2.3$ & $53.3\pm2.4$ \\
Mix $.9$ & Brownian ODE & $0.0\pm0.0$ & $15.3\pm3.6$ \\
 & quantile ODE & $11.2\pm0.7$ & $55.3\pm2.3$ \\
\bottomrule
\end{tabular}
\end{table}

\paragraph{Threshold tuning and information grid.}
Thresholds are selected per kernel and budget on held-out puzzles. At
$B=180$ the chosen values are $0.999995$ for Heat, $0.999999$ for
$\lambda=.5$, and $0.9999995$ for $\lambda=.9$. The information grid blends
measured information progress with log-radius progress at weight $0.15$. At
$B=180$ (181 NFEs with posterior decoding), q-argmax reaches $77.9\pm1.8$, $83.9\pm1.2$, and $83.1\pm1.9$ for
Heat, $\lambda=.5$, and $\lambda=.9$. Tuned gating reaches $82.4\pm1.6$,
$76.7\pm1.4$, and $78.6\pm2.0$. These comparisons use five independently
trained models per kernel, so the deviations measure training-seed variation.
Table~\ref{tab:five-seed-gated-kernel-effect} reports the full gated-kernel
comparison.

\begin{table}[!htbp]
\caption{Effect of the training kernel on tuned confidence-gated refinement
on hard Sudoku. Entries are accuracy differences in percentage points relative
to Heat on the information grid, mean $\pm$ standard deviation over five
independent training seeds per kernel. Parentheses give the paired
$t$ statistic. SLERP is a separate three-training-seed comparison with
paired spreads; see Table~\ref{tab:slerp-full}. Budgets are sampler calls $B$,
with one further call for posterior decoding.}
\label{tab:five-seed-gated-kernel-effect}
\centering\small
\setlength{\tabcolsep}{5pt}
\begin{tabular}{lrrrr}
\toprule
$B$ & Moment-matched vMF & SLERP & Mix $\lambda=.5$ & Mix $\lambda=.9$ \\
\midrule
16  & $+0.3\pm0.8$ & -- & $-4.7\pm2.3\;(t=-4.5)$ & $-5.0\pm4.1\;(t=-2.7)$ \\
64  & $+0.2\pm1.1$ & $+1.1\pm3.6$ & $-5.2\pm2.4\;(t=-4.8)$ & $-5.3\pm3.6\;(t=-3.3)$ \\
180 & $-0.1\pm1.0$ & $+0.1\pm3.4$ & $-5.8\pm1.8\;(t=-7.4)$ & $-3.8\pm3.0\;(t=-2.8)$ \\
\bottomrule
\end{tabular}
\end{table}

Moment-matched vMF is also within $\pm0.5$ points of Heat at every tested
budget on the log grid, including $-0.2\pm1.1$ points at $B=180$, and selects
a similar confidence threshold. On the information grid, the $B=16$
difference of $+0.3\pm0.8$ points gives an approximate two-sided $95\%$
$t$ interval of $[-0.7,1.3]$ points. We therefore describe the two channels as
practically equivalent within roughly one point at that budget, rather than as
exactly equal. The same near-zero pattern holds at $B=64$ and $B=180$, with wider
intervals. vMF's q-argmax result is $0.5$--$0.8$ points above Heat, but is only
borderline at $B=64$ ($t=2.6$) and weaker elsewhere. The evidence is insufficient to establish a vMF advantage. Together with
Proposition~\ref{prop:heat-vmf-approximation}, this shows that moment-matched vMF remains within roughly one point of Heat under gated refinement, unlike the more distant bimodal kernels.

\paragraph{Time-relabelled cross.}
After information relabelling, the valid quantile ODE loses
$4.8\pm0.3$ points for $\lambda=.5$, slightly more than its
$4.4\pm0.6$-point loss before relabelling. For $\lambda=.9$, relabelling
reduces the loss from $14.6\pm2.9$ to $9.6\pm0.8$ points. For q-argmax,
Heat weights with mixture re-noising, relabelling, and the mixture information
grid reach $82.5\pm1.3$ at $\lambda=.5$, compared with $77.9\pm1.8$ for
Heat and $83.9\pm1.2$ for the matched mixture.

\subsubsection{Brownian score-substitution baseline}
This baseline applies Brownian dynamics to substituted mixture scores and uses
nearest-codeword decoding, a natural but invalid choice for a non-Brownian
path. At $B=180$, Heat-trained SDE/ODE under
$\lambda=.9$ both solve $0.0\%$ of puzzles; q-argmax reaches $12.5\pm3.3$
and gating $80.3\pm1.9$. Increasing Heun resolution from 180 to 11,520 NFEs
does not repair the misspecified dynamics. We include these misspecified-dynamics runs as failure-mode baselines.

\paragraph{Reproducibility note.}
Evaluations use checkpoints in evaluation mode, count actual
Transformer calls, keep stochastic noise outside the drift cap (i.e. a clamp on the norm of the deterministic drift (score or velocity) step in the SDE/ODE integrators.), reseed every
sampler, and use a tuning set disjoint from training and evaluation. We also
validated the conditional CDF derivative.

\subsection{SLERP comparison}
\label{app:slerp-sudoku}

We compare full and truncated SLERP paths alongside Heat. Each run uses 20k
updates, global batch size 256, the same eight-layer backbone as the other
Sudoku arms, 2,000 evaluation puzzles, and posterior decoding. Results below
use three independent training seeds. The full SLERP path matches the Heat
information clock within $|\Delta u|\leq0.016$. We also evaluate a truncated
path ending at $r_{\min}=0.942$, where SLERP's own information is exhausted. Using the S-FLM cutoff unchanged at \(d=128\) would leave 77\% of the entropy unresolved.
Table~\ref{tab:slerp-full} gives the accuracy comparison for both paths.

\begin{table}[!htbp]
\caption{SLERP kernel comparison on hard Sudoku. Values are accuracies in
percent, mean $\pm$ sd over three independent training seeds; parenthesized
values are paired differences from Heat, also mean $\pm$ sd. All runs use the
same eight-layer backbone, 20k updates, batch 256, 2,000 evaluation puzzles,
and posterior decoding, so total NFE is $B+1$. The truncated gating
difference at $B=180$ has paired $t=-2.3$ and is not resolved with three seeds.}
\label{tab:slerp-full}
\centering\scriptsize
\setlength{\tabcolsep}{3.5pt}
\begin{tabular}{llrrr}
\toprule
Sampler & $B$ & Heat & SLERP & SLERP truncated \\
\midrule
ODE & 16 & $53.1\pm2.2$ & $55.0\pm1.6\;(+2.0\pm2.0)$ & $51.2\pm2.7\;(-1.8\pm4.8)$ \\
ODE & 180 & $54.5\pm3.2$ & $56.2\pm2.2\;(+1.8\pm3.3)$ & $51.4\pm1.9\;(-3.1\pm4.8)$ \\
q-argmax & 64 & $65.3\pm1.5$ & $67.4\pm1.3\;(+2.1\pm2.6)$ & $66.8\pm1.7\;(+1.5\pm3.2)$ \\
q-argmax & 180 & $71.5\pm2.5$ & $73.1\pm1.3\;(+1.6\pm2.5)$ & $71.9\pm1.6\;(+0.3\pm3.9)$ \\
Gating, tuned & 64 & $78.3\pm2.0$ & $79.5\pm2.1\;(+1.1\pm3.6)$ & $77.7\pm0.9\;(-0.7\pm2.7)$ \\
Gating, tuned & 180 & $81.7\pm1.1$ & $81.8\pm3.0\;(+0.1\pm3.4)$ & $78.4\pm1.4\;(-3.2\pm2.5)$ \\
\bottomrule
\end{tabular}
\end{table}

Full-path SLERP is nominally ahead of Heat in all six cells. Across both
paths, paired differences are at most 3.2 points in magnitude and are not
resolved with three training seeds.
Its information clock is close to Heat's, so it behaves like Heat and vMF
rather than like the bimodal mixtures. The matched SLERP velocity agrees with
S-FLM's exact velocity to $1.4\times10^{-7}$ per step and $4\times10^{-5}$
over a rollout. These checks use SLERP's own model. Separately, the
$+0.38\pm0.20$-point cost in Table~\ref{tab:kernel-mismatch-cost} comes from
Heat-trained checkpoints 0--2 sampled with SLERP's matched geodesic quantile
ODE, paired against the same checkpoints with Heat's valid quantile ODE.
Both arms use $B=180$ and posterior decoding (181 NFEs). This cross uses
the full SLERP path without time relabelling: $\alpha^*=1.003$ and its clock
differs from Heat by at most 0.016. Moment-matched vMF likewise uses its
native clock; only the mixture crosses use information relabelling.

\subsection{Published Sudoku baselines}
\label{app:sudoku-published-baselines}

Table~\ref{tab:sudoku-results-published} reports the published comparison
with autoregressive, discrete-diffusion, and continuous-diffusion Sudoku
baselines. The SBD row uses the gated q-argmax result at $B=180$ (181 total NFEs) from
Table~\ref{tab:sudoku-full-nfe}; its uncertainty is over five sampler seeds
on the fixed checkpoint for each difficulty.
The AR, MDLM, Duo, CANDI, FLM, and S-FLM values in the main comparison
come from their published evaluations \citep{deschenaux2026hyperspherical,
sahoo2024simple,sahoo2025diffusionduality,pynadath2025candi,
lee2026onestep}. FLM is from \citet{lee2026onestep}. The FMLM+ row uses our
five-seed evaluation
of the released checkpoint in Table~\ref{tab:fmlm-full-nfe}
\citep{agarwal2026posteriorrefinement}.

\begin{table}[!htbp]
  \captionsetup{font=footnotesize}
  \caption{Sudoku exact-match accuracy (\%) at $B=180$ sampler-call budget.
  AR, MDLM, Duo, CANDI, FLM, and S-FLM are published values. FMLM+ and SBD
  use five sampler seeds on fixed checkpoints; gated SBD uses 181 calls.
  Bold/underline mark the best/second-best non-AR value per column.}
  \label{tab:sudoku-results-published}
  \centering
  \footnotesize
  \setlength{\tabcolsep}{2.5pt}
  \begin{tabular*}{\linewidth}{@{\extracolsep{\fill}}lrrr@{}}
    \toprule
    Model & Easy & Medium & Hard \\
    \midrule
    \multicolumn{4}{l}{\textit{Autoregressive}} \\
    AR (sample) & 13.9 & 5.1 & 0.6 \\
    AR (greedy) & 14.6 & 5.1 & 1.0 \\
    \addlinespace
    \multicolumn{4}{l}{\textit{Discrete diffusion}} \\
    MDLM & 92.0 & 77.1 & 30.2 \\
    Duo & 96.3 & 84.7 & 58.4 \\
    \addlinespace
    \multicolumn{4}{l}{\textit{Continuous diffusion}} \\
    CANDI & 79.3 & 45.9 & 16.7 \\
    FLM & 94.2 & 82.7 & 44.5 \\
    S-FLM & 81.5 & 50.6 & 14.0 \\
    \quad + $\alpha^{\star}(0.1)$ & 94.0 & 77.6 & 43.2 \\
    \quad + adaptive schedule & 94.8 & 85.2 & 45.0 \\
    FMLM+ PR & $\underline{97.77\pm.13}$ & $\underline{92.99\pm.33}$ & $\underline{81.65\pm.36}$ \\
    \rowcolor{ourrowpink}
    \textbf{Ours} & $\mathbf{99.11\pm.12}$ & $\mathbf{96.66\pm.15}$ & $\mathbf{83.36\pm.42}$ \\
    \bottomrule
  \end{tabular*}
\end{table}
\FloatBarrier

\subsection{GSM8K setup and metric}
\label{app:gsm8k-protocol}

We train on TinyGSM \citep{liu2023tinygsm}, which pairs synthetic math word
problems with Python solutions generated by GPT-3.5. Evaluation uses all
1,319 problems in the GSM8K test split
\citep{cobbe2021trainingverifierssolvemath}; we do not report TinyGSM test
accuracy. We tokenize the problem and solution with the SmolLM tokenizer
and use context length 512. The problem statement is kept clean during
training and fixed during generation.

\paragraph{Executed final-answer accuracy.}
For each problem $j$, we generate one Python program and execute it to
obtain $\widehat a_j$. Let $a_j$ be the reference numerical answer and
$b_j=1$ only when execution returns the correct answer. The reported score
is
\begin{equation}
  \mathrm{Acc}_{\mathrm{GSM8K}}
  =\frac{100}{1319}\sum_{j=1}^{1319}b_j.
  \label{eq:gsm8k-answer-accuracy}
\end{equation}
Programs that do not produce the correct answer are not counted as
successful. We compare answers, not program text: different programs can
solve the same problem. This is a single-generation accuracy, not
pass@$k$ or a best-of-multiple-generations score. Parentheses in
Table~\ref{tab:gsm8k-ours-sweep} give $\sum_j b_j$, and percentages are
rounded to two decimal places.

\paragraph{Baselines and comparability.}
Dense NFE sweeps use the released 250k-update checkpoints for MDLM, Duo,
CANDI, FLM, S-FLM, and FMLM+ variants
\citep{deschenaux2026hyperspherical,agarwal2026posteriorrefinement}.
This suite matches the TinyGSM training data, backbone scale (12 layers,
hidden size 768, and 12 attention heads), and 250k model-update budget.
Training hyperparameters, including the learning rate and noise schedule,
remain method-specific; we do not claim an identical training recipe.
The main controlled curves are MDLM, Duo, S-FLM, FMLM+ trained from scratch,
and our Brownian model.

FMLM+ initialized from MDLM starts from a pretrained checkpoint, and its
distilled variant uses a pretrained MDLM teacher
\citep{agarwal2026posteriorrefinement}. These curves appear in the
enhanced comparison in Figure~\ref{fig:gsm8k-enhanced}. The 250k updates describe training of
the evaluated model and do not include the cost of producing its teacher
or initialization. IDLM also uses pretrained
MDLM or Duo teachers and initializes its student from them
\citep{li2026idlm}. We retain these strong baselines as teacher-assisted
comparisons, not as matched from-scratch or equal-total-compute runs.

IDLM and vMF results are reported values from their papers
\citep{li2026idlm,chemseddine2026spherical}, not evaluations of our models.
In Figure~\ref{fig:gsm8k-comparison}, $\dagger$ marks the paper-reported
vMF curve; IDLM appears only in Table~\ref{tab:gsm8k-results}. In our numerical tables, $\dagger$ instead indicates
pretrained assistance and $\ddagger$ indicates self-conditioning with
predictor-corrector sampling. We also report vMF PC without
self-conditioning from Table~20 of \citet{chemseddine2026spherical}.
Both PC curves use the best reported configuration at each NFE budget,
with corrector calls included in the budget; they are not new evaluations
under our sampler. Our model uses neither pretrained assistance nor
self-conditioning. AR scores use 512 sequential passes.
LatentMDM reports $20.8\%$, $33.8\%$, and $45.5\%$ with two, four, and
eight latent-segment rounds \citep{kim2026interface}. Because a segment
round is not a matched NFE, we do not include these values in
Table~\ref{tab:gsm8k-results}.
A missing NFE measurement is marked NR; we do not interpolate
missing accuracy values.

\subsection{Full GSM8K comparison}

Figure~\ref{fig:gsm8k-enhanced} and Table~\ref{tab:gsm8k-summary} extend the
controlled comparison with teacher-initialized and self-conditioned models.

\begin{figure}[!htbp]
  \centering
  \includegraphics[width=0.75\linewidth]{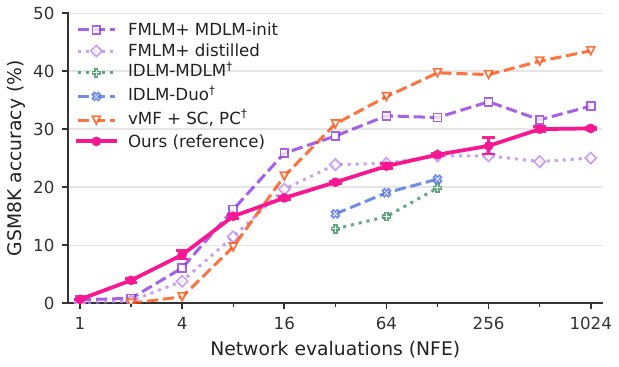}
  \caption{Enhanced GSM8K comparisons using teacher initialization or
  self-conditioning. Our model uses neither; the controlled comparison is in
  Figure~\ref{fig:gsm8k-comparison}.}
  \label{fig:gsm8k-enhanced}
\end{figure}

\begin{table}[!htbp]
  \caption{GSM8K accuracy (\%) at the full reported budget.
  Diffusion/flow models use the reported 1,024-call budget; decoding regimes
  differ: MDLM/Duo use $T=1$, SBD uses q-argmax, and S-FLM top-1 and FMLM+
  use deterministic token selection. vMF uses its published PC setting.
  SBD is the mean over three inference seeds; AR uses 512 sequential passes.
  Bold/underline mark the best/second-best non-AR
  result. $\dagger$: pretrained assistance; $\ddagger$: self-conditioning
  with predictor-corrector sampling. See Appendix~\ref{app:gsm8k-protocol}
  for provenance and comparability.}
  \label{tab:gsm8k-summary}
  \centering
  \small
  \setlength{\tabcolsep}{10pt}
  \begin{tabular}{lr}
    \toprule
    Model & Accuracy \\
    \midrule
    \multicolumn{2}{l}{\textit{Autoregressive}} \\
    AR (sample) & 53.9 \\
    AR (greedy) & 63.3 \\
    \addlinespace
    \multicolumn{2}{l}{\textit{Discrete diffusion}} \\
    MDLM & 16.91 \\
    Duo & 18.95 \\
    \addlinespace
    \multicolumn{2}{l}{\textit{Continuous diffusion and flows}} \\
    CANDI & 0.15 \\
    FLM & 0.53 \\
    S-FLM (exact) & 11.98 \\
    S-FLM (top-$1$) & 16.53 \\
    FMLM+ (scratch) & 20.32 \\
    FMLM+ (MDLM init.)$^{\dagger}$ & \underline{33.97} \\
    FMLM+ (distilled)$^{\dagger}$ & 25.02 \\
    vMF (PC, no SC) & 33.0 \\
    vMF+SC (PC)$^{\ddagger}$ & \textbf{43.5} \\
    \rowcolor{ourrowpink}
    \textbf{Ours: q-argmax} & 30.12 \\
    \bottomrule
  \end{tabular}
\end{table}

Table~\ref{tab:gsm8k-summary} summarizes the full-budget comparison;
Table~\ref{tab:gsm8k-results} gives the broader baseline ladder and the
three-seed SBD q-argmax values used in Figure~\ref{fig:gsm8k-comparison}.
The matched-temperature subset is in Table~\ref{tab:gsm8k-matched-temperature},
and the SBD memoryless curve in Table~\ref{tab:gsm8k-temperature-ladder}.
We report the complete NFE ladder when a
released checkpoint is available and use NR for values that were not
reported.

\begin{table}[!htbp]
  \caption{GSM8K final-answer accuracy (\%) across network-evaluation budgets.
  Dense rows are evaluations of released 250k-update checkpoints, excluding
  teacher/init pretraining costs. vMF PC (with/without SC) and IDLM values
  are taken from their papers. SBD q-argmax reports a mean and a separate
  standard-deviation row over three inference seeds; other dense rows are
  single runs. $\dagger$: pretrained teacher or
  initialization; $\ddagger$: self-conditioning and PC.
  NR means that a value
  is not reported.  AR uses 512 sequential passes.  Bold marks the best
  non-autoregressive result at each budget.}
  \label{tab:gsm8k-results}
  \centering
  \scriptsize
  \setlength{\tabcolsep}{1.8pt}
  \begin{tabular}{l*{11}{r}}
    \toprule
    Model & 1 & 2 & 4 & 8 & 16 & 32 & 64 & 128 & 256 & 512 & 1024 \\
    \midrule
    \multicolumn{12}{l}{\textit{Autoregressive}} \\
    AR (sample) & NR & NR & NR & NR & NR & NR & NR & NR & NR & 53.9 & NR \\
    AR (greedy) & NR & NR & NR & NR & NR & NR & NR & NR & NR & 63.3 & NR \\
    \addlinespace
    \multicolumn{12}{l}{\textit{Discrete diffusion and distillation}} \\
    MDLM & 0.00 & 0.00 & 0.00 & 0.23 & 1.82 & 5.69 & 10.01 & 13.57 & 15.09 & 14.94 & 16.91 \\
    Duo & 0.00 & 0.38 & 1.67 & 3.64 & 10.01 & 13.87 & 16.60 & 15.39 & 18.50 & 18.27 & 18.95 \\
    IDLM-MDLM$^{\dagger}$ & NR & NR & NR & NR & NR & 12.81 & 14.94 & 19.86 & NR & NR & NR \\
    IDLM-Duo$^{\dagger}$ & NR & NR & NR & NR & NR & 15.39 & 19.03 & 21.38 & NR & NR & NR \\
    \addlinespace
    \multicolumn{12}{l}{\textit{Continuous diffusion and flows}} \\
    CANDI (lr $10^{-3}$) & 0.00 & 0.00 & 0.00 & 0.00 & 0.00 & 0.00 & 0.00 & 0.08 & 0.08 & 0.23 & 0.15 \\
    FLM (default) & 0.00 & 0.00 & 0.00 & 0.00 & 0.00 & 0.00 & 0.08 & 0.15 & 0.23 & 0.45 & 0.53 \\
    S-FLM (exact) & \textbf{0.83} & 2.58 & 4.70 & 7.51 & 9.25 & 10.31 & 11.22 & 11.22 & 11.68 & 11.75 & 11.98 \\
    S-FLM (top-$1$) & \textbf{0.83} & 3.34 & 8.04 & 11.22 & 15.47 & 16.45 & 16.07 & 16.07 & 16.60 & 16.45 & 16.53 \\
    FMLM+ (scratch) & 0.15 & 0.15 & 3.87 & 8.19 & 14.48 & 19.26 & 20.92 & 21.68 & 20.62 & 21.99 & 20.32 \\
    FMLM+ (MDLM init.)$^{\dagger}$ & 0.61 & 0.83 & 6.07 & \textbf{16.15} & \textbf{25.85} & 28.81 & 32.30 & 31.99 & 34.72 & 31.61 & 33.97 \\
    FMLM+ (distilled)$^{\dagger}$ & 0.08 & 0.38 & 3.79 & 11.45 & 19.64 & 23.88 & 24.11 & 25.40 & 25.40 & 24.41 & 25.02 \\
    vMF (PC, no SC) & NR & 0.0 & 0.7 & 6.2 & 15.5 & 21.8 & 26.9 & 27.4 & 31.8 & 32.4 & 33.0 \\
    vMF+SC (PC)$^{\ddagger}$ & NR & 0.0 & 1.1 & 9.7 & 21.9 & \textbf{30.9} & \textbf{35.6} & \textbf{39.7} & \textbf{39.4} & \textbf{41.7} & \textbf{43.5} \\
    \addlinespace
    \multicolumn{12}{l}{\emph{Ours}} \\
    Ours: q-argmax, mean & 0.66 & \textbf{3.94} & \textbf{8.34} & 14.99 & 18.15 & 20.85 & 23.63 & 25.60 & 27.12 & 30.00 & 30.12 \\
    Ours: standard deviation & 0.04 & 0.35 & 0.72 & 0.34 & 0.23 & 0.20 & 0.46 & 0.22 & 1.41 & 0.42 & 0.16 \\
    \bottomrule
  \end{tabular}
\end{table}
\FloatBarrier

\subsection{S-FLM S-arch q-argmax transfer}
\label{app:sflm-qargmax}

We also apply q-argmax to the released S-FLM S-arch checkpoint, using SLERP
re-noising at the model's own time labels. The full NFE curve is shown in
Figure~\ref{fig:gsm8k-comparison}; at higher budgets this refinement roughly
doubles GSM8K accuracy relative to S-FLM top-1 decoding. S-FLM and SBD differ
in schedule, truncation, embedding handling, and batch size, so this transfer
does not isolate a kernel effect.

Table~\ref{tab:sarch-qargmax-ladder} gives the full values. This is one
evaluation seed, so it shows transfer feasibility instead of seed-averaged
comparative performance.
\begin{table}[!htbp]
  \caption{GSM8K accuracy of q-argmax on the released S-FLM S-arch checkpoint,
  using SLERP re-noising at the model's own time labels. One evaluation seed.}
  \label{tab:sarch-qargmax-ladder}
  \centering\small
  \begin{tabular}{r|rrrrrrrrrrr}
    \toprule
    NFE & 1 & 2 & 4 & 8 & 16 & 32 & 64 & 128 & 256 & 512 & 1024 \\
    \midrule
    Accuracy (\%) & 0.8 & 3.3 & 8.2 & 14.9 & 20.5 & 23.2 & 26.2 & 28.7 & 29.3 & 30.5 & 31.2 \\
    \bottomrule
  \end{tabular}
\end{table}

\subsection{GSM8K kernel-consistency study}
\label{app:gsm8k-heat-vmf}

We evaluate all 1,319 GSM8K test problems with EMA weights, context length
512, an entropy-spaced grid, and three sampling seeds. The primary readout is
q-argmax with posterior decoding. Quantile-flow and Brownian ODE results are
included as checks, although both continuous samplers are weak on this task.

\paragraph{Kernel cross.}
Moment-matched vMF remains close to Heat, while the mixture mismatch destroys
program accuracy even under valid transport. The information-matched vMF is
a pure time-change control: querying the Heat network at the equal-cosine
time removes its low-budget cost (Table~\ref{tab:gsm8k-kernel-cross-new}).

\begin{table}[!htbp]
\caption{GSM8K accuracy (\%), mean $\pm$ standard deviation over three
inference seeds. Each row names training weights and inference kernel.
Budgets are sampler calls $B$; posterior decoding adds one call, giving
65 and 257 total NFEs.}
\label{tab:gsm8k-kernel-cross-new}
\centering\scriptsize
\setlength{\tabcolsep}{3pt}
\begin{tabular}{lrrrr}
\toprule
Weights $\to$ kernel & \multicolumn{2}{c}{q-argmax} & \multicolumn{2}{c}{quantile ODE} \\
 & $B=64$ & $B=256$ & $B=64$ & $B=256$ \\
\midrule
Heat $\to$ Heat & $23.63\pm.46$ & $27.12\pm1.41$ & $6.92\pm.61$ & $7.56\pm.61$ \\
Heat $\to$ moment vMF & $22.74\pm.08$ & $27.32\pm.24$ & $6.85\pm.69$ & $7.48\pm.69$ \\
Heat $\to$ mixture $.5$ & $0.05\pm.09$ & $0.03\pm.04$ & $0.40\pm.16$ & $0.45\pm.13$ \\
Heat $\to$ mixture $.9$ & $0.00$ & $0.00$ & $0.00$ & $0.00$ \\
Info-vMF $\to$ info-vMF & $23.05\pm.85$ & $27.24\pm.76$ & $7.15\pm.16$ & $7.43\pm.27$ \\
Heat $\to$ info-vMF & $19.81\pm.42$ & $26.59\pm.42$ & $5.94\pm.36$ & $6.85\pm.56$ \\
Heat $\to$ info-vMF, relabelled & $23.10\pm.27$ & $27.39\pm.42$ & $6.87\pm.78$ & $7.43\pm.67$ \\
\bottomrule
\end{tabular}
\end{table}

\paragraph{Kernel and training comparison.}
Table~\ref{tab:heat-vmf-resolution} compares Heat-trained weights sampled
with Heat against information-vMF-trained weights sampled with their own
kernel. It changes the training channel and inference kernel while leaving the Heat table resolution unchanged. At $B=32$ the paired difference is $-1.67\pm0.92$ points
($t=-3.1$, $p\simeq0.09$); at $B=256$ it is $+0.13$ points.

\begin{table}[!htbp]
  \caption{GSM8K kernel/training comparison: Heat-trained Heat inference
  versus information-vMF-trained information-vMF inference. Values are means
  over three inference seeds; differences are info-vMF minus Heat. Budgets
  are sampler calls $B$, before a separate posterior decode.}
  \label{tab:heat-vmf-resolution}
  \centering
  \small
  \setlength{\tabcolsep}{7pt}
  \begin{tabular}{rrrr}
    \toprule
    $B$ & Heat $\to$ Heat & info-vMF $\to$ info-vMF & Difference \\
    \midrule
    1   & 0.66 & 0.96 & $+0.30$ \\
    2   & 3.94 & 4.20 & $+0.25$ \\
    4   & 8.34 & 8.82 & $+0.48$ \\
    8   & 14.99 & 14.25 & $-0.73$ \\
    16  & 18.15 & 18.12 & $-0.03$ \\
    32  & 20.85 & 19.18 & $-1.67\pm0.92$ ($t=-3.1$, $p\simeq0.09$) \\
    64  & 23.63 & 23.05 & $-0.58$ \\
    256 & 27.12 & 27.24 & $+0.13$ \\
    \bottomrule
  \end{tabular}
\end{table}

\paragraph{Denoising probe and relabel maps.}
We build the relabel maps on 512 TinyGSM validation examples that are disjoint
from both the GSM8K test set and the 128-example mechanism probe. For each
mixture time, the map chooses the monotone Heat-network time label with the
highest clean-token accuracy. This gives an upper bound on what a scalar time
change can repair; the probe measurements are in
Table~\ref{tab:gsm8k-denoising-probe}.

\begin{table}[!htbp]
\caption{Code-token probe: one-step top-1 accuracy of the Heat posterior on the
9,345 program tokens in 128 TinyGSM validation answers, excluding PAD.
Unqualified columns use the raw time label; ``best'' chooses the best tested
Heat-network time label.}
\label{tab:gsm8k-denoising-probe}
\centering\scriptsize
\begin{tabular}{rrrrrrr}
\toprule
$r$ & Heat & moment vMF & mix $.5$ & mix $.5$, best & mix $.9$ & mix $.9$, best \\
\midrule
.2 & 1.000 & 1.000 & .996 & .999 & .828 & .971 \\
.5 & .999 & .999 & .994 & .998 & .620 & .945 \\
1.0 & .996 & .996 & .949 & .988 & .381 & .880 \\
1.6 & .984 & .983 & .860 & .960 & .213 & .758 \\
2.5 & .936 & .936 & .814 & .883 & .153 & .481 \\
\bottomrule
\end{tabular}
\end{table}

The best-case map partly repairs the $.5$ mixture but cannot repair $.9$.
At $B=256$ with posterior decoding (257 NFEs), q-argmax moves from $0.03\%$
to $8.72\pm0.60\%$ for $.5$, while $.9$ stays at $0.00\%$; the Heat reference is
$27.12\%$. The mixture
tables match their target means, the conditional-law and intermediate-marginal checks pass at $d=768$, and generated
sequences show the same posterior-level corruption.

\paragraph{Why program accuracy amplifies small token errors.}
The full-answer probe uses the same 128 TinyGSM validation answers, mixture corruption,
and raw time labels as the code-token probe, but counts the full answer region
(code plus PAD, about 459 positions per answer). PAD accounts for about
80\% of these positions and is predicted correctly more than 99\% of the
time, so these per-token accuracies are not directly comparable with
Table~\ref{tab:gsm8k-denoising-probe}. In the full-answer probe, Heat at $r=.5$ has
full-answer-region token accuracy $.9999$ and exact
sequence accuracy $.969$. The $.5$ mixture at $r=1.0$ has per-token accuracy
$.9917$ but only $.250$ exact accuracy, while the $.9$ mixture at $r=.1$ has
per-token accuracy $.9792$ and zero exact accuracy. The
mixture-to-Heat success ratio by program length is reported in
Table~\ref{tab:gsm8k-length-new}.

\begin{table}[!htbp]
\caption{GSM8K success by generated-program length at $B=256$ with
posterior decoding (257 NFEs). The same
1,319 problems define both columns; values average three inference seeds.}
\label{tab:gsm8k-length-new}
\centering\scriptsize
\begin{tabular}{rrrrr}
\toprule Length range (mean) & $n$ & Heat & mix $.5$, relabelled & ratio \\
\midrule
67--108 (96) & 259 & 44.3 & 21.0 & .47 \\
108--125 (116) & 258 & 40.8 & 13.3 & .33 \\
125--142 (132) & 265 & 26.7 & 6.2 & .23 \\
142--164 (152) & 266 & 17.8 & 3.5 & .20 \\
164--290 (189) & 271 & 7.3 & .25 & .03 \\
\bottomrule
\end{tabular}
\end{table}

The success ratio decreases monotonically with length, with correlation
$-0.96$ between length and log ratio. Here the \emph{independent-token model}
assumes that the mismatch independently preserves each generated token with a
length-independent probability $p$; for a program of length $L$, it therefore
predicts a mixture-to-Heat success ratio proportional to $p^L$. The
per-token probability in this diagnostic is measured over the full-answer probe's full
answer region, including PAD, instead of the code-token probe's code-only mask.
The observed decay is steeper, consistent with context-mediated error
propagation. The exponential curve in Figure~\ref{fig:gsm8k-consistency}
is a descriptive fit; it is not a causal test of this explanation.

\paragraph{Matched fine-tuning.}
We fine-tune the same Heat weights for 25k updates on either Heat or the $.5$
mixture, then sample each model with its own kernel. Table~\ref{tab:gsm8k-finetune-new}
reports the results; each arm uses one training seed, and the standard
deviation is over three inference seeds.

\begin{table}[!htbp]
\caption{Matched GSM8K fine-tunes (\%), mean $\pm$ sd over three inference
seeds, with one training seed per arm. The main comparison uses posterior
decoding at $B=256$ (257 NFEs); nearest-codeword controls use $B$ total NFEs.
The posterior-decoded paired mixture-minus-Heat difference is
$-1.09\pm1.80$ points.}
\label{tab:gsm8k-finetune-new}
\centering\small
\begin{tabular}{lr}
\toprule Posterior decoding, EMA & $B=256$ (257 NFEs) \\
\midrule
Heat fine-tune $\to$ Heat & $27.32\pm.52$ \\
Mix $.5$ fine-tune $\to$ mix $.5$ & $26.23\pm1.30$ \\
\bottomrule
\end{tabular}
\par\smallskip
\begin{tabular}{lrr}
\toprule Nearest-codeword decoding & 64 NFEs & 256 NFEs \\
\midrule
Heat fine-tune $\to$ Heat, EMA & $23.83\pm1.00$ & $27.32\pm.52$ \\
Mix $.5$ fine-tune $\to$ mix $.5$, EMA & $21.10\pm.49$ & $26.03\pm1.22$ \\
Heat fine-tune $\to$ Heat, no EMA & $21.53\pm.27$ & $26.00\pm.50$ \\
Mix $.5$ fine-tune $\to$ mix $.5$, no EMA & $19.16\pm.24$ & $23.02\pm1.37$ \\
\bottomrule
\end{tabular}
\end{table}

This shows that restoring train--inference kernel consistency can recover
most of the mismatch loss, but does not determine the performance of a
mixture model trained from scratch. The nearest-decoded EMA comparison gives
a paired gap of $-1.29\pm1.71$ points at 256 NFEs. With posterior decoding,
the gap is $-1.09\pm1.80$ at 257 NFEs; the three seed differences are
$+0.5$, $-3.0$, and $-0.8$ points. The residual gap is not resolved by these
three inference seeds.
\FloatBarrier

\subsection{GSM8K checkpoint sweep}

Figure~\ref{fig:gsm8k-ours-ablations} summarizes the checkpoint comparison;
Table~\ref{tab:gsm8k-ours-sweep} reports all available NFE measurements for
the five checkpoints used in our GSM8K ablation.  Each percentage is computed
on the 1,319-example GSM8K test set; the number of correct programs is shown in
parentheses.  Entropy sampling is strongest from 16 to 256 NFEs, whereas the
online adaptive configuration is strongest at 1 to 8 and 512 to 1,024 NFEs.

\begin{figure}[!htbp]
  \centering
  \includegraphics[width=0.75\linewidth]{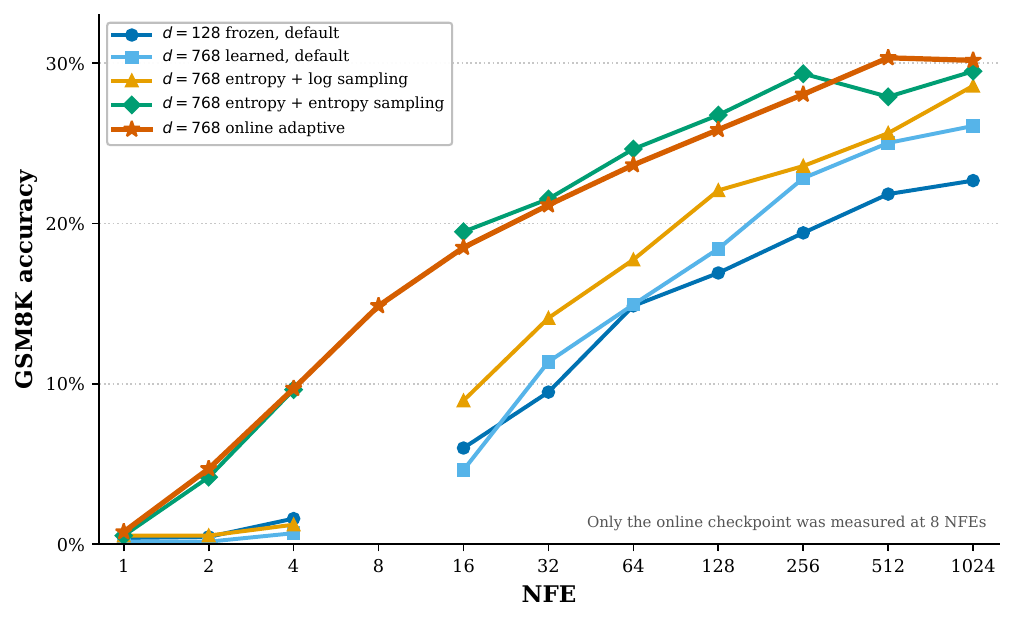}
  \caption{Our GSM8K checkpoint and sampling ablations. All configurations
  use plain q-argmax without Gumbel perturbations, score updates,
  self-conditioning, or pretrained assistance. Missing NFE measurements
  are omitted; Table~\ref{tab:gsm8k-ours-sweep} gives the numerical results.}
  \label{fig:gsm8k-ours-ablations}
\end{figure}

\begin{table}[!htbp]
  \caption{GSM8K final-answer accuracy (\%) across NFE budgets for our
  TinyGSM-trained checkpoints. Parentheses give the number of correct examples
  out of 1,319. Bold marks the best result at each budget; NR means not reported.}
  \label{tab:gsm8k-ours-sweep}
  \centering
  \scriptsize
  \setlength{\tabcolsep}{5pt}
  \begin{tabular}{rrrrrr}
    \toprule
    NFE & \shortstack{$d=128$ frozen\\default}
        & \shortstack{$d=768$ learned\\default}
        & \shortstack{$d=768$ entropy\\+ log sampling}
        & \shortstack{$d=768$ entropy\\+ entropy sampling}
        & \shortstack{$d=768$ online\\adaptive} \\
    \midrule
       1 &  0.38 (5)   &  0.15 (2)   &  0.53 (7)   &  0.53 (7)   & \textbf{0.76 (10)} \\
       2 &  0.45 (6)   &  0.15 (2)   &  0.53 (7)   &  4.17 (55)  & \textbf{4.70 (62)} \\
       4 &  1.59 (21)  &  0.68 (9)   &  1.21 (16)  &  9.63 (127) & \textbf{9.70 (128)} \\
       8 & NR          & NR          & NR          & NR          & \textbf{14.86} \\
      16 &  5.99 (79)  &  4.62 (61)  &  8.95 (118) & \textbf{19.48 (257)} & 18.50 (244) \\
      32 &  9.48 (125) & 11.37 (150) & 14.10 (186) & \textbf{21.53 (284)} & 21.15 (279) \\
      64 & 14.86 (196) & 14.94 (197) & 17.74 (234) & \textbf{24.64 (325)} & 23.65 (312) \\
     128 & 16.91 (223) & 18.42 (243) & 22.06 (291) & \textbf{26.76 (353)} & 25.85 (341) \\
     256 & 19.41 (256) & 22.82 (301) & 23.58 (311) & \textbf{29.34 (387)} & 28.05 (370) \\
     512 & 21.83 (288) & 25.02 (330) & 25.63 (338) & 27.90 (368) & \textbf{30.33 (400)} \\
    1024 & 22.67 (299) & 26.08 (344) & 28.58 (377) & 29.49 (389) & \textbf{30.17 (398)} \\
    \bottomrule
  \end{tabular}
\end{table}

\paragraph{Full-test continuous-sampler sweep.}
Table~\ref{tab:gsm8k-continuous-sweep} reports ODE and Heun SDE accuracy
on all 1,319 GSM8K test questions. ODE is stronger at 32 and 64 NFEs,
while Heun SDE is stronger from 128 NFEs onward. Both achieve their
highest reported accuracy at 256 NFEs, with ODE matching that value at
512 NFEs. Increasing the budget beyond 256 gives no further gain in
this sweep. These full-test measurements are separate from the
generated-trajectory diagnostics in
Table~\ref{tab:commitment-generated-diagnostics}, which use 64 examples
per condition and separately record a final diagnostic call.

\begin{table}[!htbp]
  \caption{GSM8K executed-answer accuracy (\%) for continuous samplers on
  the full 1,319-question test set. Budgets are reported NFEs. Bold marks
  the better sampler at each budget. This table is separate from the
  checkpoint and generated-trajectory ablations.}
  \label{tab:gsm8k-continuous-sweep}
  \centering
  \small
  \setlength{\tabcolsep}{14pt}
  \begin{tabular}{rrr}
    \toprule
    NFE & ODE & Heun SDE \\
    \midrule
      32 & \textbf{3.64} & 0.00 \\
      64 & \textbf{7.81} & 7.20 \\
     128 & 7.96 & \textbf{10.08} \\
     256 & 8.34 & \textbf{12.21} \\
     512 & 8.34 & \textbf{11.60} \\
    1024 & 8.26 & \textbf{11.98} \\
    \bottomrule
  \end{tabular}
\end{table}

\FloatBarrier
\subsection{OpenWebText evaluation}
\label{app:owt-protocol}

Appendix~\ref{app:owt-generated-samples} gives unedited generated excerpts
at four target entropy levels and six NFE budgets.

\paragraph{Generation settings.}
Our reported OWT model uses a freely learned $d=768$ codebook and the
online adaptive training schedule, with 1M updates.
For each method and NFE budget, we sweep the same 15 whole-trajectory
temperatures, $T\in\{0.50,0.55,\ldots,1.20\}$. Apart from temperature,
each method's sampler settings remain fixed throughout the sweep.
Temperature is applied at every sampling
step: to the categorical transition for MDLM and Duo, and to the posterior
weights used by S-FLM and our mixture score.  We use
$\mathrm{NFE}\in\{32,64,128,256,512,1024\}$. Our model and the MDLM,
Duo, and S-FLM models shown as temperature sweeps use 1M training updates.
The main paper shows 1,024 NFEs;
Figure~\ref{fig:owt-all-budgets} gives all six budgets.
The FLM, FMLM, and FMLM+ markers are our evaluations of released
checkpoints with their native samplers, not paper-reported numbers or
temperature sweeps. We do not treat them as matched-training-budget
comparisons. For FMLM+, we retain four posterior-refinement rounds,
$p=0.999$, fresh noise, and an equal step budget per refinement round. Only the 128-NFE
point uses the authors' released OWT configuration: 32 inner flow steps
per round, for $32\times4=128$ evaluations. At the other budgets,
$\dagger$ marks our budget-matched extrapolation of the inner step count
to $\mathrm{NFE}/4$, including $256\times4$ at 1,024 NFEs. These values correspond to directly measured operating points, with perplexity and entropy evaluated at those settings.

\paragraph{Hybrid sampler.}
For our $d=768$ model, we use an entropy grid and initialize from the uniform
distribution. Above $r=2.064065$, we use stochastic Heun score updates. Below this
switch, $k$ controls the q-projection frequency: $k=1$ projects at every
step, $k=2$ every second step, and $k=4$ every fourth step. Intervening
steps use continuous stochastic score/SDE updates. At a fixed NFE budget, increasing $k$
therefore allocates fewer updates to q-projection and more to the score
sampler. Here $k$ is a cadence, not top-$k$ posterior truncation.
The default is $k=1$; at 1,024 NFEs we also evaluate $k=2$ and $k=4$
with the same temperature sweep and all other sampler settings fixed.

In the Gumbel configuration, q-projection uses the perturbed posterior
without the additional commitment gate (the confidence gating defined in Section \ref{sec:spherical-brownian-diffusion}).
One fixed Gumbel noise tensor is reused at every q-projection. Its effective
strength is
\begin{equation}
  \gamma(r)=0.25\,
  \frac{\log(r/r_{\min})}{\log(r_{\max}/r_{\min})}.
  \label{eq:owt-gumbel-strength}
\end{equation}
Thus $0.25$ is a base multiplier, not a final temperature: the strength is
$0.25$ at $r_{\max}$ and decreases toward zero at $r_{\min}$. Since q-projection begins below the switch, its first effective strength is already smaller than $0.25$. Whole-trajectory temperature is swept while this schedule and the other settings remain fixed. Budgets are expressed in network evaluations, not in grid points; each Heun update requires two evaluations.

\paragraph{Generative perplexity.}
We use the sample-level quality and diversity metrics of partition
generative modeling \citep{deschenaux2025partitiongenerativemodelingmasked}.
The main temperature sweeps and isolated reference points use $N=512$
sequences per setting. The plotted $k=2$ and $k=4$ variants instead use
64 sequences, with 32 at each of seeds 1 and 2; the plotted $k=1$ curve
uses 512, with 64 at each of seeds 1 to 8. We also evaluate $k=1$
with the same 64-sample protocol as $k=2/4$ for the controlled
temperature comparison below. All sequences have length 1,024 and are
scored with the same pretrained GPT-2 Large reference model and revision
\citep{radford2019language}, using its tokenizer. For a sample
$\mathbf x$, we compute
\begin{equation}
  \mathrm{PPL}_{\mathrm{gen}}(\mathbf x)
  =\exp\!\left(-\frac{1}{m}\sum_{i=1}^{m}
  \log p_{\mathrm{GPT2\text{-}L}}(x_i\mid x_{<i})\right),
  \label{eq:gen-ppl}
\end{equation}
where $m$ is the number of scored tokens before the first end-of-text
marker. Tokens at and after that marker are excluded from the score.
This measures the reference model likelihood assigned to generated text. Lower values indicate text that the
reference model finds more likely, but can also result from repetitive
outputs.

\paragraph{Unigram entropy.}
Within each sample, let $c(v,\mathbf x)$ count occurrences of token $v$ and
$L_{\mathbf x}$ be the number of tokens used in the entropy calculation.
The empirical frequencies are
$\widehat p_{\mathbf x}(v)=c(v,\mathbf x)/L_{\mathbf x}$. We compute
\begin{equation}
  H(\mathbf x)=-\sum_{v\in\mathcal V}
  \widehat p_{\mathbf x}(v)\log\widehat p_{\mathbf x}(v).
  \label{eq:unigram-entropy}
\end{equation}
We take $0\log0=0$ and use natural logarithms, so entropy is measured in
nats. Repeated use of a small set of tokens gives low entropy; a broader
token distribution gives higher entropy. This is a within-sample token
diversity measure, not the entropy of the learned posterior or a measure of
semantic diversity across the batch.

\paragraph{Averaging and frontier interpretation.}
Each temperature contributes one point obtained by averaging the two
metrics separately across samples:
\begin{equation}
  \overline H=\frac1N\sum_{n=1}^N H(\mathbf x^{(n)}),
  \qquad
  \overline{\mathrm{PPL}}_{\mathrm{gen}}
  =\frac1N\sum_{n=1}^N
  \mathrm{PPL}_{\mathrm{gen}}(\mathbf x^{(n)}).
  \label{eq:owt-metric-averaging}
\end{equation}
We average per-sample perplexities after exponentiation; this differs from
exponentiating the pooled mean negative log-likelihood. We likewise
average per-sample unigram entropies rather than pooling token counts
across samples.

The plotted curves retain only measured Pareto-nondominated points:
a point is excluded if another measured point has at least as much entropy
and no larger perplexity, with one strict improvement. We connect the
retained points without fitting a curve. The temperature-matched ablation
uses all 15 settings before this filtering.

For each NFE budget, the temperature sweep gives a quality-diversity
curve. At a shared entropy, lower generative perplexity is better.
Perplexity alone can favor collapsed samples, while entropy alone does
not guarantee coherent text; the paired metrics expose this trade-off.
We compare only the entropy range reached by both methods and do not
extrapolate a curve beyond its measured range. The sweep gives an empirical frontier under the stated sampler settings, without asserting global optimality across all possible samplers.

\paragraph{Results across sampling budgets.}
Figure~\ref{fig:owt-all-budgets} gives the complete NFE sweep.
At 128 NFEs, our default $k=1$ sampler has lower perplexity than S-FLM
at low and moderate entropy, but the advantage disappears at the
high-entropy end. At 1,024 NFEs, it has lower perplexity throughout
their shared entropy range. MDLM and Duo generally achieve lower
perplexity at matched entropy. Increasing NFE mainly improves our
low- and moderate-entropy performance, while the default sampler's
high-entropy coverage remains limited. The isolated FLM, FMLM, and
FMLM+ reference points are not full temperature sweeps, so we do not
infer their frontiers from these markers.

\paragraph{Effect of q-update frequency.}
The controlled $k$ ablation uses 64 samples per temperature and variant,
with 32 samples at each of seeds 1 and 2. It holds the model, nominal
NFE budget, noise switch, grid construction rule, Gumbel schedule,
whole-trajectory temperature, and reference-model revision fixed,
changing the allocation between q-updates and continuous stochastic score
updates below the switch. Entropy and perplexity both increase from
$k=1$ to $k=2$ to $k=4$ at every tested temperature.
This operating-point shift should be distinguished from a comparison at
matched entropy, which can pair different temperatures: in the common
entropy range, $k=1$ generally has lower perplexity than $k=2$ and $k=4$.
Larger $k$ extends the attained entropy range, so no $k=1$ comparison is
available beyond its measured endpoint. Budgets refer to network calls,
not grid points: $k=1$ uses 443 two-call score updates and 138 one-call
q-updates, $k=2$ uses 479 and 65, and $k=4$ uses 496 and 32.
The realised NFEs are therefore 1,024, 1,023, and 1,024, respectively;
$k=2$ leaves one evaluation unused because of the two-call Heun updates.
Table~\ref{tab:owt-k-matched-temperature} gives representative temperatures
from the matched-64-sample data, separate from the 512-sample $k=1$
curve in Figure~\ref{fig:owt-all-budgets}. These aggregate means do not provide error bars
or a statistical significance test.

\begin{table}[!htbp]
  \caption{Temperature-matched OWT $k$ ablation. Each entry uses 64 samples
  (two seeds, 32 each) at the same nominal 1,024-NFE budget; realised NFEs
  are 1,024 / 1,023 / 1,024 for $k=1/2/4$. $H$ is mean unigram entropy
  in nats and PPL is mean generative perplexity.}
  \label{tab:owt-k-matched-temperature}
  \centering
  \small
  \setlength{\tabcolsep}{7pt}
  \begin{tabular}{rrrrrrr}
    \toprule
    & \multicolumn{2}{c}{$k=1$} & \multicolumn{2}{c}{$k=2$}
      & \multicolumn{2}{c}{$k=4$} \\
    \cmidrule(lr){2-3}\cmidrule(lr){4-5}\cmidrule(lr){6-7}
    $T$ & $H$ & PPL & $H$ & PPL & $H$ & PPL \\
    \midrule
    0.80 & 4.56 & 16.32 & 4.61 & 20.82 & 4.78 & 26.13 \\
    1.00 & 5.18 & 35.12 & 5.24 & 43.58 & 5.36 & 59.42 \\
    1.20 & 5.36 & 65.10 & 5.57 & 136.87 & 5.75 & 230.57 \\
    \bottomrule
  \end{tabular}
\end{table}

\begin{figure}[p]
  \centering
  \includegraphics[width=0.495\linewidth]{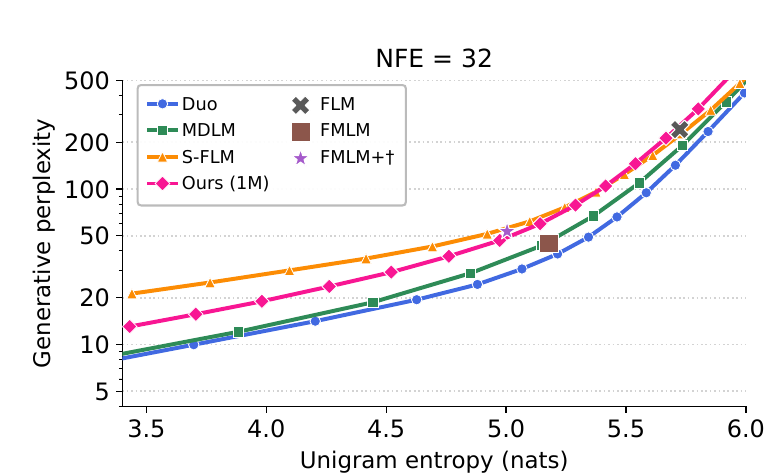}\hfill
  \includegraphics[width=0.495\linewidth]{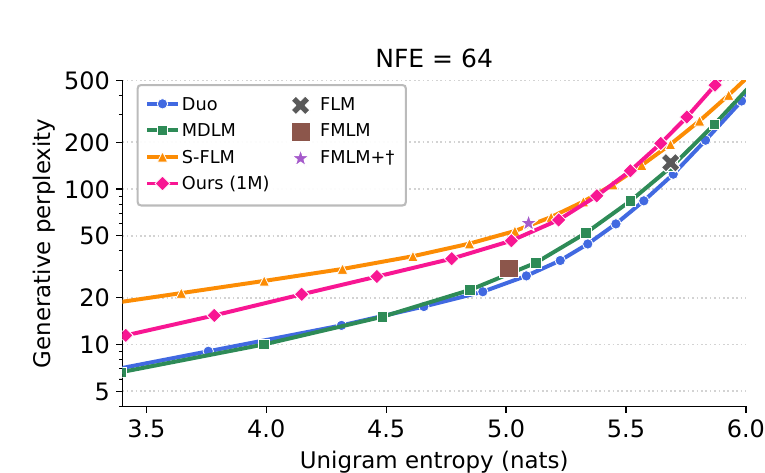}\\[0.5em]
  \includegraphics[width=0.495\linewidth]{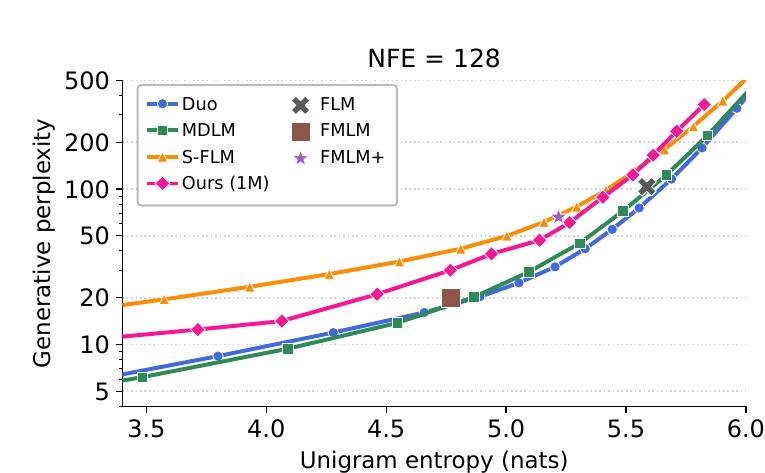}\hfill
  \includegraphics[width=0.495\linewidth]{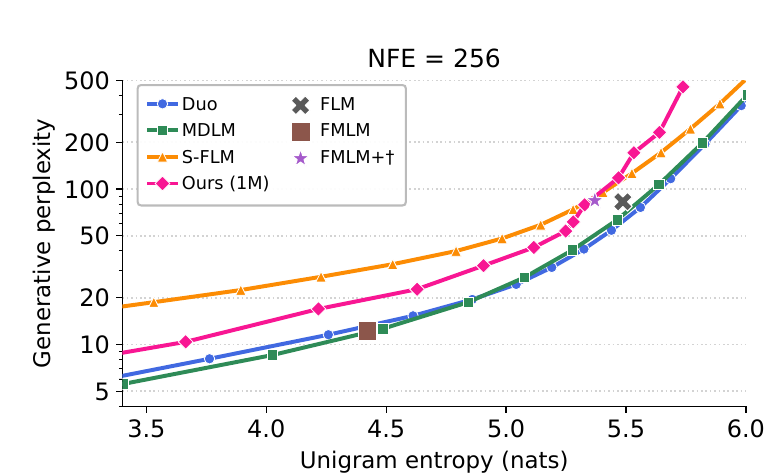}\\[0.5em]
  \includegraphics[width=0.495\linewidth]{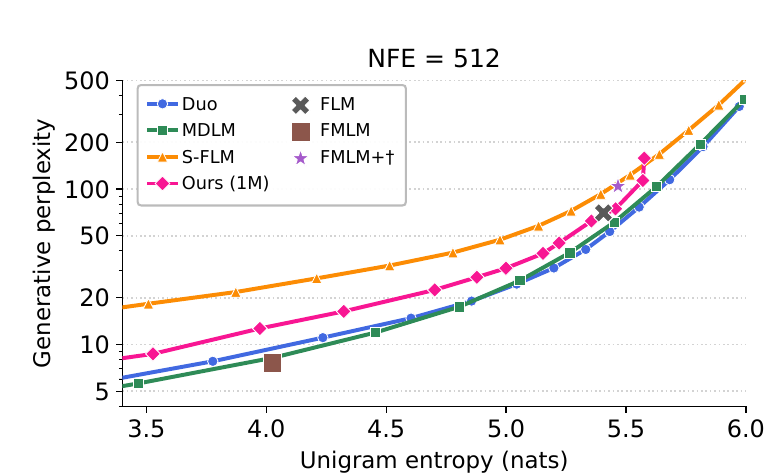}\hfill
  \includegraphics[width=0.495\linewidth]{output/pdf/owt_frontier_nfe1024.pdf}
  \caption{OpenWebText frontiers at all six budgets. Our 1M-update model
  uses $k=1$ unless noted. Curves show measured Pareto-filtered points
  from 15 temperatures, with 512 samples per point except $k=2/4$ (64).
  Reference markers use 512 samples; $k=2$ realises 1,023 evaluations.
  FMLM+ $\dagger$ denotes our budget-matched extrapolation of the inner
  flow-step count with four refinement rounds ($256\times4$ at 1,024);
  only $32\times4=128$ is the released OWT configuration.
  Lower perplexity is better at matched entropy.}
  \label{fig:owt-all-budgets}
\end{figure}
\FloatBarrier

\Needspace{28\baselineskip}
\paragraph{Evolution during training.}
Figure~\ref{fig:owt-checkpoint-progress} compares our OWT frontiers after
50k, 100k, 500k, and 1M training updates at 256 and 1,024 NFEs.
\citet{pynadath2026generativefrontiersevaluationmatters} report that
50k-update frontiers can already closely resemble those at 1M updates
for MDLM, Duo, and CANDI. Our early checkpoints also produce a clear
quality-diversity trade-off, but later training lowers perplexity at
matched entropy, particularly in the intermediate- and high-entropy
range. Most of the visible improvement occurs by 500k updates, while
the 500k and 1M curves are close in both panels. 

\begin{figure}[!htbp]
  \centering
  \includegraphics[width=\linewidth]{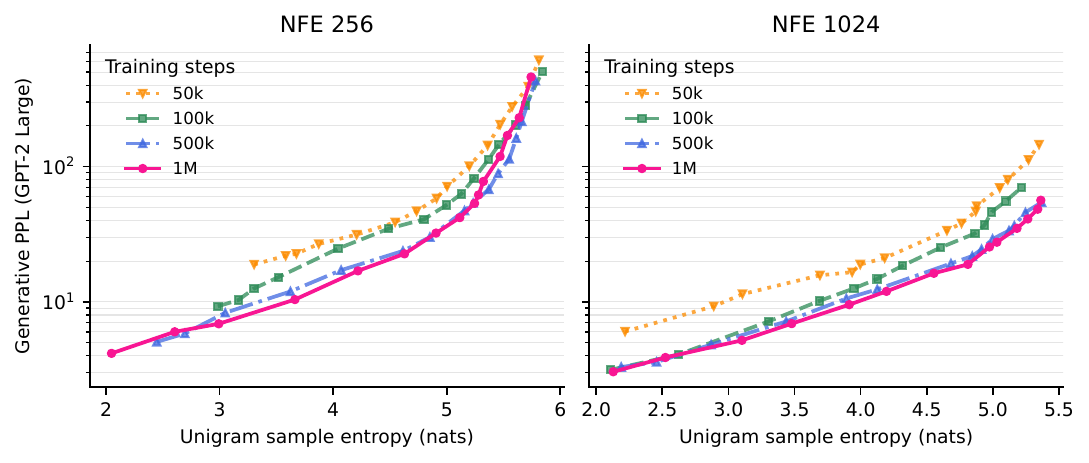}
  \caption{Evolution of the OWT generative frontier during training.
  Curves compare checkpoints at 50k, 100k, 500k, and 1M updates,
  with 256 NFEs (left) and 1,024 NFEs (right). This 128-sample checkpoint
  diagnostic uses GPT-2 Large to measure generative perplexity;
  lower perplexity is better at matched unigram entropy. The 500k
  and 1M curves are close, although both improve over the early
  checkpoints across much of the shared entropy range. Sample counts
  differ from the main 512-sample frontier evaluation.}
  \label{fig:owt-checkpoint-progress}
\end{figure}
\FloatBarrier

\subsection{Discrete commitment across tasks}
\label{app:commitment-diagnostics}

We distinguish posterior diagnostics on forward-noised reference examples
from measurements along actual generation trajectories. The former test
denoising states obtained from known clean sequences; the latter include
states induced by the sampler's own predictions. Their statistics need not
agree. These additional diagnostic runs complement the main benchmark results and the OWT temperature sweeps.
\paragraph{Diagnostic setup.}
Table~\ref{tab:commitment-diagnostic-setup} identifies the checkpoints
and grids. Language checkpoints use exponential-moving-average (EMA)
weights. Generated-trajectory diagnostics use 64 examples per condition
and sampling seeds 1 and 2. Forward readiness uses 128 validation
examples, two noise seeds, and 41 radii; no forward GSM8K probe is run.
OWT score temperatures are
$T\in\{0.60,0.75,0.90,1.05,1.20\}$.

\begin{table}[!htbp]
  \caption{Setup for the additional commitment diagnostics. These runs
  are separate from the matched-NFE Sudoku accuracy sweep.}
  \label{tab:commitment-diagnostic-setup}
  \centering
  \small
  \setlength{\tabcolsep}{4pt}
  \begin{tabular}{@{}p{0.24\linewidth}p{0.39\linewidth}p{0.32\linewidth}@{}}
    \toprule
    Task & Checkpoint & Trajectory settings \\
    \midrule
    Sudoku easy/medium & 8-layer DiT, embedding 512; $d=768$, learned-online,
      20k updates
      & $T=1$, log-radius grid \\
    TinyGSM / GSM8K & Online 250k checkpoint, EMA
      & Length 512, $T=1$, entropy grid \\
    OWT & Free-online 1M checkpoint, EMA
      & Length 1,024, entropy grid; score-temperature sweep \\
    \bottomrule
  \end{tabular}
\end{table}

\paragraph{Posterior diagnostics.}
At position $\ell$, write $q_k^\ell(v)=q_\theta^\ell(v\mid\mathbf z_k,r_k)$.
The token-posterior entropy, confidence indicator, and adjacent-state MAP
stability indicator are
\begin{equation}
  \begin{aligned}
    H_q^\ell(k)&=-\sum_v q_k^\ell(v)\log q_k^\ell(v),\\
    C_{q,c}^\ell(k)&=\mathbf 1\!\left[\max_v q_k^\ell(v)>c\right],\\
    S_q^\ell(k)&=\mathbf 1\!\left[
      \arg\max_v q_k^\ell(v)=\arg\max_v q_{k+1}^\ell(v)\right].
  \end{aligned}
  \label{eq:commitment-readiness-metrics}
\end{equation}
Entropy uses natural logarithms. We record confidence thresholds
$c\in\{0.9,0.99,0.999\}$ and show $c=0.99$ in
Figure~\ref{fig:language-onpolicy-readiness}. Confidence and stability rates summarize
the corresponding indicators over the evaluated positions and examples.
Stability measures agreement between token hypotheses, not agreement with
the correct answer. Posterior entropy also differs from the unigram
entropy of generated text (Equation~\ref{eq:unigram-entropy}). Prompt
positions are excluded throughout. The primary masks contain unknown
Sudoku digits, reference non-PAD TinyGSM answer positions, all GSM8K
answer slots including PAD, and all OWT positions. A separate GSM8K
diagnostic restricts to predicted non-PAD positions. Because the masks differ, their aggregate confidence statistics answer different questions. 
Adjacent-state stability compares predictions at the sampler's actual
successive radii. The q-argmax and Heun grids have different radius
increments at the matched language-call budget; both radii are recorded.
Their stability values are therefore not directly matched over a common
noise interval. Comparing stability across samplers therefore requires matched radius increments.

\paragraph{Diagnostic call accounting.}
Language trajectories use 256 sampling calls: 256 q-argmax steps,
128 two-call Heun SDE steps, or 111 Heun steps plus 34 q-argmax steps
for the native OWT hybrid ($2\times111+34=256$).
Each run also records one separate final
diagnostic call, giving 257 network evaluations in total. The reasoning
comparison below is therefore matched in both sampling and total call
counts. Sudoku trajectory diagnostics instead use the same 90-step noise
grid for both samplers, requiring 90 q-argmax calls versus 180 Heun calls.
These diagnostics compare posterior behavior along a shared grid; matched-NFE accuracy is reported separately in Table~\ref{tab:sudoku-full-nfe}.

\Needspace{4\baselineskip}
\paragraph{Forward-noised reference examples.}
Table~\ref{tab:commitment-forward-readiness} reports diagnostics near
$r=3.04$. Sudoku posteriors are almost fully resolved, TinyGSM is
intermediate, and OWT retains much greater uncertainty. These results describe readiness for a single projection; repeated-projection behavior is examined below.

\begin{table}[!htbp]
  \caption{Posterior diagnostics on forward-noised reference examples near
  $r=3.04$. Entropy is in nats; confidence and adjacent-noise MAP stability
  are percentages. Measurements are rounded.}
  \label{tab:commitment-forward-readiness}
  \centering
  \small
  \setlength{\tabcolsep}{8pt}
  \begin{tabular}{lrrr}
    \toprule
    Task & Posterior entropy & $\max q>0.99$ & MAP stability \\
    \midrule
    Sudoku easy & $\approx0$ & 100 & 100 \\
    Sudoku medium & $\approx0$ & 99.99 & 100 \\
    TinyGSM & 0.20 & 79.4 & 97.2 \\
    OWT & 5.61 & 0.51 & 55.3 \\
    \bottomrule
  \end{tabular}
\end{table}

\paragraph{Generated trajectories.}
The fixed-checkpoint Sudoku sweep in
Table~\ref{tab:sudoku-full-nfe} shows that inference procedure changes
accuracy at matched NFE. On hard puzzles, increasing the budget from
16 to 180 NFEs gives gains of 10.74 points for q-argmax and 8.56 points for
gated q-argmax. Nevertheless, trajectory
diagnostics near $r=3.15$ find that approximately 98\% to 99.5\% of unknown
Sudoku cells have $\max q>0.99$ for both q-argmax and Heun SDE, with
almost perfectly stable adjacent-state MAP predictions. These readiness traces use easy and medium puzzles with a learned \(d=768\), 20k-update checkpoint; the accuracy sweep uses the frozen \(d=128\) checkpoint.
At these measured noise levels on the shared 90-step grid, the SDE is
not simply uncertain while q-argmax is sharp. Because this diagnostic uses a different checkpoint and unequal call counts, it cannot explain the matched-NFE accuracy gap by itself.
The samplers act differently on their
predictions: q-argmax repeatedly selects a codeword and re-noises it,
whereas the SDE evolves a continuous state with the soft mixture score.

Table~\ref{tab:commitment-generated-diagnostics} gives the separate
reasoning and OWT comparisons at 256 sampling calls plus one diagnostic
call. For TinyGSM, q-argmax has
lower posterior entropy, a larger high-confidence fraction, and higher
accuracy; its accuracy advantage also appears on GSM8K. The pattern is consistent with a refinement effect, although confidence does not directly measure content-token correctness. TinyGSM's primary
statistics exclude reference PAD. GSM8K's PAD-inclusive confidence must
instead be interpreted alongside its separate predicted-non-PAD diagnostic.

\begin{table}[!htbp]
  \caption{Additional generated-trajectory diagnostics with 256 sampling
  calls plus one separate final diagnostic call (257 total evaluations).
  Accuracy and high-confidence fractions are percentages; both entropy
  metrics use nats but measure different quantities. These runs are
  separate from the main checkpoint sweeps.}
  \label{tab:commitment-generated-diagnostics}
  \centering
  \small
  \setlength{\tabcolsep}{8pt}
  \begin{tabular}{llrr}
    \toprule
    Task & Metric & q-argmax & Heun SDE \\
    \midrule
    TinyGSM & Posterior entropy & 0.12 & 0.32 \\
            & Positions with $\max q>0.99$ & 83.6 & 70.9 \\
            & Executed-answer accuracy & 39.1 & 20.3 \\
    \addlinespace
    GSM8K & Executed-answer accuracy & 24.2 & 12.5 \\
    \addlinespace
    OWT & Sample unigram entropy & $\approx0.014$ & 5.31 \\
    \bottomrule
  \end{tabular}
\end{table}

For OWT, pure q-argmax can become stable and confident while producing
almost constant text. Its near-zero sample unigram entropy is evidence
of within-sample token-diversity collapse, not successful generation.
The hybrid's continuous stochastic phase provides an alternative to
early repeated MAP commitment. The late-stage $k$ ablation in
Table~\ref{tab:owt-k-matched-temperature} separately shows that reducing
projection frequency increases both entropy and perplexity at fixed
temperature. The ablation therefore establishes only the trade-off among the tested projection frequencies.
\Needspace{5\baselineskip}
\paragraph{Irreversible freezing and padding.}
A separate GSM8K diagnostic uses a 512-token context and
1,024 NFEs, covering all 1,319 questions with sampling seeds 3 and 4.
Naive permanent confidence freezing solves 0 of 1,319 problems.
It is distinct from the 64-example trajectory diagnostics above, although
it uses the same context length as the main reasoning experiments.
One possible failure mechanism is premature PAD commitment:
confident PAD predictions can freeze positions that later need to become
content, removing the revision channel. Sudoku has no corresponding
variable-length output decision: every solution cell must contain a digit.

\paragraph{Interpretation.}
Our working explanation separates stable token hypotheses from the
desirability of mode-seeking. Constrained evaluation rewards recovery of
a correct answer, whereas unconditional text evaluation also requires
diversity. High confidence can accompany an incorrect token, premature PAD, or repetitive output, making it insufficient for irreversible commitment.

\begin{figure}[!htbp]
  \centering
  \includegraphics[width=\linewidth]{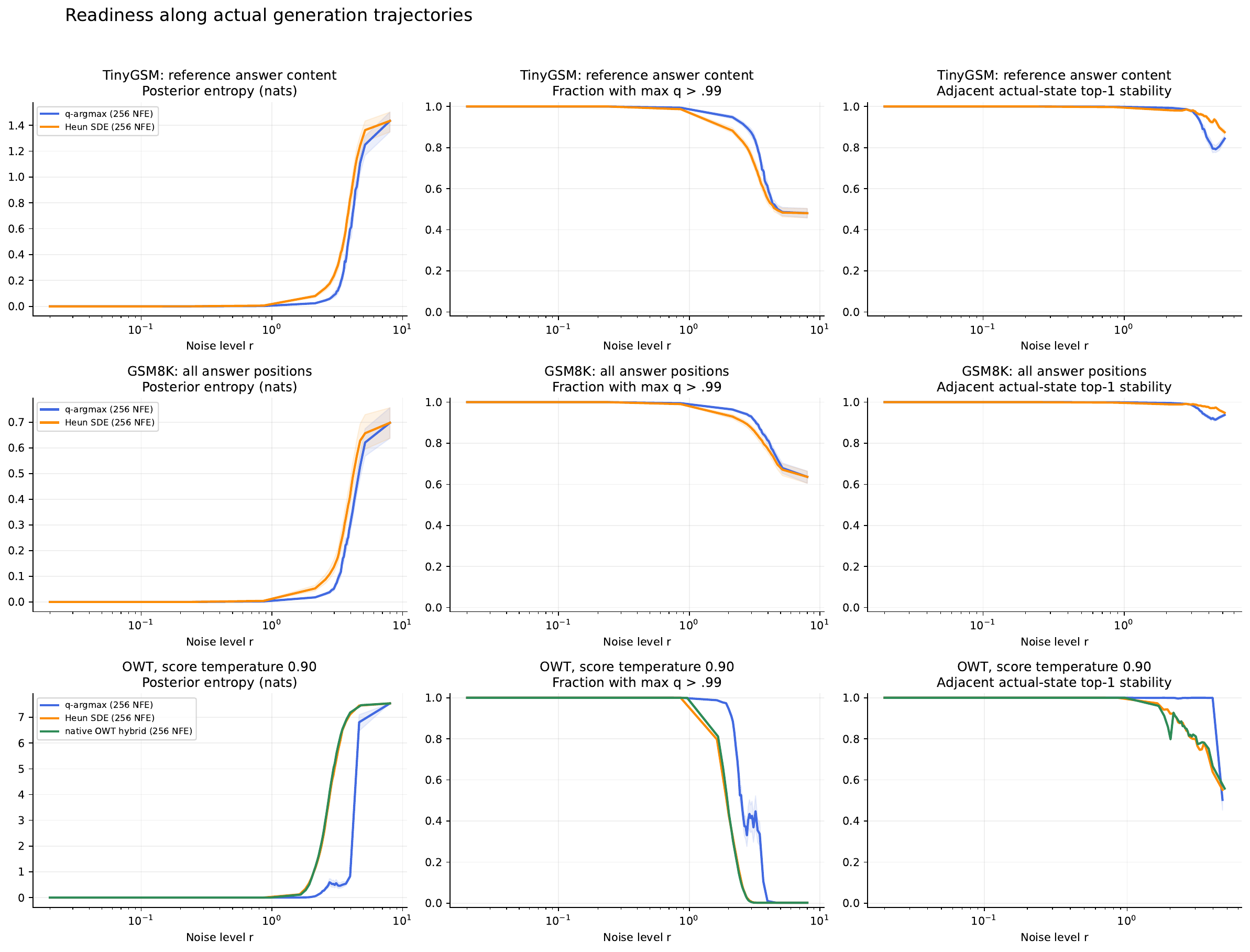}
  \caption{Posterior readiness along generated language trajectories.
  Rows: TinyGSM reference non-PAD answer positions, GSM8K answer slots
  including PAD, and OWT at score temperature $T=0.90$.
  Columns: posterior entropy, the fraction with $\max q>0.99$, and
  adjacent-state MAP stability. Blue: q-argmax; orange: Heun SDE;
  green: the native OWT hybrid. Each method uses 256 sampling calls
  plus one separate final diagnostic call. The horizontal axis is noise
  level $r$; generation proceeds from right to left. Stability uses
  different adjacent-radius increments across samplers and is not a
  common-interval comparison. GSM8K confidence includes PAD and should
  not be read as content-only certainty.}
  \label{fig:language-onpolicy-readiness}
\end{figure}

\subsection{Learned codebook geometry}
\label{app:codebook-geometry}
\label{sec:experiments:codebook-geometry}

\paragraph{Checkpoint and controls.}
This is a separate OWT diagnostic with $d=128$, not the $d=768$,
1M-update model used for the generative frontier. We analyze the free
learned codebook at 150k updates,
with anchor weight zero and separation weight one. Noise levels are drawn log-uniformly for the first 100k updates and from the entropy-based schedule from 100k to 150k. The controls are a frozen max-min packing
and a random spherical codebook at the same dimension. We compute all
diagnostics with the same evaluation code.

\begin{figure}[!htbp]
  \centering
  \includegraphics[width=\linewidth]{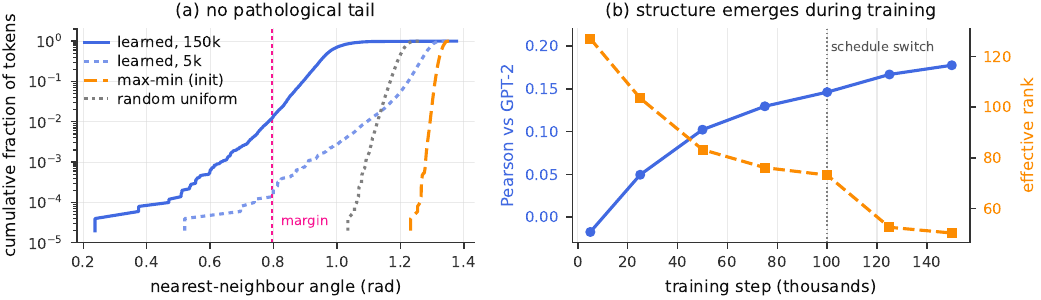}
  \caption{Geometry of the learned spherical codebook ($d=128$).
  \textbf{(a)} Empirical CDF of nearest-neighbor (NN) angles at 5k and
  150k updates, compared with frozen max-min and random uniform controls.
  Only 1.2\% fall below the separation margin
  $\arccos(0.7)\simeq0.7954$ radians.
  \textbf{(b)} GPT-2 cosine alignment increases while centered effective
  rank (participation ratio) decreases. The dashed line marks the
  log-to-entropy training-schedule switch at 100k; changes after it do
  not isolate the schedule's effect. This is a single-run diagnostic.}
  \label{fig:codebook-geometry}
\end{figure}

\paragraph{Geometry and alignment metrics.}
For each unit codeword, the nearest-neighbor angle is
\begin{equation}
  \theta_v=\arccos\!\left(
    \max_{w\neq v}\langle\mathbf e_v,\mathbf e_w\rangle\right).
  \label{eq:codebook-nn-angle}
\end{equation}
We summarize its median, first percentile, and minimum over the vocabulary.
Effective rank is the participation ratio of the centered covariance
spectrum: if $\lambda_i$ are its eigenvalues,
\begin{equation}
  r_{\mathrm{eff}}=
  \frac{(\sum_i\lambda_i)^2}{\sum_i\lambda_i^2}.
  \label{eq:codebook-effective-rank}
\end{equation}
Centering matters: at 150k, the uncentered value is 44.04, compared
with 50.38 after centering (50.4 when rounded).

Table~\ref{tab:codebook-geometry} reports the codebook diagnostics. The
primary Pearson alignment metric correlates codebook and GPT-2
pairwise cosine similarities on two million uniformly sampled pairs
over the full vocabulary. Only Pearson uses this pair sample. Top-10
overlap is evaluated over all 50,257 tokens shared with GPT-2, without
token sampling: it is the fraction whose nearest codebook neighbor also
appears among their ten nearest GPT-2 neighbors, using cosine similarity.
Leading-space agreement
measures whether a token and its nearest codebook neighbor share the
same leading-space form. This token-form diagnostic gives 55.0\% for both controls; semantic agreement is evaluated separately, and no analytical chance level is assumed.
\begin{table}[!htbp]
  \caption{OWT codebook diagnostics ($d=128$, 150k updates).
  Angles are in radians; rank is centered. Above-margin counts use NN
  cosine $>0.7$, or angle $<\arccos(0.7)\simeq0.7954$.
  Top-10 overlap uses all 50,257 shared tokens.}
  \label{tab:codebook-geometry}
  \centering
  \small
  \setlength{\tabcolsep}{6pt}
  \begin{tabular}{lrrr}
    \toprule
    Metric & Learned & Max-min & Random \\
    \midrule
    NN angle: median & 0.970 & 1.327 & 1.202 \\
    NN angle: first percentile & 0.784 & 1.294 & 1.128 \\
    NN angle: minimum & 0.237 & 1.232 & 1.035 \\
    Codewords with NN cosine $>0.7$ & 624 & 0 & 0 \\
    Centered effective rank & 50.4 & 127.9 & 127.7 \\
    Mean pairwise cosine & 0.041 & 0.00005 & $-0.00006$ \\
    Pearson alignment with GPT-2 & 0.176 & 0.0003 & 0.0013 \\
    NN in GPT-2 top ten (\%) & 21.4 & 0.01 & 0.02 \\
    Leading-space agreement (\%) & 95.3 & 55.0 & 55.0 \\
    \bottomrule
  \end{tabular}
\end{table}

\paragraph{Organization and separation.}
The codebook becomes denser and anisotropic, while most codewords remain
well separated: the first-percentile nearest-neighbor angle is 0.784
radians. Only 624 of 50,258 codewords (1.2\%) have NN cosine above
0.7, or angle below $\arccos(0.7)\simeq0.7954$ radians. This is
the positive-argument region of the softplus regularizer in
Equation~\ref{eq:codebook-regularizer}, not the fraction with any nonzero
loss: softplus also gives nonzero penalties below the margin.
Median angular displacement from the initial max-min anchor is
63.7 degrees. Greater displacement is associated with smaller final
nearest-neighbor angles (Spearman correlation $-0.764$).
Frequency rank, with larger ranks denoting rarer tokens, correlates
positively with nearest-neighbor angle (0.735): frequent tokens occupy
denser regions. The above-margin codewords have median
frequency rank 1,015, compared with 25,128 over the full vocabulary.
These correlations suggest denser codebook allocation around frequent tokens, but are observational.

\paragraph{Emergence during training.}
Table~\ref{tab:codebook-trajectory} separates the training trajectory
from the endpoint comparison. Between 5k and 100k, before the noise-draw
change, effective rank decreases from 127.0 to 73.2 while Pearson
alignment rises from $-0.017$ to 0.146 and top-10 overlap from 0.8\%
to 21.1\%. At 5k, alignment and token-form diagnostics remain much
closer to the controls than at later checkpoints. Further changes after 100k
confound additional training with the switch to entropy-based noise
draws. The pre-switch trajectory already shows progressive organization.

\begin{table}[!htbp]
  \caption{Codebook organization through training in one OWT run.
  The noise draw changes from log to entropy at 100k; comparisons across
  that boundary do not isolate the effect of additional training.
  Angles are in radians, and overlap and form agreement are percentages.}
  \label{tab:codebook-trajectory}
  \centering
  \small
  \setlength{\tabcolsep}{6pt}
  \begin{tabular}{rrrrrr}
    \toprule
    Updates & NN median & Eff. rank & Pearson
      & \shortstack{GPT-2\\top-10 (\%)}
      & \shortstack{Space\\agreement (\%)} \\
    \midrule
    5k & 1.279 & 127.0 & $-0.017$ & 0.8 & 57.8 \\
    25k & 1.160 & 103.5 & 0.050 & 6.7 & 79.1 \\
    50k & 1.082 & 83.3 & 0.102 & 14.5 & 85.2 \\
    75k & 1.055 & 76.2 & 0.130 & 18.9 & 86.0 \\
    100k & 1.044 & 73.2 & 0.146 & 21.1 & 86.2 \\
    \addlinespace
    125k & 0.970 & 52.6 & 0.167 & 21.3 & 94.8 \\
    150k & 0.970 & 50.4 & 0.176 & 21.4 & 95.3 \\
    \bottomrule
  \end{tabular}
\end{table}

\paragraph{Sensitivity to pair sampling.}
Pearson alignment depends on which token pairs are evaluated
(Table~\ref{tab:codebook-pair-sampling}). We use uniform sampling over
the full vocabulary as the primary reported metric. Frequency-weighted and restricted-vocabulary results serve as descriptive sensitivity checks alongside the primary headline metrics. The rise in
alignment among frequent tokens is consistent with the density analysis.

\begin{table}[!htbp]
  \caption{Sensitivity of GPT-2 pairwise cosine alignment to pair sampling
  for the 150k learned codebook. Uniform full-vocabulary sampling is the
  primary metric; restricted subsets are diagnostic comparisons.}
  \label{tab:codebook-pair-sampling}
  \centering
  \small
  \setlength{\tabcolsep}{8pt}
  \begin{tabular}{lr}
    \toprule
    Pair sampling & Pearson \\
    \midrule
    Uniform, full vocabulary & 0.176 \\
    Frequency-weighted by unigram prior & 0.210 \\
    Uniform, 20,000 most frequent tokens & 0.277 \\
    Uniform, 8,000 most frequent tokens & 0.382 \\
    Uniform, 2,000 most frequent tokens & 0.463 \\
    \bottomrule
  \end{tabular}
\end{table}

This single-run analysis documents organization of the learned codebook; its effect on generation quality remains unisolated.
Figure~\ref{fig:codebook-geometry} presents the empirical NN-angle CDFs
and the alignment/rank trajectory.
\FloatBarrier

\subsection{Numerical validation of the heat kernel}
\label{app:kernel-validation}

These diagnostics evaluate the conditional corruption channel independently of a vocabulary or trained denoiser. They assess numerical consistency and the asymptotic analysis, while downstream language-model accuracy is evaluated separately.

\paragraph{Clock and semigroup.}
For $d\in\{32,128,768\}$, we estimate the retained mean cosine from forward
draws and compare it with $e^{-r}$. Relative errors are reported only when
the target is at least ten Monte Carlo standard errors above zero;
otherwise the small denominator makes relative error misleading.
This is a consistency check because the clock is defined by
$r=-\log\rho$.

For the semigroup diagnostic, we compare a direct draw at $r_1+r_2$ with
two successive corruptions at $r_1$ and $r_2$. Each draw set contains
400k samples. We compare the empirical distributions of the final cosine
relative to the initial codeword using the two-sample
Kolmogorov-Smirnov distance,
\begin{equation}
  D_{\mathrm{KS}}
  =\sup_c\left|\widehat F_{\mathrm{direct}}(c)
  -\widehat F_{\mathrm{sequential}}(c)\right|.
  \label{eq:semigroup-ks-distance}
\end{equation}
A null control compares two independent direct draw sets at the same
radius and sample count. This estimates the finite-sample floor: a small
distance is not evidence of exact agreement. Across the three dimensions
and seven tested radius pairs, the distance is at most $0.023$, but the
worst discrepancies exceed the null floor. We therefore distinguish the
exact analytic heat semigroup from its approximate numerical sampler.

\paragraph{Moment-matched heat-vMF comparison.}
We use the validated Heat tables described in
Appendix~\ref{app:implementation-details}, at
$d\in\{16,32,64,128,256,512,768,1024\}$. At each operating point, the vMF
concentration is chosen to match the heat law's mean cosine, using
deterministic quadrature on the same cosine grid as the error integrals.
This avoids using a Bessel-ratio approximation outside its stable range.
The coordinate densities include the spherical volume factor
$(1-c^2)^{(d-3)/2}$ and are normalized before integration.

Let $f_H(c)$ and $f_V(c)$ be these cosine densities and
$a_H(c)=\partial_c\log K^H_{d,r}(c)$. We measure the RMS relative radial-score
error and the conditional channel KL:
\begin{align}
  R_{\mathrm{score}}
  &=\left[\int_{-1}^{1}f_H(c)
  \left(\frac{a_H(c)}{\kappa^\star}-1\right)^2\,\mathrm dc\right]^{1/2},
  \label{eq:numerical-score-rms}\\
  \KL(H\Vert V)
  &=\int_{-1}^{1}f_H(c)\log\frac{f_H(c)}{f_V(c)}\,\mathrm dc.
  \label{eq:numerical-channel-kl}
\end{align}
The main scaling sweep follows $r_d=(\log d+\xi)/2$ for
$\xi\in\{0,1,2\}$ and compares
$\widehat m_{\mathrm{score}}=R_{\mathrm{score}}/t_d$ and
$\widehat m_{\mathrm{KL}}=(4\KL/t_d^2)^{1/4}$ with
$m=e^{-\xi/2}$. RMS is not divided by the Gaussian mean-absolute conversion
factor; that factor applies only when starting from a mean-absolute error.
We also report the fixed-$r$ diagnostic separately because it is a different
asymptotic regime. These are codebook-free conditional errors, not posterior
KL under a learned vocabulary.

Figure~\ref{fig:heat-vmf-critical-scaling} gives the numerical validation of
Proposition~\ref{prop:heat-vmf-approximation}: across eight dimensions and
three critical-noise paths, both coefficient estimates approach the predicted
target, with finite-dimensional discrepancies shrinking slowly.
\begin{figure}[t]
  \centering
  \includegraphics[width=0.95\linewidth]{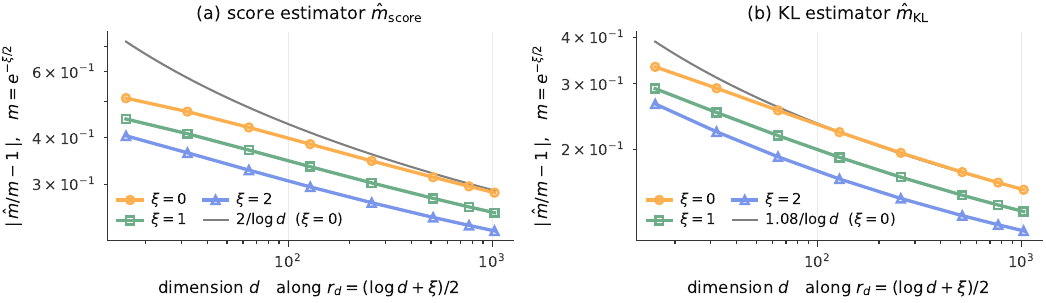}
  \caption{Numerical validation of Proposition~\ref{prop:heat-vmf-approximation}
  along $r_d=(\log d+\xi)/2$ for $\xi\in\{0,1,2\}$ and $d=16$ to $1024$.
  The score- and KL-based coefficient estimates approach the unfitted target
  $m=e^{-\xi/2}$; gray curves show $1/\log d$ guides.}
  \label{fig:heat-vmf-critical-scaling}
\end{figure}

\section{Technical estimates for the heat-vMF proof}
\label{app:heat-vmf-technical}

This appendix supplies the analytic details used in
Appendix~\ref{app:heat-vmf-reduction}. We keep the same notation and fix
$\xi$ throughout. The bounds below justify differentiation and the
likelihood Taylor expansion; the numerical sweep is not used in the proof.

\subsection{Coordinate moments and concentration}
\label{app:coordinate-moments}

In the dimension-free clock $r=(d-1)t/2$, the rescaled cosine obeys
\begin{equation}
  \mathrm dU_r=-U_r\,\mathrm dr
  +\sqrt{\frac{2(d-U_r^2)}{d-1}}\,\mathrm dB_r,
  \qquad U_0=\sqrt d.
  \label{eq:appendix-coordinate-sde}
\end{equation}
At $r=r_d$, variation of constants gives
\begin{equation}
  U_{r_d}=m+
  \int_0^{r_d}e^{-(r_d-s)}
  \sqrt{\frac{2(d-U_s^2)}{d-1}}\,\mathrm dB_s.
  \label{eq:appendix-coordinate-martingale}
\end{equation}
The martingale has quadratic variation at most $d/(d-1)$, so
\begin{equation}
  \Pr_H(|U-m|\geq x)
  \leq2\exp\!\left[-\frac{d-1}{2d}x^2\right].
  \label{eq:appendix-heat-subgaussian-tail}
\end{equation}
The generator applied to $u^j$ also gives the exact recursion
\begin{equation}
  \frac{\mathrm d}{\mathrm dr}\E[U_r^j]
  =-\frac{j(d+j-2)}{d-1}\E[U_r^j]
   +\frac{j(j-1)d}{d-1}\E[U_r^{j-2}].
  \label{eq:appendix-moment-recursion}
\end{equation}
Induction in $j$ yields bounded moments of every fixed order.  Solving the
recursion explicitly for $j\leq4$ also gives
\begin{equation}
  \Var_H[\operatorname{He}_2(U-m)]
  =2+O(t_d+d^{-1}).
  \label{eq:appendix-hermite-variance}
\end{equation}

The vMF coordinate density is uniformly strongly log-concave.  Indeed, its
negative log-density has
\begin{equation}
  W_d''(u)
  =\frac{(d-3)(d+u^2)}{(d-u^2)^2}
  \geq\frac{d-3}{d},
  \label{eq:appendix-vmf-convexity}
\end{equation}
and its mean is $m$ by construction.  It therefore has a Gaussian tail of
the same form as Equation~\ref{eq:appendix-heat-subgaussian-tail}, up to
absolute constants.

\subsection{Proof of the heat-Poisson estimate}
\label{app:heat-poisson-proof}

Formally expanding the operator in
Equation~\ref{eq:appendix-exact-poisson-bridge} to first order gives
\begin{equation}
  \log\frac{K^H_{d,r_d}(u/\sqrt d)}{P_{q_d,d}(u)}
  =-\frac{m^2t_d}{2}\operatorname{He}_2(u-m)+\text{remainder},
  \label{eq:appendix-formal-heat-poisson}
\end{equation}
because direct differentiation of
Equation~\ref{eq:appendix-poisson-kernel} yields, on the central scale,
\begin{equation}
  \frac{D(D-1)P_{q_d,d}(u)}{P_{q_d,d}(u)}
  =m^2\operatorname{He}_2(u-m)
   +O\!\left(\frac{1+|u|^4}{d}\right).
  \label{eq:appendix-poisson-operator-leading}
\end{equation}
Here and below $\operatorname{He}_2(x)=x^2-1$.

The remainders in Lemma~\ref{lem:heat-poisson-central} also satisfy the
following bound for some fixed polynomial exponent $k$ and every
sufficiently large fixed $A$, uniformly on
$|u-m|\leq A\sqrt{\log d}$, with
\begin{equation}
  |R_d^H(u)|+|S_d^H(u)|
  \leq C(t_d^2+t_d/d)(1+|u|^k).
  \label{eq:appendix-heat-poisson-central-bound}
\end{equation}

\begin{proof}[Proof of Lemma~\ref{lem:heat-poisson-central}]
Define the rescaled Poisson kernel
\begin{equation}
  \Pi_d(z,u)
  :=P_{z/\sqrt d,d}(u)
  =\left(1-\frac{z^2}{d}\right)
   \left(1-\frac{2zu}{d}+\frac{z^2}{d}\right)^{-d/2}.
  \label{eq:appendix-rescaled-poisson}
\end{equation}
Thus $P_{q_d,d}(u)=\Pi_d(m,u)$.  If $Z\sim\mathcal N(0,1)$ and
$W_d=t_d/2+i\sqrt{t_d}Z$, then
\begin{equation}
  \E[e^{nW_d}]
  =e^{-n(n-1)t_d/2}.
\end{equation}
Termwise expectation in the absolutely convergent zonal expansion therefore
turns Equation~\ref{eq:appendix-exact-poisson-bridge} into the exact identity
\begin{equation}
  K^H_{d,r_d}(u/\sqrt d)
  =\E\!\left[\Pi_d(me^{W_d},u)\right].
  \label{eq:appendix-complex-gaussian-bridge}
\end{equation}
Indeed, $|q_de^{W_d}|=q_de^{t_d/2}<1$ for all sufficiently large $d$.

We next control the analytic Taylor remainder.  Put
\begin{equation}
  f_{d,u}(s)=\frac{\Pi_d(me^s,u)}{\Pi_d(m,u)}.
\end{equation}
For every fixed $A$, direct differentiation of
Equation~\ref{eq:appendix-rescaled-poisson} gives fixed integers $K$ and
constants $C_A$ such that, for $j\leq1$, $\ell\leq4$, and
$|u-m|\leq A\sqrt{\log d}$,
\begin{equation}
  \left|\partial_u^j\partial_s^\ell
  f_{d,u}(\theta W_d)\right|
  \leq C_A(1+|u|)^K
  \exp\!\left\{C_A t_d(1+|u|)(1+Z^2)\right\},
  \qquad 0\leq\theta\leq1.
  \label{eq:appendix-analytic-derivative-bound}
\end{equation}
To verify this bound, note that $|me^{\theta W_d}|=O(1)$ and the two
denominators in Equation~\ref{eq:appendix-rescaled-poisson} are bounded away
from zero.  Its logarithmic derivatives are therefore polynomial in $u$.
For the ratio itself, taking real parts and using
$1-\cos(\theta\sqrt{t_d}Z)\leq\theta^2t_dZ^2/2$ gives the exponential factor
in Equation~\ref{eq:appendix-analytic-derivative-bound}.  Since
$t_d(1+|u|)=o(1)$ on this window, the right-hand side has a uniformly bounded
Gaussian expectation after multiplication by $|W_d|^4$.

Taylor expansion through degree three, followed by expectation, uses
\begin{equation}
  \E W_d=\frac{t_d}{2},\qquad
  \E W_d^2=-t_d+\frac{t_d^2}{4},\qquad
  \E W_d^3=-\frac{3t_d^2}{2}+\frac{t_d^3}{8},
\end{equation}
and $\E|W_d|^4=O(t_d^2)$.  Equation~\ref{eq:appendix-analytic-derivative-bound}
then yields, also after one $u$ derivative,
\begin{equation}
  \frac{K^H_{d,r_d}(u/\sqrt d)}{P_{q_d,d}(u)}
  =1-\frac{t_d}{2}B_d(u)
   +O\!\left(t_d^2(1+|u|)^K\right),
  \label{eq:appendix-heat-poisson-ratio-taylor}
\end{equation}
where
\begin{equation}
  B_d(u)=
  \frac{[D(D-1)\Pi_d](m,u)}{\Pi_d(m,u)},
  \qquad D=z\partial_z.
\end{equation}
For clarity, the first-order coefficient follows because
$f_{d,u}''(0)-f_{d,u}'(0)=B_d(u)$; all second- and third-moment contributions
not shown explicitly are $O(t_d^2)$.

The coefficient is elementary.  From
\begin{equation}
  D\log\Pi_d(z,u)
  =\frac{dz(u-z)}{d-2zu+z^2}-\frac{2z^2}{d-z^2}
\end{equation}
and
$B_d=D^2\log\Pi_d-D\log\Pi_d+(D\log\Pi_d)^2$ at $z=m$, expansion of the
two denominators gives, for $j\leq1$,
\begin{equation}
  \partial_u^j\!\left[
    B_d(u)-m^2\operatorname{He}_2(u-m)
  \right]
  =O\!\left(\frac{1+|u|^4}{d}\right).
  \label{eq:appendix-poisson-operator-c1}
\end{equation}
On the central window, $t_dB_d(u)=o(1)$.  Applying $\log(1+x)$ to
Equation~\ref{eq:appendix-heat-poisson-ratio-taylor}, differentiating once,
and using Equation~\ref{eq:appendix-poisson-operator-c1} proves
Equations~\ref{eq:appendix-heat-poisson-density},
\ref{eq:appendix-heat-poisson-score}, and the pointwise bound
in Equation~\ref{eq:appendix-heat-poisson-central-bound}.

It remains to pass from the central window to the heat-weighted norms.  The
moment recursion in Equation~\ref{eq:appendix-moment-recursion} gives
$\sup_d\E_H|U|^p<\infty$ for every fixed $p$, and
Equation~\ref{eq:appendix-heat-subgaussian-tail} makes the complement of
$|U-m|\leq A\sqrt{\log d}$ smaller than any prescribed polynomial order when
$A$ is chosen large enough.  We only need polynomial global bounds for the
remainders and the radial score.  We record them next.

Let $K_t(\mathbf z,\mathbf e)$ denote the spherical heat kernel relative to
uniform measure, with generator $\tfrac12\Delta_{\mathbb S^{d-1}}$.  Positivity
of the spectral coefficients and $|P_{n,d}(c)|\leq1$ give
\begin{equation}
  K_{t_d}(\mathbf z,\mathbf e)
  \leq K_{t_d}(\mathbf e,\mathbf e)
  \leq P_{q_d,d}(\sqrt d)=\exp(O(\sqrt d)).
  \label{eq:appendix-heat-global-upper}
\end{equation}
The Li-Yau parabolic Harnack inequality on the sphere
\citep{li1986parabolic} gives
\begin{equation}
  K_{t_d/2}(\mathbf e,\mathbf e)
  \leq 2^{(d-1)/2}
  \exp\!\left(\frac{\operatorname{dist}(\mathbf z,\mathbf e)^2}{t_d}\right)
  K_{t_d}(\mathbf z,\mathbf e).
  \label{eq:appendix-li-yau-lower}
\end{equation}
Since $K_{t_d/2}(\mathbf e,\mathbf e)\geq1$ and the spherical distance is at
most $\pi$, Equations~\ref{eq:appendix-heat-global-upper} and
\ref{eq:appendix-li-yau-lower}, together with the explicit Poisson formula,
imply
\begin{equation}
  \left|\log\frac{K^H_{d,r_d}(u/\sqrt d)}{P_{q_d,d}(u)}\right|=O(d)
  \qquad (|u|<\sqrt d).
  \label{eq:appendix-global-log-ratio}
\end{equation}
Thus the density remainder is polynomially bounded on its full support, and
the sub-Gaussian tail upgrades its central $L^2$ and $L^3$ estimates to
Equations~\ref{eq:appendix-heat-poisson-density-bound}.

For the differentiated tail, apply Hamilton's gradient estimate to
$K_{t_d/2+s}(\mathbf z,\mathbf e)$ for $0\leq s\leq t_d/2$
\citep{kotschwar2007hamilton}.  Its spacetime maximum is
$M_d=K_{t_d/2}(\mathbf e,\mathbf e)$, and
\begin{equation}
  M_d\leq P_{\sqrt{q_d},d}(\sqrt d)=\exp(O(d^{3/4})).
\end{equation}
After rescaling from generator $\Delta$ to $\tfrac12\Delta$, Hamilton's
estimate and Equation~\ref{eq:appendix-li-yau-lower} give
\begin{equation}
  |\nabla\log K_{t_d}(\mathbf z,\mathbf e)|^2
  \leq \frac{C}{t_d}\log\frac{M_d}{K_{t_d}(\mathbf z,\mathbf e)}
  \leq \frac{Cd}{t_d}.
  \label{eq:appendix-hamilton-score-bound}
\end{equation}

For $c=\langle\mathbf z,\mathbf e\rangle$ and $u=\sqrt d\,c$, radiality gives
almost everywhere
\begin{equation}
  |\partial_u\log K_{t_d}|^2
  =\frac{|\nabla\log K_{t_d}|^2}{d(1-c^2)}.
  \label{eq:appendix-radial-gradient-relation}
\end{equation}
On $|c|\leq1/2$, this is $O(t_d^{-1})$, so the sub-Gaussian tail controls the
part outside the central window.  On $|c|>1/2$, using
Equation~\ref{eq:appendix-heat-global-upper} and the uniform-coordinate
density gives
\begin{align}
  \E_H\!\left[|\partial_u\log K_{t_d}|^2
              \mathbf1_{\{|C|>1/2\}}\right]
  &\leq \frac{C}{t_d}e^{C\sqrt d}
    \int_{|c|>1/2}\frac{\nu_d(\mathrm dc)}{1-c^2} \\
  &\leq \operatorname{poly}(d)e^{C\sqrt d}
    \left(\frac34\right)^{(d-5)/2}
  =o(d^{-B})
  \label{eq:appendix-endpoint-score-tail}
\end{align}
for every fixed $B$, where $\nu_d$ is the uniform coordinate law of $C$.
Finally, $\partial_u\log P_{q_d,d}(u)=O(1)$ globally, and the leading linear
term in Equation~\ref{eq:appendix-heat-poisson-score} has the same
sub-Gaussian tail.  This proves
Equation~\ref{eq:appendix-heat-poisson-score-bound}.  Centering $R_d^H$
changes its $L^p$ norm by at most a factor of two, completing the proof.
\end{proof}

\subsection{Explicit Poisson-vMF calibration}
\label{app:poisson-vmf-calibration}

This part of the comparison is explicit.  Let $\mu_d$ be the uniform
coordinate law
\begin{equation}
  \mu_d(\mathrm du)
  =c_d\left(1-\frac{u^2}{d}\right)^{(d-3)/2}
  \mathbf 1_{\{|u|<\sqrt d\}}\,\mathrm du,
  \label{eq:appendix-uniform-coordinate-law}
\end{equation}
and define its log-partition function
\begin{equation}
  \psi_d(\lambda)=\log\int e^{\lambda u}\,\mu_d(\mathrm du).
\end{equation}
The exact even moments
\begin{equation}
  \E_{\mu_d}[U^{2j}]
  =d^j\frac{(1/2)_j}{(d/2)_j}
\end{equation}
give, uniformly for $\lambda$ in compact sets,
\begin{equation}
  \psi_d(\lambda)
  =\frac{\lambda^2}{2}-\frac{\lambda^4}{4d}+O(d^{-2}),
  \qquad
  \psi_d'(\lambda)
  =\lambda-\frac{\lambda^3}{d}+O(d^{-2}).
  \label{eq:appendix-uniform-log-partition}
\end{equation}
Writing $\lambda_d^\star=\kappa_d^\star/\sqrt d$, moment matching is
$\psi_d'(\lambda_d^\star)=m$.  Hence
\begin{equation}
  \lambda_d^\star=m+\frac{m^3}{d}+O(d^{-2}),
  \qquad
  \kappa_d^\star
  =m\sqrt d+\frac{m^3}{\sqrt d}+O(d^{-3/2}).
  \label{eq:appendix-vmf-concentration-expansion}
\end{equation}

For $q_d=m/\sqrt d$, direct expansion of
Equation~\ref{eq:appendix-poisson-kernel} gives
\begin{align}
  \log P_{q_d,d}(u)
  &=mu-\frac{m^2}{2}
    +\frac{-m^2+m^4/4-m^3u+m^2u^2}{d}
    +O\!\left(\frac{1+|u|^3}{d^2}\right),
  \label{eq:appendix-poisson-log-expansion}\\
  \log K^V_{d,\kappa_d^\star}(u/\sqrt d)
  &=mu-\frac{m^2}{2}
    +\frac{m^3u-3m^4/4}{d}
    +O\!\left(\frac{1+|u|}{d^2}\right).
  \label{eq:appendix-vmf-log-expansion}
\end{align}
Subtracting and differentiating proves
Equations~\ref{eq:appendix-poisson-vmf-density} and
\ref{eq:appendix-poisson-vmf-score}.
The deliberately loose polynomial envelopes in
Equations~\ref{eq:appendix-poisson-vmf-density} and
\ref{eq:appendix-poisson-vmf-score} are convenient on the growing central
region.  On the full support, the explicit formulas give a bound polynomial
in $d$ for the log-ratio and its derivative.  Splitting at
$|u-m|=A\sqrt{\log d}$ and using the heat-law sub-Gaussian estimate proved in
Equation~\ref{eq:appendix-heat-subgaussian-tail} therefore upgrades the
displayed remainders to $O(d^{-2})$ in the fixed $L^p(H)$ norms used below,
after increasing $A$ if necessary.  No uniform Bessel asymptotic is required
for this Poisson-vMF step.

\subsection{Exponentially weighted likelihood tails}
\label{app:likelihood-tail-control}

For completeness, the exponentially weighted likelihood tail can be
controlled without assuming exponential moments of the residual.  If
$X_d=L_d-\E_HL_d$ and
$G_d=\{|U-m|\leq A\sqrt{\log d}\}$, then the exact change of measure gives
\begin{equation}
  \E_H[e^{-X_d}\mathbf1_{G_d^c}]
  =e^{\KL(H\Vert V)}\Pr_V(G_d^c).
  \label{eq:appendix-tail-change-measure}
\end{equation}
The prefactor is bounded.  Spectral orthogonality and
Equation~\ref{eq:appendix-heat-coefficient-factorization} imply
\begin{equation}
  \int (K^H_{d,r_d})^2\,\mathrm d\mu_d
  \leq\sum_{n\geq0}h_{n,d}q_d^{2n}
  =P_{q_d^2,d}(\sqrt d)=O(1).
\end{equation}
Hence $\E_H\log K^H\leq O(1)$, while
$-\E_H\log K^V=-\kappa_d^\star\E_HC+\psi_d(\lambda_d^\star)=O(1)$.
Thus $\KL(H\Vert V)=O(1)$.  Equations~\ref{eq:appendix-heat-subgaussian-tail},
\ref{eq:appendix-vmf-convexity}, and
\ref{eq:appendix-tail-change-measure} make the exponentially weighted
complementary term smaller than any required polynomial order by choosing
$A$ sufficiently large.  In addition,
Equation~\ref{eq:appendix-density-remainder} gives
$\|X_d\|_{L^3(H)}=O(t_d)$; H\"older's inequality and
Equation~\ref{eq:appendix-heat-subgaussian-tail} then control the linear and
quadratic complementary terms.  On $G_d$,
the central expansions in Lemma~\ref{lem:heat-poisson-central} and
Equation~\ref{eq:appendix-poisson-vmf-density} give
$\sup_{G_d}|X_d|=o(1)$, while the global $L^3$ bound controls the cubic Taylor
remainder.  Thus this argument requires no additional tail assumption.

\clearpage
\raggedbottom
\section{OpenWebText generated samples}
\label{app:owt-generated-samples}

We show the first 140 words of sample index 0 at each saved point
of our 1M-update OWT model, without quality ranking or text edits.
Temperatures select the nearest saved nondominated point to each
target entropy. Reported $H$ and GenPPL describe the source point,
not the passage; different NFEs are separate generations, not revisions.

\texttt{[U+FFFD]} marks a decoder replacement character already in
the output. These passages are not verified factual statements.

\subsection{Target entropy \texorpdfstring{$H=4.5$}{H=4.5} nats}
\label{app:owt-samples-h4-5}

\begin{owtsamplebox}{Ours (1M), budget: 32 NFE\hfill $H_\mathrm{target}=4.5,\ k=1$}
{\color{black!65} $T=0.70$; $H=4.5213$ nats;
GenPPL $=29.17$; 32 calls; $n=512$.}\par
\textless{}\textbar{}endoftext\textbar{}\textgreater{} it.”

At the conference, the designation attracted the support of the state’s development player, Gov. Peter L. Turner, who said the state would be looking for a jack presence.

Mobhan said the move was part of the state's "undefastic vision" for the state. In that description, he noted discussions between the town of Baltimore Smith and the surrounding community the need to find local funding for the development.

He also referred to the cap of the amount that an entity can fund Pro. in the time they go to the state to serve on the U.S. Coast.

Mobhan said the had changed in the past.

Games and the end of its game

Burson said that some of the Moody’s schools, and the future of the state’s history, will be under pressure to fund the program.

Andy Jensen, the longtime
\par{\color{black!55}\itshape [Excerpt ends here.]}
\end{owtsamplebox}

\begin{owtsamplebox}{Ours (1M), budget: 64 NFE\hfill $H_\mathrm{target}=4.5,\ k=1$}
{\color{black!65} $T=0.75$; $H=4.4612$ nats;
GenPPL $=27.33$; 64 calls; $n=512$.}\par
\textless{}\textbar{}endoftext\textbar{}\textgreater{} it.”

At the time, though, there was an army of loud tour musicians \textemdash{} “on the building, trying to get it from 6 a.m. to 3 a.m.,” Dorander said.

“We don’t know if we can keep to with it. Why do the club?”

One or two sessions discussing the tour with the rest of the announcement.

“When they say it was possibly terrorism, it’s an invasion, which is a big deal. They just can’t battle the street,” she said.

Pam also wondered what happened the third day shaking the conductor’s hand.

“A pile of thought it had to do with the rest of the band and the sound of it,” she said.

“You don’t know where people are going,’ you say, ‘We don’t want a reason or how the rest of this happened.’

“It took us half a week
\par{\color{black!55}\itshape [Excerpt ends here.]}
\end{owtsamplebox}

\begin{owtsamplebox}{Ours (1M), budget: 128 NFE\hfill $H_\mathrm{target}=4.5,\ k=1$}
{\color{black!65} $T=0.75$; $H=4.4625$ nats;
GenPPL $=21.06$; 127 calls; $n=512$.}\par
\textless{}\textbar{}endoftext\textbar{}\textgreater{}.

“Most people said this happen at the end, but many don’t want to get up the road \textellipsis{} trying to get it at 3 a.m. to 7 a.m.,” Mathic said.

“We don’t know if it will happen, but then we will have to stay.”

Meanwhile, buses were shutting around last year before the river, but it would be a longer time to go back, especially in the frozen foothills of the park.

The trio expects it will remain a project.

On Friday afternoon, after a chat, a friend of business, Shirley, made a formal opinion on the tunnel’s construction.

“The road and the main technique, and the bridge itself, that’s important for us,” she said.

“You don’t complain about it. You don’t talk to church. We don’t have homeless people a lot, but [under the road], we don’t have
\par{\color{black!55}\itshape [Excerpt ends here.]}
\end{owtsamplebox}

\begin{owtsamplebox}{Ours (1M), budget: 256 NFE\hfill $H_\mathrm{target}=4.5,\ k=1$}
{\color{black!65} $T=0.75$; $H=4.6289$ nats;
GenPPL $=22.67$; 256 calls; $n=512$.}\par
\textless{}\textbar{}endoftext\textbar{}\textgreater{}.

600,000.

And the error rate is very bad, but we’ll have to go to the Sun to know what a lot will go on yet.

And that:

Two Devapse from Earth’s Sun

From the Sun eye that will begin, and the planets will begin to return to Earth, in a minute, then in two and a half of days.

This will be a pretty day, it will be to see the Sun, one of the Spicon. In two days it will see the Sun.

They will make it to the Sun orbit before view of the entire moon, in the great northern part of the Moon. From there, they will see a moon will unfold between their galactic side and the northern side, much as in the northern part, and east.

The moon will be pointed to see the
\par{\color{black!55}\itshape [Excerpt ends here.]}
\end{owtsamplebox}

\begin{owtsamplebox}{Ours (1M), budget: 512 NFE\hfill $H_\mathrm{target}=4.5,\ k=1$}
{\color{black!65} $T=0.75$; $H=4.3232$ nats;
GenPPL $=16.33$; 512 calls; $n=512$.}\par
\textless{}\textbar{}endoftext\textbar{}\textgreater{} of 10,000,000.

Therefore the creation of the three months of creation, with that day of earth, will be constituted into account of its age, while, on the earth, It will consist of 10,000 years and shall be six days.

And the whole generation which will be composed of the ages, at a time of 70,000 and 50,000 days according to the names of one.

I can say that this generation the ages will be divided to the earth, lasting by 10,000 days, and sixty-three in a longer time.

That will be the whole period of the whole life of the earth seven, in the whole life of all the stars. It is, that the age between two years, all the earth, and the year thirty, can be as many as 34,000 years.

This earth will be formed in the
\par{\color{black!55}\itshape [Excerpt ends here.]}
\end{owtsamplebox}

\begin{owtsamplebox}{Ours (1M), budget: 1,024 NFE\hfill $H_\mathrm{target}=4.5,\ k=1$}
{\color{black!65} $T=0.80$; $H=4.5494$ nats;
GenPPL $=16.20$; 1,024 calls; $n=512$.}\par
\textless{}\textbar{}endoftext\textbar{}\textgreater{} day of October 11, 1955.

For the specific quantity of the operation of this period, which is not to say it is to be given to the manufacturer, or any of the services, may have been called the preceding medium, in provision to the company. So there is an additional factor which is that the operation of the operation, the, after the date of October 1, 1955, prior to the beginning of the operation, appears on the ground as to the quantity of the plant, and the quantity, producing the water produced in this necessary quantity, and does not affect the relation between the water company and the plant, and in the important quantity of water which was produced in the plant where it was produced. Then there is a second factor, but one must hypothesise that there are
\par{\color{black!55}\itshape [Excerpt ends here.]}
\end{owtsamplebox}

\Needspace{8\baselineskip}

\subsection{Target entropy \texorpdfstring{$H=5.0$}{H=5.0} nats}
\label{app:owt-samples-h5-0}

\begin{owtsamplebox}{Ours (1M), budget: 32 NFE\hfill $H_\mathrm{target}=5.0,\ k=1$}
{\color{black!65} $T=0.80$; $H=4.9733$ nats;
GenPPL $=46.68$; 32 calls; $n=512$.}\par
\textless{}\textbar{}endoftext\textbar{}\textgreater{} it.”

Not only that, the move shows the support of the men’s golf team. Rep. Peter L. Schumer, D-N.C., voted for the drawats.

Hobna said the move was part of the state's "irquastic project" for several reasons. In this case, he noted discussions with a group of state athletes and the community and the need to find local benefits for the team.

He also outlined in in the bonds the amount that an organization can pay. and in the time they get to the team they play on the U.S. Coast.

Hobna said this went out in the past.

Games into the General Master' game

Bash said that the growth and growth of interest in soccer, and the future of the state’s pro league will be sorely positioned to support the change.

He said, the next group of legislators
\par{\color{black!55}\itshape [Excerpt ends here.]}
\end{owtsamplebox}

\begin{owtsamplebox}{Ours (1M), budget: 64 NFE\hfill $H_\mathrm{target}=5.0,\ k=1$}
{\color{black!65} $T=0.85$; $H=5.0222$ nats;
GenPPL $=46.51$; 64 calls; $n=512$.}\par
\textless{}\textbar{}endoftext\textbar{}\textgreater{} there.”

At the Sunday, Saturday, there is an army of street music bands \textemdash{} “finding the building, trying to get it around 6 a.m. to 3 a.m.,” Mathas said.

“We don’t know if we can keep away with kids. We love the club.”

Students and musician sessions break up a rally in the downtown, Friday.

Last July, Progressive announced it was interested in the Queen’s Centre Project, which is a landmark venture. They were participating in what will tour the stations, offices and transport rooms.

Corza questioned what happened the teens an affirmation of Ontario’s needs.

“These sorts of things have nothing to do with the values of the province and the values of all,” she said.

“You don’t know where people are going,” she added, “We don’t want to have a stage choir coming to this festival.”

Korice
\par{\color{black!55}\itshape [Excerpt ends here.]}
\end{owtsamplebox}

\begin{owtsamplebox}{Ours (1M), budget: 128 NFE\hfill $H_\mathrm{target}=5.0,\ k=1$}
{\color{black!65} $T=0.85$; $H=4.9393$ nats;
GenPPL $=38.23$; 127 calls; $n=512$.}\par
\textless{}\textbar{}endoftext\textbar{}\textgreater{}.

“We have developed our attitude in the past, but we can’t get our player up the road \textemdash{} trying to see what the model is \textemdash{} and tell them what the idea is,” Casinson said.

“We don’t know if it will happen, but certainly we will have to stay.”

The are expected to set together a club in the Territory, but it would be a long way to go forward, locking down the finals snout against the Malans, although the club expects it might be a way.

At home Camp, in the season, a group of about 30 people, played the game at the small city’s expense.

“We think about how the model works and for the club, that’s important for us,” Burke said.

“You don’t know about it. You don’t know about family. We didn’t want a family
\par{\color{black!55}\itshape [Excerpt ends here.]}
\end{owtsamplebox}

\begin{owtsamplebox}{Ours (1M), budget: 256 NFE\hfill $H_\mathrm{target}=5.0,\ k=1$}
{\color{black!65} $T=0.80$; $H=4.9059$ nats;
GenPPL $=32.15$; 256 calls; $n=512$.}\par
\textless{}\textbar{}endoftext\textbar{}\textgreater{} in the long run, too.

Certainly the end round is very bad, but we’ll have to go to the US to see if a lot will go on yet.\textless{}\textbar{}endoftext\textbar{}\textgreater{}Posted on by: Director Tao Lau- in Shanbe’s Congress

A new training mission that starts today, with the crew in position to go to China, from Chinese coast, 2 or 5 times a day at home.

This will be on Thursday as the crew will begin to climb the main surface station to the International Space Station. The 400 men will then take a prototype.

They will communicate the signal the satellite system receives instead of the a metre, in this case the path of the boom. From there, they will be getting a testator between their balloon information and the water minimum as well \textendash{} as they are working, and
\par{\color{black!55}\itshape [Excerpt ends here.]}
\end{owtsamplebox}

\begin{owtsamplebox}{Ours (1M), budget: 512 NFE\hfill $H_\mathrm{target}=5.0,\ k=1$}
{\color{black!65} $T=0.90$; $H=4.9996$ nats;
GenPPL $=30.93$; 512 calls; $n=512$.}\par
\textless{}\textbar{}endoftext\textbar{}\textgreater{} Thursday, October 7, 2017.

But the main mission achieved an almost exclusively visual point for that day. The mission will be planning to launch once a year, and by Wednesday the Lunar and Mars Mission rockets are expected to send operations for at least six weeks.

The number of spacecraft that have been launched on the moon rests on a scale between 17,000 and 8,000 days, with all the other ones.

Officials say the mission took the 4.6 miles long, 61-mile trip from Mars launched in 1976, about 250 miles away from home to Earth.

That will propel the closest mission to the icy view of Mars that has sailed through its previous journey of 32 years. The rover orbits its core about 1,300 miles away from the sun, and the trajectory of space can be as low as
\par{\color{black!55}\itshape [Excerpt ends here.]}
\end{owtsamplebox}

\begin{owtsamplebox}{Ours (1M), budget: 1,024 NFE\hfill $H_\mathrm{target}=5.0,\ k=1$}
{\color{black!65} $T=0.90$; $H=4.9800$ nats;
GenPPL $=25.60$; 1,024 calls; $n=512$.}\par
\textless{}\textbar{}endoftext\textbar{}\textgreater{}, on March 11, 2017.

In the global age of waste waste problems, waste, a method similar to the one used by U.S. homes and homes in Wilmington, Michigan, would have been a great way to try and respond to that problem. But there are multiple alternative advantages that would be a solution to the problem. American Waste Inc. argued, in April, that it could reduce the amount of waste waste and reduce waste levels of as much as 1.4 billion pounds. This would allow waste to be used in buildings or other areas that could be identified from waste.

“In waste, there is a tremendous amount of waste that is injection into the waste that you put in. So it creates a very simple environment, and it also means that there are very acidic volumes of the waste needed
\par{\color{black!55}\itshape [Excerpt ends here.]}
\end{owtsamplebox}

\Needspace{8\baselineskip}

\subsection{Target entropy \texorpdfstring{$H=5.25$}{H=5.25} nats}
\label{app:owt-samples-h5-25}

\begin{owtsamplebox}{Ours (1M), budget: 32 NFE\hfill $H_\mathrm{target}=5.25,\ k=1$}
{\color{black!65} $T=0.90$; $H=5.2894$ nats;
GenPPL $=78.70$; 32 calls; $n=512$.}\par
\textless{}\textbar{}endoftext\textbar{}\textgreater{} it.”

On Monday afternoon, Christie and appearing at a rally called d’s on Pittsburgh. Rep. Peter L. Mallan (R) \& Tenn., met with a Hall of business owners Larry Klolffer.

Bosan called fostering the "irquastic model" for competing people. Earlier this year, he would work with a group community committee to discuss the problem and the need to boost efforts.\textless{}\textbar{}endoftext\textbar{}\textgreater{}CAP CEO Aie Hansen highlighted an email carrying in the Federal 25 process that an attorney might speak with other local organizations if they were to the committee to serve on the B.C. Greens.

Mosbee said the group he was involved with other members had bought into itself.

“I kind of thought about the talent and amount of involvement we need, and the overall impact on Fair’s (conference) leadership, plan to tackle this program. I believe we have to recruit
\par{\color{black!55}\itshape [Excerpt ends here.]}
\end{owtsamplebox}

\begin{owtsamplebox}{Ours (1M), budget: 64 NFE\hfill $H_\mathrm{target}=5.25,\ k=1$}
{\color{black!65} $T=0.90$; $H=5.2192$ nats;
GenPPL $=63.26$; 64 calls; $n=512$.}\par
\textless{}\textbar{}endoftext\textbar{}\textgreater{} communities.”

At the weekend, residents and vendors \textemdash{} an event of town service weekly \textemdash{} “that the mayor is going to prioritize police from 6 a.m. to 3 a.m.,” Rars said.

Rars said the council “probiously will stand along with people to antagonite the city.”

Books and business libraries have up a quarterly in the city, weeks.

Last week, Progressive announced it was investing in the Queen’s Centre Team, working with the SpotlightSciters to bring in what will become the stations, businesses and transit services.

Rater also questioned what caused the dramatic continued rise of Winnipeg’s needs.

“These sorts of things have nothing to do with the here of the few and the future of life,” she said.

“We don’t know where parents are going,” she added, “We don’t want to have a football bus coming to this city.”
\par{\color{black!55}\itshape [Excerpt ends here.]}
\end{owtsamplebox}

\begin{owtsamplebox}{Ours (1M), budget: 128 NFE\hfill $H_\mathrm{target}=5.25,\ k=1$}
{\color{black!65} $T=0.95$; $H=5.2645$ nats;
GenPPL $=60.77$; 127 calls; $n=512$.}\par
\textless{}\textbar{}endoftext\textbar{}\textgreater{}\texttt{[U+FFFD]}

The list of players with an emphasis is limited. “We can’t do a lot of the field science processing to guess what a game is, and then learn what the game is,” Matto said. “But, because we understand what it is, it will kind of make you down right back to games.”

The coaches need to test building a game in every river, but it will take a longer time to collect player data. In the process, none of the competitors will fit. The team expects it will be able to move those young players and their leaders in a row, a group of roughly 17,000 players to play at the Wildcats’ 14th games.

“We like thinking about the simulation experience and the simulation problem, that’s important for our coaches,” Matto said. “You don’t know about it. You don’t
\par{\color{black!55}\itshape [Excerpt ends here.]}
\end{owtsamplebox}

\begin{owtsamplebox}{Ours (1M), budget: 256 NFE\hfill $H_\mathrm{target}=5.25,\ k=1$}
{\color{black!65} $T=0.90$; $H=5.2485$ nats;
GenPPL $=53.62$; 256 calls; $n=512$.}\par
\textless{}\textbar{}endoftext\textbar{}\textgreater{} Canada until March 1, 2016.

But the exact assessment is unclear: ‘We’ll have to go to military experts to know what a solution is,’ Army spokesman Col. Gen. John Dosman, told reporters’ event.

Dosman cautioned that the importance of deploying the troops would be a problem to an officer who drinks 10, 25 or four days a day in Canada.

He also said this fall the usual plan would be to expand the soldiers base battalion to 10,000 people. But two troops would not meet that option.

That would give the Canadian National Armed Forces an advance of about 10,000 from the Canadian side of the base base. From there, that would mean connecting three training zones between the base base and the infantry reserve, each of which has about 80,000 people.

The detachment could require Wednesday to train
\par{\color{black!55}\itshape [Excerpt ends here.]}
\end{owtsamplebox}

\begin{owtsamplebox}{Ours (1M), budget: 512 NFE\hfill $H_\mathrm{target}=5.25,\ k=1$}
{\color{black!65} $T=1.00$; $H=5.2211$ nats;
GenPPL $=44.93$; 512 calls; $n=512$.}\par
\textless{}\textbar{}endoftext\textbar{}\textgreater{} results of March 1, 2011.

In August, NASA launched the JENN-2 spacecraft, the third spacecraft the agency has built to undertake such a mission, and by July the agency's 12STAR missions will deploy 50 satellite, communications and EVICs. They will provide a limited spacecraft which will enable the agency to transmit data about a mission of 17 neosetary objects, including Jupiter, to the future solar mission.

Jawn was collecting at the 4.6MEs at the west point of the joint probe in 2005. Just two years later, NASA's only space mission is working directly toward the OELISIII mission, a spacecraft that has been at its best stage for seven years. The telescope has an estimated 440 kilometers between nearby objects V and C rockets. At the top of the spacecraft will work an estimated 1,100 passengers, according to the joint
\par{\color{black!55}\itshape [Excerpt ends here.]}
\end{owtsamplebox}

\begin{owtsamplebox}{Ours (1M), budget: 1,024 NFE\hfill $H_\mathrm{target}=5.25,\ k=1$}
{\color{black!65} $T=1.05$; $H=5.2632$ nats;
GenPPL $=41.24$; 1,024 calls; $n=512$.}\par
\textless{}\textbar{}endoftext\textbar{}\textgreater{} price of \$0.79. The approval for the drug is an indication that Shada Pharma, already under way for the drug announced by U.S. and European regulators in January, might come ahead with an outlier clinical drug, despite the fact that B. vulgari is used in mice, to produce clinical testing and discovery trial. Shada has inserted DIMA, a test for proteins found in human cells, bones and other tissues consisting of as many as 1.4 billion genes. It has been used to identify spread of different cells in many different samples, in this case meat, liver, brain and liver, and in a specific source of tissue that is involved in the human genome. “DIMA is a common molecule that occurs in the body, but it is involved in lung and in muscle cord cells and when grown in muscle
\par{\color{black!55}\itshape [Excerpt ends here.]}
\end{owtsamplebox}

\Needspace{8\baselineskip}

\subsection{Target entropy \texorpdfstring{$H=5.5$}{H=5.5} nats}
\label{app:owt-samples-h5-5}

At the 1,024-NFE budget, the default $k=1$ sweep does not reach
this entropy target. Its $k=2$ example below is a separate sampler
configuration and uses 64 source samples rather than 512.

\begin{owtsamplebox}{Ours (1M), budget: 32 NFE\hfill $H_\mathrm{target}=5.5,\ k=1$}
{\color{black!65} $T=1.00$; $H=5.5393$ nats;
GenPPL $=145.56$; 32 calls; $n=512$.}\par
\textless{}\textbar{}endoftext\textbar{}\textgreater{} her.”

ITITIONATION: dinner and showing at a rally called yaleshell on Sept. 24. - surprised to get what K\&a.S. did with a Paganiku: ACT organizer Jeol Perkins.

Oerr says the reservation of the Kinderhaster is tough for numerous people. In this post, he gives examples with two community and local cultures discuss the problem with the horning bird rescue.\textless{}\textbar{}endoftext\textbar{}\textgreater{}ORITA, Md. (AP) \textemdash{} The Fed’s Education Department is working with other local organizations so they can test the money they save on law enforcement residents and their businesses.

Director Jason Taylor said the group of around 5,000 members have lost their wounds.

“A lot of thought about the skills and amount of information we need, and the biggest danger from Capital Education center (ETA) program is if you don’t define where you are or recruit any group of people with
\par{\color{black!55}\itshape [Excerpt ends here.]}
\end{owtsamplebox}

\begin{owtsamplebox}{Ours (1M), budget: 64 NFE\hfill $H_\mathrm{target}=5.5,\ k=1$}
{\color{black!65} $T=1.00$; $H=5.5185$ nats;
GenPPL $=131.02$; 64 calls; $n=512$.}\par
\textless{}\textbar{}endoftext\textbar{}\textgreater{}\texttt{[U+FFFD]}

A number of mayors told Providence,C., meeting at an event of town management Wednesday that “the executive director is going to define what B.B.C. would be like. Nothing,” Dorell said.

Quence said the committee “propartiously will continue to identify the streets that be of residents.”

Education and job libraries have become increasingly focus in the area, too.

In July, Richmond said it was interested in the District’s Centre Act, which a same-party “strong coalition” will begin on infrastructure, businesses and transit services \textemdash{} a tactic Manager Patricia Patterson described as signaled the broad continued status of Clark’s body.

“These sorts of decisions have nothing to do with the best public facilities \textellipsis{} and the highest city services,” she said.

“You don’t know exactly who people are,” she added, “you don’t want to have a moment to go to [the
\par{\color{black!55}\itshape [Excerpt ends here.]}
\end{owtsamplebox}

\begin{owtsamplebox}{Ours (1M), budget: 128 NFE\hfill $H_\mathrm{target}=5.5,\ k=1$}
{\color{black!65} $T=1.05$; $H=5.5284$ nats;
GenPPL $=123.09$; 127 calls; $n=512$.}\par
\textless{}\textbar{}endoftext\textbar{}\textgreater{}\texttt{[U+FFFD]}

The list of people with positive stature is limited. “We can’t do a lot of research and then try to guess what a gene most people do before they know the person is,” Paulas said. “But then, to understand what it is, it is kind of becoming just silent out in a situation.”

So, people may look like a god once every century, but it would take a long time to recover their DNA. But the looming potential field of aging people looks complex. The team hopes it might be better to move these young people from their neighbors in a college, a group of maybe college individuals, and to collect data on others’ feelings and preferences.

“We can learn about the intro arts and the technical community, that feed into using data for their insight,” Paulas says. “You
\par{\color{black!55}\itshape [Excerpt ends here.]}
\end{owtsamplebox}

\begin{owtsamplebox}{Ours (1M), budget: 256 NFE\hfill $H_\mathrm{target}=5.5,\ k=1$}
{\color{black!65} $T=1.05$; $H=5.4690$ nats;
GenPPL $=118.06$; 256 calls; $n=512$.}\par
\textless{}\textbar{}endoftext\textbar{}\textgreater{} service until March 15, 2016.

And the sad thing is that Yahoo doesn’t have to have to limit it, with never wanting to be so trusted and not yet.\textless{}\textbar{}endoftext\textbar{}\textgreater{}Posted on:: Joe Tao McDonald/WassTV’s Flickr

\mbox{-}\mbox{-} PT is Radio One LIVE. View the form in here and listen to this title stream from 11,0595, for a moment of change.

FRE announced this week the radio station will be equipped with the AM surface monitor for fast resonation speeds. With two days already end, the OEMIADRAN satellites will test the radio components for the specific purpose of designing a field designed for the best synchronization of weather weather studies. They will begin with a satellite measuring two satellites added into their own definitions, an ocean receiving magnetom objects as high as 15,000 miles.

The station will provide the radio stations and
\par{\color{black!55}\itshape [Excerpt ends here.]}
\end{owtsamplebox}

\begin{owtsamplebox}{Ours (1M), budget: 512 NFE\hfill $H_\mathrm{target}=5.5,\ k=1$}
{\color{black!65} $T=1.10$; $H=5.4565$ nats;
GenPPL $=74.59$; 512 calls; $n=512$.}\par
\textless{}\textbar{}endoftext\textbar{}\textgreater{}, in particular26,27.

With the single thermore molecule BALIS2, a two-dimensional experimental model was developed to design such a compound. In October 1979 the German Regiwald University AMIRDEER/AMICE collaboration Otto A. Fagokamp conducted a demonstration of concrete, the mechanism for converting acid from a reaction that has six hydrogen molecules combined, namely the chemical unit (like carbon release buckhanine, for example), that store and store in lighting, helium-like material for helium molecules in contact.25 Six years later when Heinrich Erbmann developed the model for the AMIRDEER/AMICE collaboration, he resulted in his most experimental experiment for screening,28 which was the first chemical material for an meturous ion compound. The first product of the new approach was this zinc atom, which was being optimized for the “clinging electrode”, but not yet bound to the interpean crystals.29 Example of the compound 87236-814 which
\par{\color{black!55}\itshape [Excerpt ends here.]}
\end{owtsamplebox}

\begin{owtsamplebox}{Ours (1M), budget: 1,024 NFE\hfill $H_\mathrm{target}=5.5,\ k=2$}
{\color{black!65} $T=1.15$; $H=5.5074$ nats;
GenPPL $=97.11$; 1,023 calls; $n=64$.}\par
\textless{}\textbar{}endoftext\textbar{}\textgreater{} proteins are amassing among genes. This findings provide a potential multam of B1-related gene expression indicating that the two homotypes are T, prima and prima6. There are also several subtypes, resulting in each B1 gene sequence. T, prima3,4 and 2.5 are identical to TGP. Aimensional OR of B1 gene genes identified in these results reveal cross-forming between LAT and F to protein regulator as well as regulator and allinetics. These results indicate a widely shared search for T or B1 gene variants. 3. Functional findings These findings include the presence of three B1-specific T intermedi groups that were identified in the same template that were selected for proteins belonging to a single-element gene. The results of this duplication are described in laboratory and outside sequence findings g. Stice coformation is more prone to adstian deformation alteration than in activated
\par{\color{black!55}\itshape [Excerpt ends here.]}
\end{owtsamplebox}

\end{document}